\documentclass{article}
\PassOptionsToPackage{numbers}{natbib}

\usepackage[final]{neurips_2019}

\usepackage[utf8]{inputenc} 
\usepackage[T1]{fontenc}    
\usepackage{hyperref}       
\usepackage{url}            
\usepackage{booktabs}       
\usepackage{amsfonts}       
\usepackage{nicefrac}       
\usepackage{microtype}      

\usepackage{caption}
\usepackage{mathtools}

\usepackage[parfill]{parskip}
\usepackage{amssymb}
\usepackage{hyperref}
\usepackage{amsthm}
\usepackage{amsmath}
\usepackage{soul}
\usepackage{float}
\usepackage{algorithm,algorithmic}
\usepackage{color}
\usepackage{caption} 
\usepackage[export]{adjustbox} 
\allowdisplaybreaks
\newcommand\numberthis{\addtocounter{equation}{1}\tag{\theequation}}

\newtheorem{theorem}{Theorem}[section]

\newtheorem{proposition}{Proposition}[section]
\newtheorem{lemma}[theorem]{Lemma}
\newtheorem{corollary}{Corollary}[theorem]

\theoremstyle{definition}
\newtheorem{definition}{Definition}[section]

\newtheorem{remark}{Remark}[section]
\renewcommand{\vec}[1]{\boldsymbol{\mathbf{#1}}}

\title{Uniform convergence may be unable to explain generalization in deep learning}

\author{%
  Vaishnavh Nagarajan \\
  Department of Computer Science\\
 Carnegie Mellon University\\
  Pittsburgh, PA \\
  \texttt{vaishnavh@cs.cmu.edu} \\
  \And
  J. Zico Kolter \\
  Department of Computer Science\\
 Carnegie Mellon University \&\\
 Bosch Center for Artificial Intelligence \\
 Pittsburgh, PA \\
 \texttt{zkolter@cs.cmu.edu}
}

\begin{document}

\maketitle

\begin{abstract}
Aimed at explaining the surprisingly good generalization behavior of overparameterized deep networks, recent works have developed a variety of generalization bounds for deep learning,  all  based on the fundamental learning-theoretic technique of uniform convergence. While
it is well-known that many of these existing bounds are numerically large, through numerous experiments, we bring to light a more concerning aspect of these bounds: 
in practice,  these bounds can {\em increase} with the training dataset size. Guided by our observations,
we then present examples of overparameterized linear classifiers and neural networks trained by  gradient descent (GD) where uniform convergence provably cannot ``explain generalization'' -- even if we take into account the implicit bias of GD {\em to the fullest extent possible}. More precisely, even if we consider only the set of classifiers output by GD, which have test errors less than some small $\epsilon$ in our settings, we show that applying (two-sided) uniform convergence on this set of classifiers will yield only a vacuous generalization guarantee larger than $1-\epsilon$. Through these findings,
we cast doubt on the power of uniform convergence-based generalization bounds to provide a complete picture of why overparameterized deep networks generalize well. 
\end{abstract}

\section{Introduction}

Explaining why overparameterized deep networks generalize well \cite{neyshabur15inductive,zhang17generalization} has become an important open question in deep learning. 
How is it possible that a large network can be trained to perfectly fit randomly labeled data (essentially by memorizing the labels), and yet, the same network when trained to perfectly fit real training data, generalizes well to unseen data? This called for a ``rethinking'' of conventional, algorithm-\textit{independent} techniques to explain generalization. Specifically, it was argued that learning-theoretic approaches must be reformed by identifying and incorporating the implicit bias/regularization of stochastic gradient descent (SGD) \cite{brutzkus18sgd,soudry18sgd,neyshabur17exploring}. Subsequently, a huge variety of novel and refined, algorithm-\textit{dependent} generalization bounds for deep networks have been developed, all based on \textit{uniform convergence}, the most widely used tool in learning theory. The ultimate goal of this ongoing endeavor	
is to derive bounds on the generalization error that  (a) are small, ideally non-vacuous (i.e., $< 1$), (b) reflect the same width/depth dependence as the generalization error (e.g., become smaller with increasing width, as has been surprisingly observed in practice), (c) apply to the network learned by SGD (without any modification or explicit regularization) and (d) increase with the proportion of randomly flipped training labels (i.e., increase with memorization).

While every bound meets some of these criteria (and sheds a valuable but partial insight into generalization in deep learning), there is no known bound that meets all of them   simultaneously. While most bounds \cite{neyshabur15norm,bartlett17spectral,golowich17size,neyshabur18pacbayes,nagarajan2018deterministic,neyshabur18unitwise} apply to the original network, they are neither numerically small for realistic dataset sizes, nor exhibit the desired width/depth dependencies (in fact, these bounds grow exponentially with the depth). The remaining bounds hold either only on a compressed network \citep{arora18compression} or a stochastic network \citep{langford01not} or a network that has been further modified via optimization or more than one of the above \citep{dziugaite17nonvacuous,zhou2018nonvacuous}. Extending these bounds to the original network is understood to be highly non-trivial \citep{nagarajan2018deterministic}.  While strong width-independent bounds have been derived for two-layer ReLU networks \citep{li18learning,zhu18beyond}, these rely on a carefully curated, small learning rate and/or large batch size. (We refer the reader to Appendix~\ref{app:table} for a tabular summary of these bounds.)

In our paper, we bring to light another fundamental issue with existing bounds. We demonstrate that these bounds violate another natural but largely overlooked criterion for explaining generalization: (e) the bounds should decrease with the dataset size at the same rate as the generalization error. In fact, we empirically observe that these bounds can {\em increase} with dataset size, which is arguably a more concerning observation than the fact that they are large for a specific dataset size.

Motivated by the seemingly insurmountable hurdles towards developing bounds satisfying all the above five necessary criteria, we take a step back and examine how the underlying technique of uniform convergence may itself be inherently limited
in the overparameterized regime. Specifically, we present examples of overparameterized linear classifiers and neural networks trained by GD (or SGD) where 
uniform convergence can {\em provably} fail to explain generalization. 
Intuitively,  our examples highlight that overparameterized models trained by gradient descent can learn decision boundaries that are largely ``simple'' -- and hence generalize well -- but have ``microscopic complexities'' which cannot be explained away by uniform convergence. Thus our results call into question the active ongoing pursuit of using uniform convergence to fully explain generalization in deep learning.


\subparagraph{Our contributions in more detail.}
We first show that in practice certain weight norms of deep ReLU networks, such as the distance from initialization, {\em increase} polynomially with the number of training examples (denoted by $m$). %
We then show that as a result, existing generalization bounds -- all of which depend on such weight norms -- fail to reflect even a dependence on $m$ even reasonably similar to the actual test error, violating criterion (e); for sufficiently small batch sizes, these bounds even grow with the number of examples. 
This observation uncovers a conceptual gap in our understanding of the puzzle, by pointing 
 towards a source of vacuity unrelated to parameter count.

As our second contribution, we consider three example setups of overparameterized models trained by (stochastic) gradient descent -- a linear classifier, a sufficiently wide neural network with ReLUs and an infinite width neural network with exponential activations (with the hidden layer weights frozen) -- that learn some underlying data  distribution with small generalization error (say, at most $\epsilon$). These settings also simulate our observation that norms such as distance from initialization grow with dataset size $m$.
More importantly, we prove that, in these settings, {\em any} two-sided {uniform convergence bound} would yield a (nearly) vacuous generalization bound.  

Notably, this vacuity holds even if we ``aggressively'' take implicit regularization into account while applying uniform convergence --  described more concretely as follows. Recall that roughly speaking a uniform convergence bound essentially evaluates the complexity of a hypothesis class (see Definition~\ref{def:unif}). 
As suggested by \citet{zhang17generalization}, one can tighten uniform convergence bounds by pruning the hypothesis class to remove extraneous hypotheses never picked by the learning algorithm  for the data distribution of interest. In our setups, even if we apply uniform convergence on the set of {only those hypotheses picked by the learner whose test errors are all negligible (at most $\epsilon$)}, one can get no better than a nearly vacuous bound on the generalization error (that is at least $1-\epsilon$). In this sense, we say that uniform convergence provably cannot explain generalization in our settings. Finally, we note that while nearly all existing uniform convergence-based techniques are two-sided, we show that even PAC-Bayesian bounds, which are typically presented only as one-sided convergence, also boil down to nearly vacuous guarantees in our settings. 


\subsection{Related Work}
\label{sec:related}

\subparagraph{Weight norms vs. training set size $m$.} 
Prior works like \citet{neyshabur17exploring} and \citet{nagarajan17role}  have studied the behavior of weight norms in deep learning. Although these works do not explicitly study the dependence of these norms on training set size $m$, one can infer from their plots
 that weight norms of deep networks show some increase with $m$.  \citet{belkin18kernel} reported a similar paradox in kernel learning, observing that norms that appear in kernel generalization bounds increase with $m$, and that this is due to noise in the labels. \citet{kawaguchi17generalization} showed that there exist linear models with arbitrarily large weight norms that can generalize well, although such weights are not necessarily found by gradient descent. We crucially supplement these observations in three ways.
  First, we empirically and theoretically demonstrate how, even with {\em zero label noise} (unlike \citep{belkin18kernel}) and by gradient descent (unlike \cite{kawaguchi17generalization}), a significant level of $m$-dependence can arise in the weight norms -- significant enough to make even the generalization bound grow with $m$. Next, we identify uniform convergence as the root cause behind this issue, and thirdly and most importantly, we provably demonstrate this is so.

\subparagraph{Weaknesses of Uniform Convergence.} 
Traditional wisdom is that uniform convergence bounds are a bad choice for complex classifiers like k-nearest neighbors because these hypotheses classes have infinite VC-dimension (which motivated the need for stability based generalization bounds in these cases \citep{rogerss78finite,bousquet02stability}). However, this sort of an argument against uniform convergence may still leave one with the faint hope that, by aggressively pruning the hypothesis class (depending on the algorithm and the data distribution), one can achieve meaningful uniform convergence. In contrast, we seek to {rigorously} and {thoroughly} rule out uniform convergence in the settings we study. We do this by first defining the tightest form of uniform convergence in Definition~\ref{def:unif-alg} -- one that lower bounds any uniform convergence bound -- and then showing that even this bound is vacuous in our settings. Additionally, we note that
we show this kind of failure of uniform convergence for linear classifiers, which is a much simpler model compared to k-nearest neighbors.
 For deep networks,
\citet{zhang17generalization} showed that applying uniform convergence on the whole hypothesis class fails, and that it should instead be applied in an algorithm-dependent way.  \textit{Ours is a much different claim} -- that uniform convergence is inherently problematic in that even the algorithm-dependent application would fail -- casting doubt on the rich line of post-\citet{zhang17generalization} algorithm-dependent approaches. At the same time, we must add the disclaimer that our results do not preclude the fact that uniform convergence may still work if GD is run with explicit regularization (such as weight decay). Such a regularized setting however, is not the main focus of the generalization puzzle \citep{zhang17generalization,neyshabur15inductive}.

Prior works \cite{vapnik71uniform,shwartz10learnability} have also focused on understanding uniform convergence for {\em learnability of learning problems}. Roughly speaking, learnability is a strict notion that does not have to hold even though an algorithm may generalize well for simple distributions in a learning problem. While we defer the details of these works in Appendix~\ref{sec:learnability}, we emphasize here that these results are orthogonal to (i.e., neither imply nor contradict) our results.

%
%


\section{Existing bounds vs. training set size}
\label{sec:experiments}
As we stated in criterion (e) in the introduction, a fundamental requirement from a generalization bound, however  numerically large the bound may be, is that it should vary inversely with the size of the training dataset size $(m)$ like the observed generalization error. Such a requirement is satisfied even by standard parameter-count-based VC-dimension bounds, like $\mathcal{O}(dh/\sqrt{m})$ for depth $d$, width $h$ ReLU networks \citep{harvey17vc}. Recent works have ``tightened'' the parameter-count-dependent terms in these bounds by replacing them with seemingly innocuous norm-based quantities; however, we show below that this has also inadvertently introduced training-set-size-count dependencies in the numerator, contributing to the vacuity of bounds.  With these dependencies, the generalization bounds even increase with training dataset size for small batch sizes.  

\textbf{Setup and notations.} We focus on fully connected networks of depth $d=5$, width $h=1024$ trained on MNIST, although we consider other settings in Appendix~\ref{app:experiments}. We use SGD with learning rate $0.1$ and batch size $1$ to minimize cross-entropy loss until $99\%$ of the training data are classified correctly by a \textit{margin} of at least $\gamma^\star =10$ i.e., if we denote by $f(\vec{x})[y]$ the real-valued logit output (i.e., pre-softmax) on class $y$ for an input $\vec{x}$, we ensure that for $99\%$ of the data $(\vec{x},y)$, the {margin} $\Gamma(f(\vec{x}),y) := f(\vec{x})[y] - \max_{y'\neq y} f(\vec{x})[y']$ is at least  $\gamma^\star$. We emphasize that, from the perspective of generalization guarantees, this stopping criterion helps standardize training across different hyperparameter values, including different values of $m$ \cite{neyshabur17exploring}. Now, observe that for this particular stopping criterion, the test error empirically decreases with size $m$ as $1/m^{0.43}$ as seen in Figure~\ref{fig:terms} (third plot). However, we will see that the story is starkly different for the generalization bounds. 

\subparagraph{Norms grow with training set size $m$.} 
Before we examine the overall generalization bounds themselves, 
we first focus on two quantities that recur in the numerator of many recent bounds: the $\ell_2$ distance of the weights from their initialization \citep{dziugaite17nonvacuous,nagarajan17role} 
 and the product of spectral norms of the weight matrices of the network \cite{neyshabur18pacbayes,bartlett17spectral}.
  We observe in Figure~\ref{fig:terms} (first two plots, blue lines) that both these quantities grow at a polynomial rate with $m$: the former at the rate of at least $m^{0.4}$ and the latter at a rate of $m$. 
Our observation is a follow-up to
\citet{nagarajan17role} who argued that while distance of the parameters from {the origin}
  grows with width as $\Omega(\sqrt{h})$, the distance from initialization is width-independent (and even decreases with width); hence, they concluded that incorporating the initialization would improve generalization bounds by a $\Omega(\sqrt{h})$ factor.
However, our observations imply that, even though distance from initialization would help explain generalization better in terms of width, it conspicuously fails to help explain generalization in terms of its dependence on $m$ 
 (and so does distance from origin as we show in Appendix Figure~\ref{fig:more-evidence-dist}). \footnote{It may be tempting to think that our observations are peculiar to the cross-entropy loss for which the optimization algorithm diverges.
But we observe that even for the squared error loss (Appendix~\ref{app:experiments}) where the optimization procedure does not diverge to infinity, distance from initialization grows with $m$.  
}

Additionally, we also examine another quantity as an alternative to distance from initialization:  the $\ell_2$ diameter of {\em the parameter space explored by SGD}. That is, for a fixed initialization and data distribution, we consider the set of all parameters learned by SGD across all draws of a dataset of size $m$; we then consider the diameter of the smallest ball enclosing this set. 
If this diameter exhibits a better behavior than the above quantities, one could then explain generalization better by replacing the distance from initialization with the distance from the center of this ball in existing bounds.  As a lower bound on this diameter, we consider the distance between the weights learned on two independently drawn datasets from the given initialization. Unfortunately, we observe that even this quantity shows a similar undesirable behavior with respect to $m$ like distance from initialization (see Figure~\ref{fig:terms}, first plot, orange line).

\begin{figure}[t]
    \centering
    \begin{minipage}[t]{\textwidth}
           \begin{minipage}{.25\textwidth}
        \centering
        \adjincludegraphics[width=1\textwidth,trim={0 {0} 0 0},clip,,valign=t]{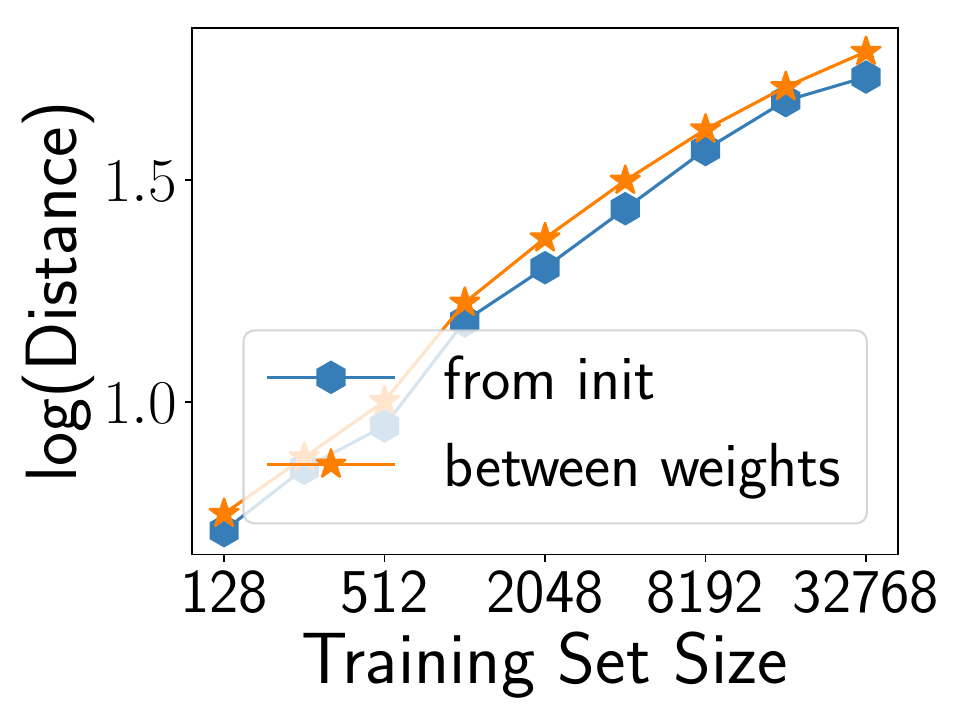} %
    \end{minipage}%
            \begin{minipage}{.25\textwidth}
        \centering
        \adjincludegraphics[width=1\textwidth,trim={0 {0} 0 0},clip,valign=t]{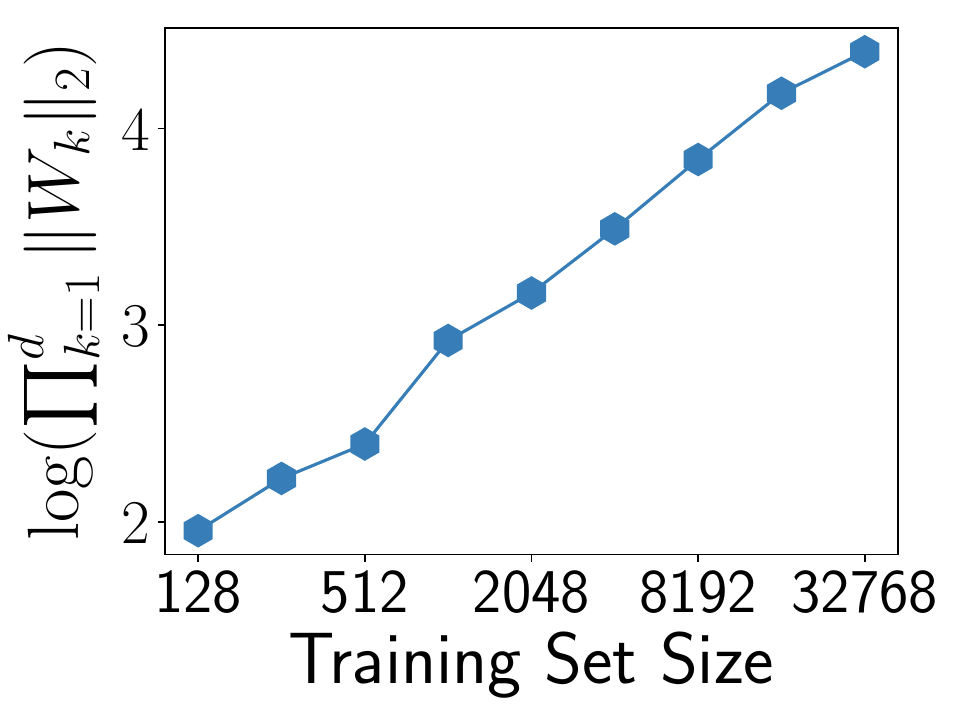} %
    \end{minipage} 
               \begin{minipage}{.25\textwidth}
        \centering
        \adjincludegraphics[width=1\textwidth,trim={0 {0} 0 0},clip,,valign=t]{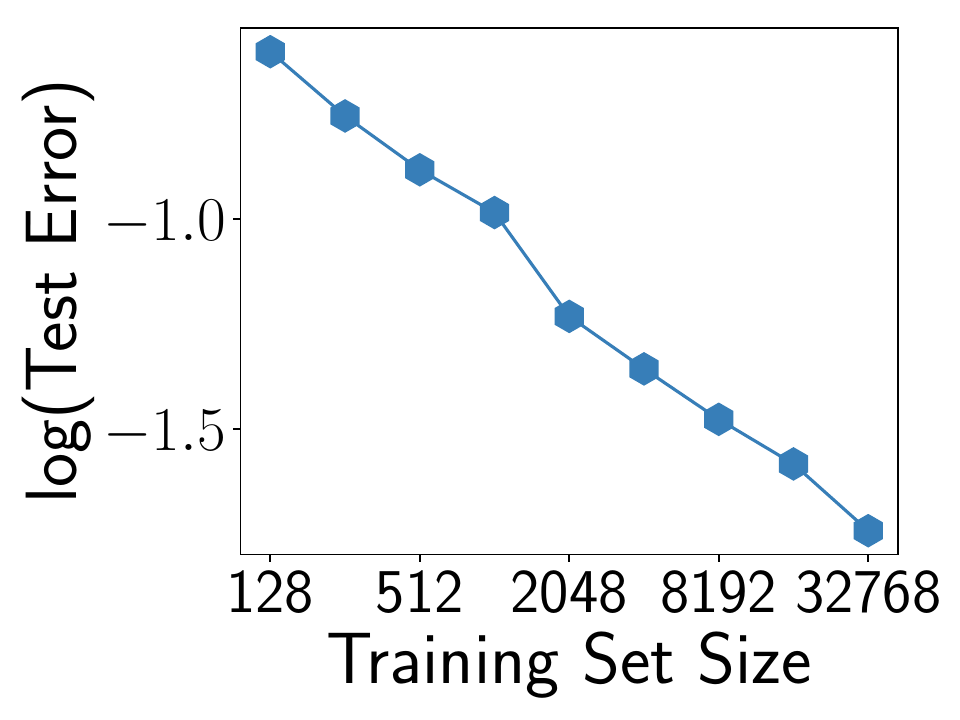} %
    \end{minipage}%
               \begin{minipage}{.25\textwidth}
        \centering
        \adjincludegraphics[width=1\textwidth,trim={0 {0} 0 0},clip,,valign=t]{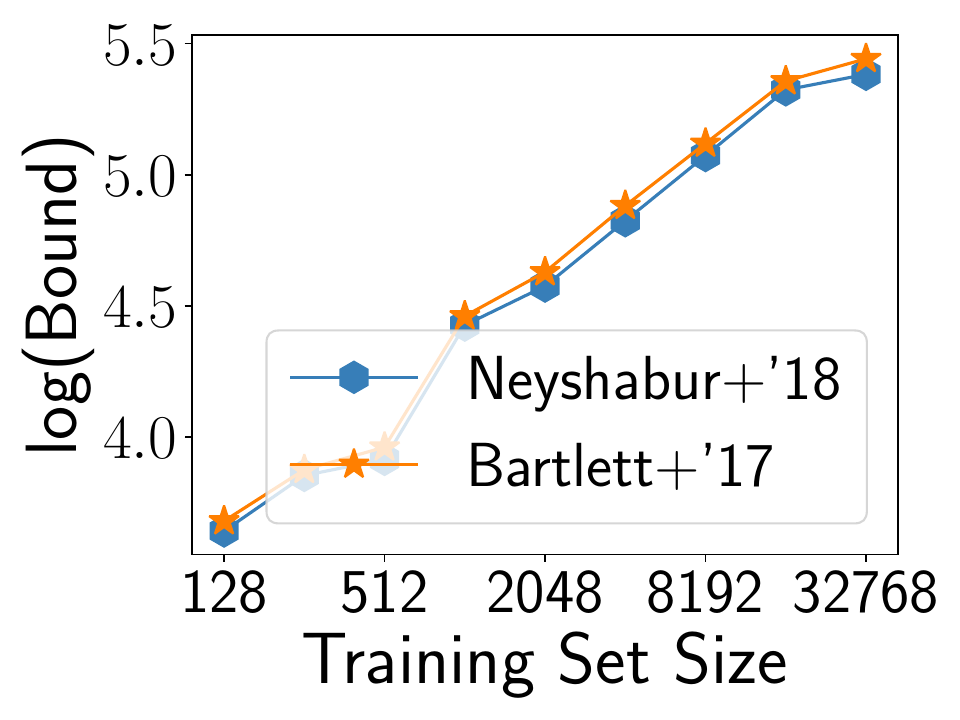} %
    \end{minipage}%
    \end{minipage}
       \caption{\textbf{Experiments in Section~\ref{sec:experiments}:} In the \textbf{first} figure, we plot (i)  $\ell_2$ the distance of the network from the initialization and (ii) the $\ell_2$ distance between the weights learned on two random draws of training data starting from the same initialization.
       In the \textbf{second} figure we plot the product of spectral norms of the weights matrices.
       In the \textbf{third} figure, we plot the  test error.
       In the \textbf{fourth} figure, we plot the bounds from \cite{neyshabur18pacbayes,bartlett17spectral}.
       Note that we have presented log-log plots and the exponent of $m$ can be recovered from the slope of these plots. } 
       \label{fig:terms}
\end{figure}

\subparagraph{The bounds grow with training set size $m$.}
We now turn to evaluating existing guarantees from \citet{neyshabur18pacbayes} and \citet{bartlett17spectral}. As we note later, our observations apply to many other bounds too. 
 Let $W_1, \hdots, W_d$ be the weights of the learned network (with $W_1$ being the weights adjacent to the inputs), $Z_1, \hdots, Z_d$  the random initialization, $\mathcal{D}$ the true data distribution and $S$ the training dataset. For all inputs $\vec{x}$, let $\| \vec{x}\|_2 \leq B$. Let $\|\cdot \|_2, \| \cdot\|_F, \| \cdot \|_{2,1}$ denote the spectral norm, the Frobenius norm and the matrix $(2,1)$-norm respectively; let $\mathbf{1}[\cdot]$ be the indicator function. Recall that 
$\Gamma(f(\vec{x}),y) := f(\vec{x})[y] - \max_{y'\neq y} f(\vec{x})[y']$  denotes the margin of the network on a datapoint.
Then, for any constant $\gamma$, these generalization guarantees are written as follows, ignoring log factors:
\begin{align*}
 \textrm{Pr}_{\mathcal{D}}[\Gamma(f(\vec{x}),y) & \leq 0]  \leq \frac{1}{m}\sum_{(x,y)\in S}\mathbf{1}[\Gamma(f(\vec{x}),y) \leq \gamma]  + \text{generalization error bound}. \numberthis \label{eq:gen-error-bound}
\end{align*}
Here the generalization error bound is of the form $\mathcal{O}\left( \frac{Bd\sqrt{h}}{\gamma\sqrt{m}}\prod_{k=1}^{d} \| W_k\|_2 \times  \texttt{dist} \right)$ where $\texttt{dist}$ equals
$\sqrt{\sum_{k=1}^{d}\frac{\|W_k - Z_k \|_F^2}{\|W_k\|^2_2}}$ in \cite{neyshabur18pacbayes}  and $\frac{1}{d\sqrt{h}} \left( {\sum_{k=1}^{d}\left(\frac{\|W_k - Z_k \|_{2,1}}{\|W_k\|_2}\right)^{2/3}}\right)^{3/2} $ in \cite{bartlett17spectral}.

 In our experiments, since we train the networks to fit at least $99\%$ of the datapoints with a margin of $10$, in the above bounds, we set $\gamma=10$ so that the first train error term in the right hand side of Equation~\ref{eq:gen-error-bound} becomes a small value of at most $0.01$. We then plot in Figure~\ref{fig:terms} (fourth plot), the second term above, namely the generalization error bounds, and 
 observe that all these bounds {\em grow with the sample size $m$} as $\Omega(m^{0.68})$, thanks to the fact that the terms in the numerator of these bounds grow with $m$. Since we are free to plug in $\gamma$ in Equation~\ref{eq:gen-error-bound}, one may wonder whether there exists a better choice of $\gamma$ for which we can observe a smaller increase on $m$ (since the plotted terms inversely depend on $\gamma$).  However, in Appendix Figure~\ref{fig:median-margin} we establish that even for larger values of $\gamma$, this $m$-dependence remains. Also note that, although we do not plot the bounds from \citep{nagarajan2018deterministic,golowich17size}, these have nearly identical norms in their numerator, and so one would not expect these bounds to show  radically better behavior with respect to $m$.
 Finally,  we defer experiments conducted for other varied settings, and the neural network bound from \citep{neyshabur18unitwise} to Appendix~\ref{app:experiments}. 

While the bounds might show better $m$-dependence for other settings -- indeed, for larger batches, we show in Appendix~\ref{app:experiments} that the bounds behave better -- we believe that the egregious break down of these bounds in this setting (and many other hyperparameter settings as presented in Appendix~\ref{app:experiments}) must imply fundamental issues with the bounds themselves.
 While this may be addressed to some extent with a better understanding of implicit regularization in deep learning, we regard our observations as a call for taking a step back and clearly understanding any inherent limitations in the theoretical tool underlying all these bounds namely, uniform convergence. 
\footnote{\textbf{Side note:} Before we proceed to the next section, where we blame uniform convergence for the above problems, we briefly note that we considered another simpler possibility. Specifically, we hypothesized that, for {\em some} (not all) existing bounds, the above problems could arise from an issue that does not involve uniform convergence, which we term as {\em pseudo-overfitting}. Roughly speaking, a classifier pseudo-overfits when its decision boundary is simple but its {\em real-valued} output has large ``bumps'' around some or all of its training datapoint.  As discussed in Appendix~\ref{app:pseudo-overfit}, deep networks pseudo-overfit only to a limited extent, and hence psuedo-overfitting does not provide a complete explanation for the issues faced by these bounds.} 


\section{Provable failure of uniform convergence}
\label{sec:setup}

\subparagraph{Preliminaries.}
Let $\mathcal{H}$ be a class of hypotheses mapping from $\mathcal{X}$ to $\mathbb{R}$, and let $\mathcal{D}$ be a distribution over $\mathcal{X} \times \{-1, +1 \}$. The loss function we mainly care about is the 0-1 error; but since a direct analysis of the uniform convergence of the 0-1 error is hard, sometimes a more general margin-based surrogate of this error (also called as ramp loss) is  analyzed for uniform convergence. Specifically, given the classifier's logit output $y' \in \mathbb{R}$ and the true label $y \in \{ -1,+1\}$, define 
\[
\mathcal{L}^{(\gamma)}(y',y) = \begin{cases} 
1 & y y' \leq 0 \\
1 - \frac{y y'}{\gamma} & yy' \in (0,\gamma) \\
0 & yy' \geq \gamma.
\end{cases}
\]
Note that $\mathcal{L}^{(0)}$ is the 0-1 error, and $\mathcal{L}^{(\gamma)}$ an upper bound on the 0-1 error. We define for any $\mathcal{L}$, the expected loss  as $\mathcal{L}_{\mathcal{D}}(h)  := \mathbb{E}_{(\vec{x},y) \sim \mathcal{D}}[\mathcal{L}(h(\vec{x}),y)   ]$ and the empirical loss  on a dataset $S$ of $m$ datapoints $\hat{\mathcal{L}}_{S}(h)  :=\frac{1}{m}\sum_{(\vec{x},y) \in S} \mathcal{L}(h(\vec{x}),y)$.
%
%
Let $\mathcal{A}$ be the learning algorithm and let
$h_{S}$ be the hypothesis output by the algorithm on a dataset $S$ (assume that any training-data-independent randomness, such as the initialization/data-shuffling is fixed). 

For a given $\delta \in (0,1)$, the generalization error of the algorithm is essentially a bound on the difference between the error of the hypothesis $h_{S}$ learned on a training set $S$ and the expected error over $\mathcal{D}$, that holds with high probability of at least $1-\delta$ over the draws of $S$. More formally: \\

\begin{definition} 
\label{def:gen-error}
The\textbf{ generalization error}  of $\mathcal{A}$
with respect to loss $\mathcal{L}$ is the smallest value $\epsilon_{\text{gen}}(m, \delta)$ such that:
$\textrm{Pr}_{S \sim \mathcal{D}^m}\left[
 \mathcal{L}_{\mathcal{D}}(h_{S}) -
 \hat{\mathcal{L}}_{S}(h_{S}) 
\leq \epsilon_{\text{gen}}(m, \delta) \right] \geq 1- \delta$.
\end{definition}

To theoretically bound the generalization error of the algorithm, the most common approach is to provide a two-sided uniform convergence bound on the hypothesis class used by the algorithm, where, for a given draw of $S$, we look at convergence for all the hypotheses in $\mathcal{H}$ instead of just $h_{S}$:\\ 

\begin{definition} 
\label{def:unif}
The \textbf{uniform convergence bound} 
with respect to loss $\mathcal{L}$ is the smallest value $\epsilon_{\text{unif}}(m, \delta)$ such that:
$\textrm{Pr}_{S \sim \mathcal{D}^m}\left[
\sup_{h \in \mathcal{H}} \left| \mathcal{L}_{\mathcal{D}}(h) -
 \hat{\mathcal{L}}_{S}(h) \right| \leq \epsilon_{\text{unif}}(m, \delta)\right] \geq 1- \delta$.
\end{definition}

\subparagraph{Tightest algorithm-dependent uniform convergence.} The bound given by $\epsilon_{\text{unif}}$
can be tightened by ignoring many extraneous hypotheses in $\mathcal{H}$ never picked by $\mathcal{A}$ for a given simple distribution $\mathcal{D}$. This is typically done by focusing on a norm-bounded class of hypotheses that the algorithm $\mathcal{A}$ implicitly restricts itself to. Let
us take this to the extreme by applying uniform convergence on ``the smallest possible class'' of hypotheses, namely, {\em only} those hypotheses that are picked by $\mathcal{A}$ under $\mathcal{D}$, excluding everything else. Observe that pruning the hypothesis class any further would not imply a bound on the generalization error, and hence applying uniform convergence on this aggressively pruned hypothesis class would yield {the tightest possible uniform convergence bound}.
 Recall that we care about this formulation because our goal is to rigorously and thoroughly rule out the possibility that no kind of uniform convergence bound, however cleverly applied, can explain generalization in our settings of interest (which we will describe later).

To formally capture this bound, it is helpful to first rephrase the above definition of $\epsilon_{\text{unif}}$: we can say that $\epsilon_{\text{unif}}(m,\delta)$ is  the smallest value for which there exists a set of sample sets $\mathcal{S}_{\delta} \subseteq (\mathcal{X} \times \{-1,1 \})^m$ for which $Pr_{S\sim \mathcal{D}^m}[S \in \mathcal{S}_{\delta}] \geq 1-\delta$ and furthermore, $\sup_{S \in \mathcal{S}_{\delta}} \sup_{h\in \mathcal{H}} | \mathcal{L}_{\mathcal{D}}(h) -
 \hat{\mathcal{L}}_{S}(h) | \leq \epsilon_{\text{unif}}(m,\delta)$. Observe that this definition is {\em equivalent} to Definition~\ref{def:unif}. Extending this rephrased definition, we can define the tightest uniform convergence bound by replacing $\mathcal{H}$ here with only those hypotheses that are explored by the algorithm $\mathcal{A}$ under the datasets belonging to $\mathcal{S}_{\delta}$:\\

\begin{definition} 
\label{def:unif-alg}
The \textbf{tightest algorithm-dependent uniform convergence bound}
with respect to loss $\mathcal{L}$ is the smallest value $\epsilon_{\text{unif-alg}}(m, \delta)$ for which there exists a set of sample sets $\mathcal{S}_{\delta}$ such that $Pr_{S\sim \mathcal{D}^m}[S \in \mathcal{S}_{\delta}] \geq 1-\delta$ and 
 if we define the space of hypotheses explored  by $\mathcal{A}$ on $\mathcal{S}_{\delta}$ as $\mathcal{H}_{\delta} := \bigcup_{S \in \mathcal{S}_{\delta}}  \{ h_{S} \} \subseteq \mathcal{H}$, the following holds:
$
\sup_{S \in \mathcal{S}_{\delta}}
\sup_{h \in \mathcal{H}_{\delta}}  \left| 
\mathcal{L}_{\mathcal{D}}(h) -
 \hat{\mathcal{L}}_{S}(h)  
\right| \leq \epsilon_{\text{unif-alg}}(m,\delta) 
$.
\end{definition} 






In the following sections, through examples of overparameterized models trained by GD (or SGD), we argue how even the above tightest algorithm-dependent uniform convergence can fail to explain generalization. i.e., in these settings, even though $\epsilon_{\text{gen}}$ is smaller than a negligible value $\epsilon$, we show that $\epsilon_{\text{unif-alg}}$ is large (specifically, at least $1-\epsilon$).   Before we delve into these examples, below we quickly outline the key mathematical idea by which uniform convergence is made to fail. 

Consider a scenario where the algorithm generalizes well i.e., for every training set $\tilde{S}$, 
$h_{\tilde{S}}$ has zero error on $\tilde{S}$ and has small test error. While this means that $h_{\tilde{S}}$ has small error on \textit{random} draws of a test set, it may still be possible that for every such $h_{\tilde{S}}$, there exists a corresponding ``bad'' dataset $\tilde{S}'$ -- that is not random, but rather dependent on $\tilde{S}$ -- on which $h_{\tilde{S}}$ has a large empirical error (say $1$).  
Unfortunately, uniform convergence runs into trouble while dealing with such bad datasets. Specifically, as we can see from the above definition, uniform convergence demands that $|\mathcal{L}_{\mathcal{D}}(h_{\tilde{S}}) -
 \hat{\mathcal{L}}_{S}(h_{\tilde{S}})  |$ be small on all datasets in $\mathcal{S}_{\delta}$, which excludes a $\delta$ fraction of the datasets. While it may be tempting to think that we can somehow exclude the bad dataset as part of the  $\delta$ fraction, there is a significant catch here: we can not carve out a $\delta$ fraction specific to each hypothesis; we can ignore only a single chunk of $\delta$ mass common to all hypotheses in $\mathcal{H}_{\delta}$.  This restriction turns out to be a tremendous bottleneck: despite ignoring this $\delta$ fraction, for most $h_{\tilde{S}} \in \mathcal{H}_{\delta}$, the corresponding bad set $\tilde{S}'$ would still be left in  $\mathcal{S}_{\delta}$. Then, for all such $h_{\tilde{S}}$, $\mathcal{L}_{\mathcal{D}}(h_{\tilde{S}})$ would be small but $\hat{\mathcal{L}}_{S}(h_{\tilde{S}}) $ large; we can then set the $S$ inside the $\sup_{S \in \mathcal{S}_{\delta}}$ to be $\tilde{S}'$ to conclude that 
 $\epsilon_{\text{unif-alg}}$ is indeed vacuous. This is the kind of failure we will demonstrate in a high-dimensional linear classifier in the following section, and a ReLU neural network in Section~\ref{sec:hypersphere}, and an infinitely wide exponential-activation neural network in Appendix~\ref{sec:exp} -- all trained by GD or SGD. \footnote{In Appendix~\ref{sec:warm-up}, the reader can find a more abstract setting illustrating this mathematical idea more clearly.}    

\textbf{Note:} Our results about failure of uniform convergence holds even for bounds that output a different value for each hypothesis. In this case, the tightest uniform convergence bound for a given hypothesis would be at least as large as $\sup_{S \in \mathcal{S}_{\delta}} |\mathcal{L}_{\mathcal{D}}(h_{\tilde{S}}) - \hat{\mathcal{L}}_{S}(h_{\tilde{S}})|$ which by a similar argument would be vacuous for most draws of the training set $\tilde{S}$. We discuss this in more detail in Appendix~\ref{sec:weight-dependence}.

 
%
%

\subsection{High-dimensional linear classifier}
\label{sec:linear}

\subparagraph{Why a linear model?}  Although we present a neural network example in the next section, we first emphasize why it is also important to understand how uniform convergence could fail for linear classifiers trained using GD. First, it is more natural to expect uniform convergence to yield poorer bounds in more complicated classifiers; linear models are arguably the simplest of classifiers, and hence showing failure of uniform convergence in these models is, in a sense, the most interesting.
Secondly, recent works (e.g., \citep{jacot18ntk}) have shown that as the width of a deep network goes to infinity, under some conditions, the network converges to a high-dimensional linear model (trained on a high-dimensional transformation of the data) -- thus making the study of high-dimensional linear models relevant to us. Note that our example is not aimed at modeling the setup of such linearized neural networks. However, 
it does provide valuable intuition about the mechanism by which uniform convergence fails, and we show how this extends to neural networks in the later sections.


\textbf{Setup.} Let each input be a $K+D$ dimensional vector (think of $K$  as a small constant and $D$ much larger than $m$). The value of any input $\vec{x}$ is denoted by $(\vec{x}_1, \vec{x}_2)$ where $\vec{x}_1 \in \mathbb{R}^K$ and $\vec{x}_2 \in \mathbb{R}^D$. Let the centers of the (two) classes be determined by an arbitrary vector $\vec{u} \in \mathbb{R}^K$ such that $\|\vec{u}\|_2=1/\sqrt{m}$.  Let $\mathcal{D}$ be such that the label $y$ has equal probability of being $+1$ and $-1$, and $\vec{x}_1 = 2\cdot y \cdot \vec{u}$ while $\vec{x}_2$ is sampled independently from a spherical Gaussian, $\mathcal{N}(0,\frac{32}{D}I)$.\footnote{As noted in Appendix~\ref{sec:remark}, it is easy to extend the discussion by assuming that $\vec{x}_1$ is spread out around $2y\vec{u}$.} Note that the distribution is linearly separable based on the first few ($K$) dimensions. 
For the learning algorithm $\mathcal{A}$, consider a linear classifier with weights $\vec{w} = (\vec{w}_1, \vec{w}_2)$ and whose output is $h(\vec{x}) = \vec{w}_1 \vec{x}_1 + \vec{w}_2 \vec{x}_2$. Assume the weights are initialized to the origin. Given a dataset $S$, $\mathcal{A}$ takes a gradient step of learning rate $1$ to maximize $y \cdot h(\vec{x})$ for each $(\vec{x},y) \in S$. Hence, regardless of the batch size, the learned weights would satisfy,  $\vec{w}_1 = 2m \vec{u}$ and $\vec{w}_2 = \sum_i y^{(i)} \vec{x}_2^{(i)}$.   Note that effectively $\vec{w}_1$ is aligned correctly along the class boundary while $\vec{w}_2$ is high-dimensional Gaussian noise.  It is fairly simple to show that this algorithm achieves {\em zero training error} for most draws of the training set. At the same time, for this setup, we have the following lower bound on uniform convergence for the $\mathcal{L}^{(\gamma)}$ loss:\footnote{While it is obvious from Theorem~\ref{thm:example} that the bound is nearly vacuous for any $\gamma \in [0,1]$,  in Appendix~\ref{sec:any-gamma}, we argue that even for any $\gamma \geq 1$, the guarantee is nearly vacuous, although in a slightly different sense. 
} \\


\begin{theorem}
\label{thm:example}
For any $\epsilon,\delta > 0, \delta \leq 1/4$, when $D = \Omega\left(\max\left( m \ln \frac{m}{\delta}, m \ln\frac{1}{\epsilon}\right) \right)$, $\gamma \in [0,1]$,   the $\mathcal{L}^{(\gamma)}$ loss satisfies $\epsilon_{\text{gen}}(m,\delta) \leq \epsilon$, while $\epsilon_{\text{unif-alg}}(m,\delta) \geq 1- \epsilon$. Furthermore,
 for all $\gamma \geq 0$, for the $\mathcal{L}^{(\gamma)}$ loss, $\epsilon_{\text{unif-alg}}(m,\delta) \geq 1- \epsilon_{\text{gen}}(m,\delta)$.
\end{theorem}

In other words, even the tightest uniform convergence bound is nearly vacuous despite good generalization. In order to better appreciate the implications of this statement, it will be helpful to look at the bound a standard technique would yield here. For example, the Rademacher complexity of the class of $\ell_2$-norm bounded linear classifiers would yield a bound of the form $\mathcal{O}(\|\vec{w}\|_2/(\gamma^\star\sqrt{m}))$  where $\gamma^\star$ is the margin on the training data. 
In this setup, the weight norm grows with dataset size as $\| \vec{w}\|_2 = \Theta(\sqrt{m})$ (which follows from the fact that $\vec{w}_2$ is a Gaussian with $m/D$ variance along each of the $D$ dimensions) and $\gamma^{\star} = \Theta(1)$. Hence, the Rademacher bound here would evaluate to a constant much larger than $\epsilon$.  One might persist and think that 
perhaps, the characterization of $\vec{w}$ to be bounded in $\ell_2$ norm does not  fully capture the implicit bias of the algorithm. Are there other properties of the Gaussian $\vec{w}_2$ that one could take into account to identify an even smaller class of hypotheses for which uniform convergence may work after all? Unfortunately, our statement rules this out:  even after fixing $\vec{w}_1$ to the learned value ($2m\vec{u}$) and for any possible $1-\delta$ truncation of the Gaussian $\vec{w}_2$, the resulting  pruned class of weights -- despite all of them having a test error less than $\epsilon$ -- would give only nearly vacuous uniform convergence bounds as $\epsilon_{\text{unif-alg}}(m,\delta) \geq 1-\epsilon$.

\subparagraph{Proof outline.}
We now provide an outline of our argument for Theorem~\ref{thm:example}, deferring the proof to the appendix. First, the small generalization (and test) error arises from the fact that $\vec{w}_1$ is aligned correctly along the true boundary; at the same time, the noisy part of the classifier $\vec{w}_2$ is poorly aligned with at least $1-\epsilon$ mass of the test inputs, and hence does not dominate the output of the classifier on test data -- preserving the good fit of $\vec{w}_1$ on the test data. On the other hand, at a very high level, under the purview of uniform convergence, we can argue that the noise vector $\vec{w}_2$ is effectively stripped of its randomness. This misleads uniform convergence into believing that the $D$ noisy dimensions (where $D > m$) contribute meaningfully to
the representational complexity of the classifier, thereby giving nearly vacuous bounds. We describe this more concretely below.

As a key step in our argument, we show that w.h.p over draws of $S$, even though the learned classifier $h_S$ correctly classifies most of the randomly picked test data, it completely misclassifies a ``bad'' dataset, namely $S'= \{ ((\vec{x}_1, -\vec{x}_2),y) \; | \; (\vec{x},y) \in S \}$ which is the noise-negated version of $S$.  Now recall that to compute $\epsilon_{\text{unif-alg}}$ one has to begin by picking a sample set space $\mathcal{S}_{\delta}$ of mass $1-\delta$. We first argue that for {\em any} choice of $\mathcal{S}_{\delta}$, there must exist $S_\star$ such that all the following four events hold: (i) $S_\star \in \mathcal{S}_{\delta}$, (ii) the noise-negated $S_\star' \in \mathcal{S}_{\delta}$, (iii) $h_{S_\star}$ has test error less than $\epsilon$ and (iv) $h_{S_\star}$ completely misclassifies $S_{\star}'$. We prove the existence of such an $S_\star$ by arguing that over draws from $\mathcal{D}^m$, there is non-zero probability of picking a dataset that satisfies these four conditions. 
Note that our argument for this crucially makes use of the fact that we have designed the ``bad'' dataset in a way that it has the same distribution as the training set, namely $\mathcal{D}^m$.
Finally, for a given $\mathcal{S}_{\delta}$, if we have an $S_\star$ satisfying (i) to (iv), we can prove our claim as $\epsilon_{\text{unif-alg}}(m,\delta)  =
 \sup_{S \in \mathcal{S}_{\delta}} \sup_{h \in \mathcal{H}_{\delta}} |{\mathcal{L}}_{\mathcal{D}}(h)- \hat{\mathcal{L}}_{S}(h)| 
 \geq  |{\mathcal{L}}_{\mathcal{D}}(h_{S_\star})-\hat{\mathcal{L}}_{S_\star'}(h_{S_\star})|= |\epsilon-1| = 1-\epsilon$. \\


\begin{remark} Our analysis depends on the fact that $\epsilon_{\text{unif-alg}}$ is a two-sided convergence bound -- which is what existing techniques bound -- and our result would not apply for hypothetical one-sided uniform convergence bounds. 
While PAC-Bayes based bounds are typically presented as one-sided bounds, we show in Appendix~\ref{sec:pac-bayes} that even these are lower-bounded by the two-sided  $\epsilon_{\text{unif-alg}}$. To the best of our knowledge, it is non-trivial to make any of these tools purely one-sided. \\
\end{remark} 

\begin{remark} The classifier modified by  setting $\vec{w}_2 \gets 0$, has small test error
and also enjoys non-vacuous bounds as it has very few parameters.  However, such a bound would not fully explain why the original classifier generalizes well. One might then wonder if such a bound could be extended to the original classifier, like it was explored in \citet{nagarajan2018deterministic} for deep networks. Our result implies that no such extension is possible in this particular example. \\
\end{remark}

 \subsection{ReLU neural network}
  \label{sec:hypersphere} 
We now design a non-linearly separable task (with no ``noisy'' dimensions) where a sufficiently wide ReLU network trained in the standard manner, like in the experiments of Section~\ref{sec:experiments} leads to failure of uniform convergence.  For our argument, we will rely on a classifier trained {\em empirically}, in contrast to our linear examples where we rely on an analytically derived expression for the learned classifier. 
 Thus, this section illustrates that the effects we modeled theoretically in the linear classifier \emph{are} indeed reflected in typical training settings, even though here it is difficult to precisely analyze the learning process. We also refer the reader to Appendix~\ref{sec:exp}, where we present an example of a neural network with exponential activation functions for which we do derive a closed form expression.

\textbf{Setup.} We consider a distribution that was originally proposed in \cite{gilmer18adversarial} as the ``adversarial spheres'' dataset (although with slightly different hyperparameters) and was used to study the independent phenomenon of adversarial examples. Specifically, we consider 1000-dimensional data, where two classes are distributed uniformly over two origin-centered hyperspheres with radius $1$ and $1.1$ respectively.  We vary the number of training examples from $4k$ to $65k$ (thus ranging through typical dataset sizes like that of MNIST). Observe that compared to the linear example, this data distribution is more realistic in two ways. First, we do not have specific dimensions in the data that are noisy and second, the data dimensionality here as such is a constant less than $m$. Given samples from this distribution, we  train a two-layer ReLU network with $h=100k$ to minimize cross entropy loss using SGD with learning rate $0.1$ and batch size $64$. We train the network until $99\%$ of the data is classified by a margin of $10$.

As shown in Figure~\ref{fig:hypersphere} (blue line), in this setup, the 0-1 error  (i.e., $\mathcal{L}^{(0)}$) as approximated by the test set, decreases with $m \in [2^{12}, 2^{16}]$ at the rate of  $O(m^{-0.5})$. 
Now, to prove failure of uniform convergence, we empirically show that a completely misclassified ``bad'' dataset $S'$ can be constructed in a manner similar to that of the previous example. In this setting, we pick $S'$ by simply projecting every training datapoint on the inner hypersphere onto the outer and vice versa, and then flipping the labels. Then, as shown in Figure~\ref{fig:hypersphere} (orange line), $S'$ is completely misclassified by the learned network. Furthermore, like in the previous example, we have  $S' \sim \mathcal{D}^m$ because the distributions are uniform over the hyperspheres. Having established these facts, the rest of the argument follows like in the previous setting, implying failure of uniform convergence as in Theorem~\ref{thm:example} here too.

\begin{figure}[t]
    \centering
    \begin{minipage}[t]{0.3\textwidth}
        \adjincludegraphics[width=1\textwidth,trim={0 {0} 0 0},clip,,valign=t]{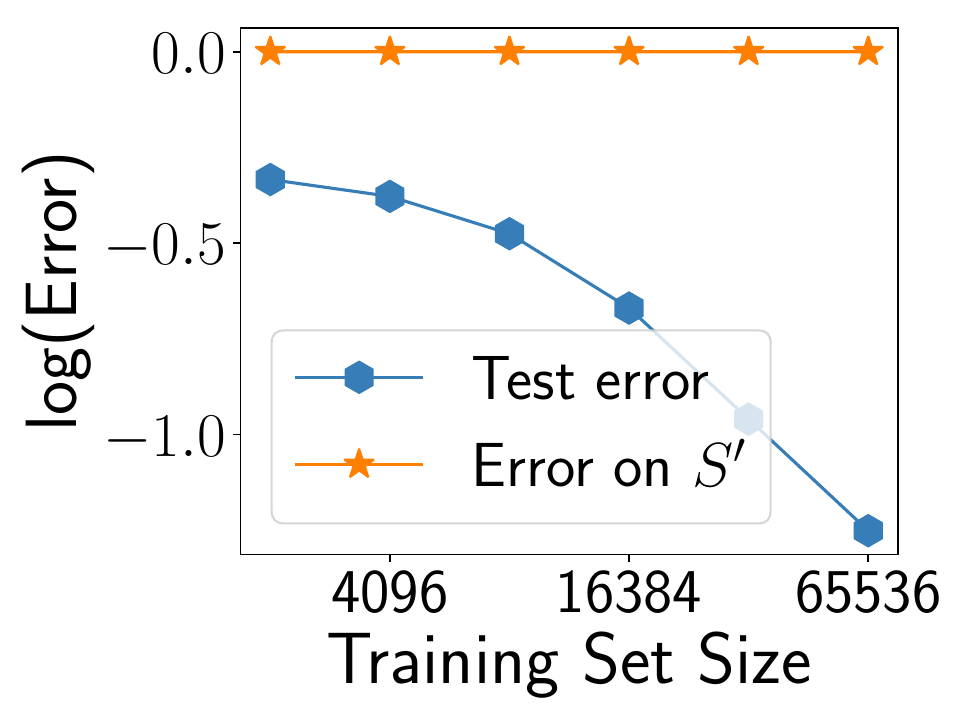} %
    \end{minipage}
    \begin{minipage}[t]{0.6\textwidth}
        \centering
        \adjincludegraphics[width=\textwidth,trim={0 {0} 0 0},clip,,valign=t]{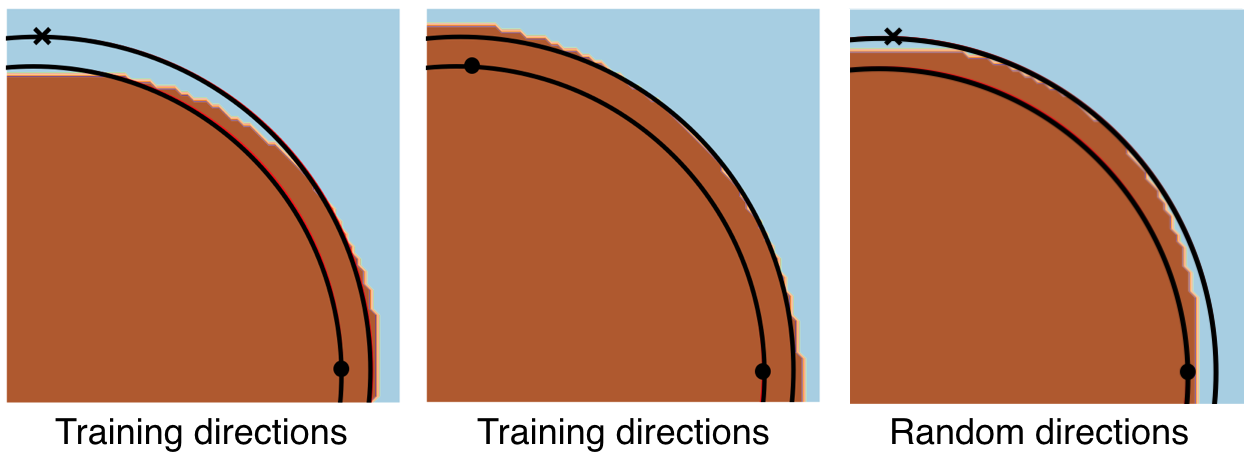} %
    \end{minipage}
       \caption{In the \textbf{first} figure, we plot the error of the ReLU network on test data and on the bad dataset $S'$, in the task described in Section~\ref{sec:hypersphere}. 
       The \textbf{second and third} images correspond to the decision boundary learned in this task, in the 2D quadrant containing two training datapoints (depicted as $\boldsymbol{\small \times}$ and $\bullet$). The black lines correspond to the two hyperspheres, while the brown and blue regions correspond to the class output by the classifier.  Here, we observe that  the boundaries are skewed around the training data in a way that it misclassifies the nearest point from the opposite class (corresponding to $S'$, that is not explicitly marked). The \textbf{fourth} image corresponds to two random (test) datapoints, where the boundaries are fairly random, and very likely to be located in between the hyperspheres (better confirmed by the low test error).} 
       \label{fig:hypersphere}
\end{figure}

In Figure~\ref{fig:hypersphere}  (right), we visualize how the learned boundaries are skewed around the training data in a way that $S'$ is misclassified. Note that $S'$ is misclassified even when it has as many as $60k$ points, and even though the network was not explicitly trained to misclassify those points. Intuitively, this demonstrates that the boundary learned by the ReLU network has sufficient complexity that hurts uniform convergence while not affecting the generalization error, at least in this setting. We discuss the applicability of this observation to other hyperparameter settings in Appendix~\ref{sec:applicability}. 

\subparagraph{Relationship to adversarial spheres \cite{gilmer18adversarial}.} While we use the same adversarial spheres distribution as \cite{gilmer18adversarial} and similarly show the existence a certain kind of an adversarial dataset, it is important to note that neither of our observations implies the other. Indeed, the observations in \cite{gilmer18adversarial} are insufficient to prove failure of uniform convergence. Specifically, \cite{gilmer18adversarial} show that in the adversarial spheres setting, it is possible to slightly perturb random {\em test} examples in some {\em arbitrary} direction to discover a misclassified example. However, to show failure of uniform convergence, we need to find a set of misclassified examples $S'$ corresponding to the {\em training} examples $S$, and furthermore, we do {\em not} want $S'$ to be arbitrary. We want $S'$ to have the same underlying distribution, $\mathcal{D}^m$.

\subparagraph{Deep learning conjecture.} 
Extending the above insights more generally, we conjecture that in overparameterized deep networks, SGD finds a fit that is simple at a macroscopic level (leading to good generalization) but also has many microscopic fluctuations (hurting uniform convergence). 
To make this more concrete, for illustration, consider the high-dimensional linear model that sufficiently wide networks have been shown to converge to \cite{jacot18ntk}. That is, roughly, these networks can be written as $h(\vec{x}) = \vec{w}^T \vec{\phi}(\vec{x})$ where $\vec{\phi}(\vec{x})$
is a rich high-dimensional representation of $\vec{x}$ computed from many random features (chosen independent of training data).
Inspired by our linear model  in Section~\ref{sec:linear}, we conjecture that the weights $\vec{w}$ learned on a dataset $S$ can be expressed as $\vec{w}_1 + \vec{w}_2$, where $\vec{w}_1^T\vec{\phi}(\vec{x})$ dominates the output on most test inputs and induces a simple decision boundary. That is, it may be possible to apply uniform convergence on the function $\vec{w}_1^T \vec{\phi}(\vec{x})$ to obtain a small generalization bound.  On the other hand, $\vec{w}_2$ corresponds to meaningless signals that gradient descent gathered from the high-dimensional representation of the training set $S$. Crucially, these signals would be specific to $S$, and hence not likely to correlate with most of the test data i.e., 
$\vec{w}_2\phi(\vec{x})$  would be negligible on most test data, thereby not affecting the generalization error significantly. However, $\vec{w}_2\phi(\vec{x})$ can still create complex fluctuations on the boundary, in low-probability regions of the input space (whose locations would depend on $S$, like in our examples). As we argued, this can lead to failure of uniform convergence. Perhaps, existing works that have achieved strong uniform convergence bounds on modified networks, may have done so by implicitly suppressing  $\vec{w}_2$, either by compression, optimization or stochasticization. Revisiting these works may help verify our conjecture.

\section{Conclusion and Future Work}

A growing variety of uniform convergence based bounds
 \cite{neyshabur15norm,bartlett17spectral,golowich17size,arora18compression,
neyshabur18pacbayes,dziugaite17nonvacuous,zhou2018nonvacuous,
li18learning,zhu18beyond,nagarajan2018deterministic,neyshabur18unitwise}
have sought to explain generalization in deep learning. While these may provide partial intuition about the puzzle, we ask a critical, high level question: by pursuing this broad direction, is it possible to achieve the grand goal of a small generalization bound that shows appropriate dependence on the sample size, width, depth, label noise, and batch size? We cast doubt on this by first, empirically showing that existing bounds can surprisingly increase with training set size for small batch sizes. We then presented example setups, including that of a ReLU neural network, for
 which uniform convergence provably fails to explain generalization, even after taking implicit bias into account.


Future work in understanding implicit regularization in deep learning may be better guided with our knowledge of the sample-size-dependence in the weight norms. To understand generalization, it may also be promising to explore other learning-theoretic techniques like, say, algorithmic stability \cite{feldman18uniformly,hardt16stability,bousquet02stability,shwartz10learnability}
;  our linear setup might also inspire new tools. Overall,  through our work, we call for going beyond uniform convergence to fully explain generalization in deep learning.




\subparagraph{Acknowledgements.} Vaishnavh Nagarajan is supported by a grant from the Bosch Center for AI.

\bibliographystyle{plainnat}
\bibliography{references.bib}
\clearpage

\appendix

\section{Summary of existing generalization bounds.}
\label{app:table}

In this section, we provide an informal summary of the properties of (some of the) existing generalization bounds for ReLU networks  in Table~\ref{tab:summary}.

\begin{table}[H]
\begingroup
\fontsize{7.5pt}{12pt}\selectfont
\begin{center}
\begin{tabular}{|c|c|c|c|c|}
\hline
\textbf{Bound} & \shortstack{Norm \\dependencies} & \shortstack{Parameter-count\\ dependencies} & \shortstack{Numerical \\ value} & \shortstack{Holds on \\ original \\ network?} \\ \hline
\citet{harvey17vc} & - & depth $\times$ width & Large & Yes \\ \hline
\shortstack{\citet{bartlett17spectral} \\ \citet{neyshabur18pacbayes}}& \shortstack{Product of spectral norms \\ dist. from init. \\ (not necessarily $\ell_2$)} & \shortstack{poly(width) \\ exp(depth)} & Large & Yes \\ \hline
\shortstack{\citet{neyshabur15norm} \\ \citet{golowich17size}}& \shortstack{Product of Frobenius norms \\ $\ell_2$ dist. from init.} & $\sqrt{\text{width}}^\text{depth}$ & Very large & Yes \\ \hline
\shortstack{\citet{nagarajan2018deterministic}}& \shortstack{Jacobian norms \\ $\ell_2$ dist. from init. \\ \text{Inverse pre-activations}} & \shortstack{poly(width) \\ poly(depth)} & \shortstack{Inverse \\ pre-activations \\ can be very large}& Yes \\ \hline
\shortstack{\citet{neyshabur18unitwise}\\ for two-layer networks } & \shortstack{Spectral norm ($1$st layer) \\ $\ell_2$ Dist. from init ($1$st layer) \\ Frobenius norm ($2$nd layer)} & $\sqrt{\text{width}}$ & {Small} & Yes \\ \hline
\citet{arora18compression} & \shortstack{Jacobian norms \\ dist. from init.} & \shortstack{poly(width) \\ poly(depth)}& Small & \shortstack{No. Holds on \\ compressed network} \\ \hline
\citet{dziugaite17nonvacuous} & \shortstack{dist. from init. \\ Noise-resilience of network} &  - & \shortstack{Non-vacuous\\ on MNIST} & \shortstack{No. Holds on an \\optimized, stochastic \\ network}\\ \hline
\citet{zhou2018nonvacuous} & \shortstack{Heuristic compressibility \& \\ noise-resilience of network} &  - &  \shortstack{Non-vacuous\\on ImageNet} & \shortstack{No. Holds on an \\ optimized, stochastic, \\ heuristically compressed, \\ network}\\ \hline
\citet{zhu18beyond} & \shortstack{$L_{2,4}$ norm ($1$st layer) \\ Frobenius norm ($2$nd layer)} & - & \shortstack{Small for carefully\\  scaled init.\\  and learning rate} & Yes\\ \hline
\citet{li18learning} & - & - & \shortstack{Small for carefully\\  scaled batch size \\ and learning rate} & Yes \\ \hline 
\end{tabular}
\end{center}
\vspace{5pt}
\caption{Summary of generalization bounds for ReLU networks. We note that the analysis in \citet{li18learning} relies on a sufficiently small learning rate ($\approx \mathcal{O}(1/m^{1.2})$) and large batch size ($\approx \Omega(\sqrt{m})$). Hence, the resulting bound cannot describe how generalization varies with any other hyperparameter, like training set size or width, with everything else fixed. A similar analysis in \citet{zhu18beyond} requires fixing the learning rate to be inversely proportional to width. Their bound decreases only as ${\Omega}(1/m^{0.16})$, although, the actual generalization error is typically as small as $\mathcal{O}(1/m^{0.43})$. } \label{tab:summary}
\endgroup
\end{table}

\section{More Experiments}
\label{app:experiments}
In this section, we present more experiments along the lines of what we presented in Section~\ref{sec:experiments}.  

\subparagraph{Layerwise dependence on $m$.} Recall that in the main paper, we show how the distance from initialization and the product of spectral norms vary with $m$ for network with six layers. In Figure~\ref{fig:depth-wise}, we show how the terms grow with sample size $m$ for each layer individually. Our main observation is that the first layer suffers from the largest dependence on $m$.

\begin{figure}[h]
    \centering
            \begin{minipage}{.35\textwidth}
        \adjincludegraphics[width=1\textwidth,trim={0 {0} 0 0},clip,,valign=t]{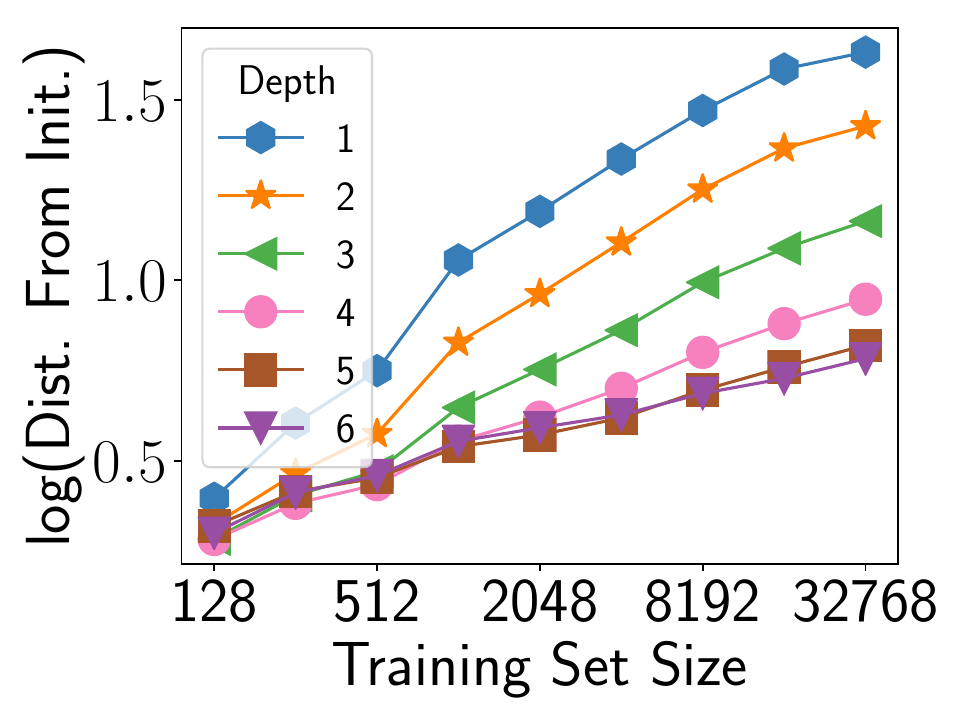} %
    \end{minipage}%
                \begin{minipage}{.35\textwidth}
        \adjincludegraphics[width=1\textwidth,trim={0 {0} 0 0},clip,valign=t]{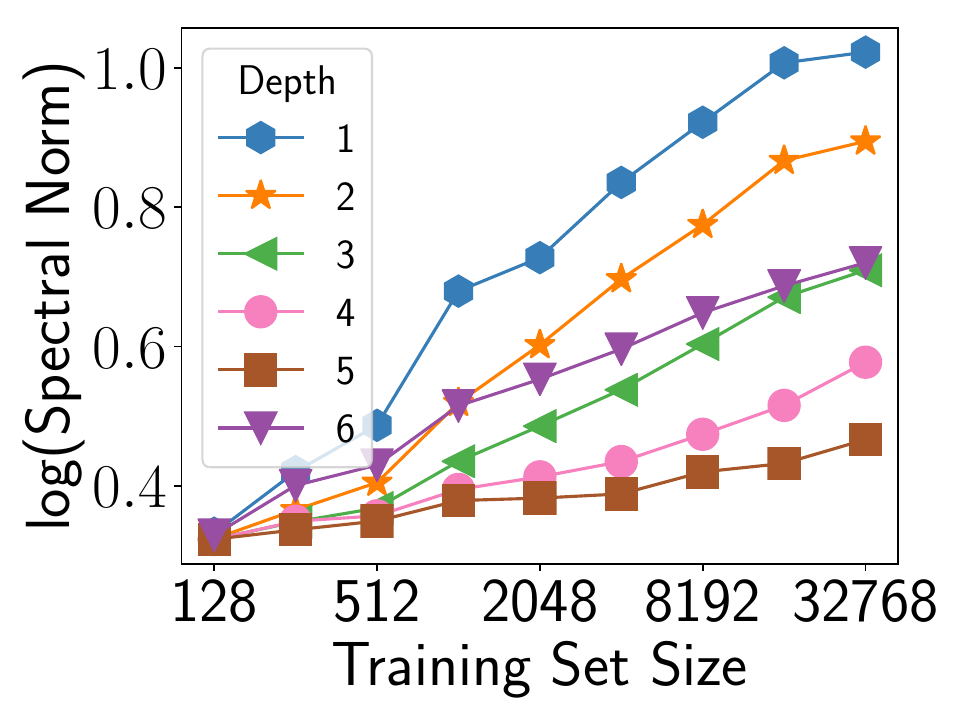} %
    \end{minipage}%
       \caption{ We plot the distance from initialization and the spectral norm of each individual layer, and observe that the lowermost layer shows the greatest dependence on $m$.  
       }      \label{fig:depth-wise}
\end{figure}   

\subparagraph{Distance between trajectories of shuffled datasets grows with $m$.}
In the main paper, we saw that the distance between the solutions learned on different draws of the dataset grow substantially with $m$. In Figure~\ref{fig:more-evidence-dist} (left), we show that even the distance between the solutions learned on the same draw, but a different shuffling of the dataset grows substantially with $m$.

\subparagraph{Flat minima}

We also relate our observations regarding distance between two independently learned weights to the popular idea of ``flat minima''. Interestingly, Figure~\ref{fig:flat-minima} demonstrates that walking linearly from the weights learned on one dataset draw to that on another draw (from the same initialization) preserves the test error. Note that although a similar observation was made in \citet{felix18barriers,garipov18loss},  they show the existence of \textit{non-linear} paths of good solutions between parameters learned from \textit{different} initializations. Our observation on the other hand implies that for a fixed initialization, SGD explores the \textit{same} basin in the test loss minimum across different training sets. As discussed in the main paper, this explored basin/space has larger $\ell_2$-width for larger $m$ giving rise to a ``paradox'': on one hand, wider minima are believed to result in, or at least correlate with better generalization \citep{hochreiter97flat,hinton93mdl,keskar17largebatch}, but on the other, a larger $\ell_2$-width of the explored space results in larger uniform convergence bounds, making it harder to explain generalization. 

We note a similar kind of paradox concerning noise in training. Specifically,
it is intriguing that on one hand, generalization is aided by larger learning rates and smaller batch sizes \cite{jastrzebski18width,hoffer17longer,keskar17largebatch} due to increased noise in SGD. On the other, theoretical analyses benefit from the opposite;  \citet{zhu18beyond} even explicitly regularize SGD for their three-layer-network result  to help ``forget false information'' gathered by SGD. In other words, it seems that {noise aids generalization, yet hinders attempts at explaining generalization}. The intuition from our examples (such as the linear example) is that such ``false information'' could provably impair uniform convergence without affecting generalization.

  \begin{figure}
        \centering
        \adjincludegraphics[width=0.35\textwidth,trim={0 {0} 0 0},clip,,valign=t]{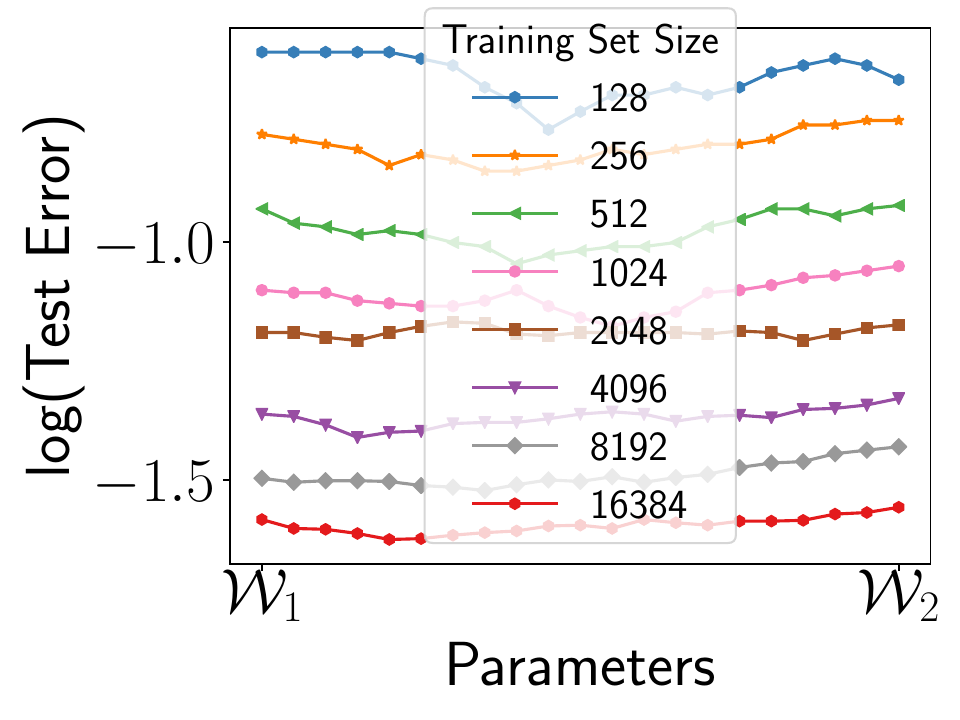} %
        \caption{We plot the test errors of the networks that lie on the straight line  between two weights learned on two independent random draws of training data starting from the same initialization. We observe that all these intermediate networks have the same test error as the original networks themselves.}
                \label{fig:flat-minima}
\end{figure}

\subparagraph{Frobenius norms grow with $m$ when $m \gg h$.} Some bounds like \citep{golowich17size} depend on the Frobenius norms of the weight matrices (or the distance from origin), which as noted in \cite{nagarajan17role} are in fact width-dependent, and grow as $\Omega(\sqrt{h})$. However, even these terms do grow with the number of samples in the regime where $m$ is larger than $h$. In Figure~\ref{fig:more-evidence-dist}, we report the total distance from origin of the learned parameters for a network with $h=256$ (we choose a smaller width to better emphasize the growth of this term with $m$); here, we see that for $m > 8192$, the distance from origin grows at a rate of $\Omega(m^{0.42})$ that is quite similar to what we observed for distance from initialization. 

\begin{figure}[!h]
    \centering
        \begin{minipage}{.35\textwidth}
        \centering
        \adjincludegraphics[width=1\textwidth,trim={0 {0} 0 0},clip,,valign=t]{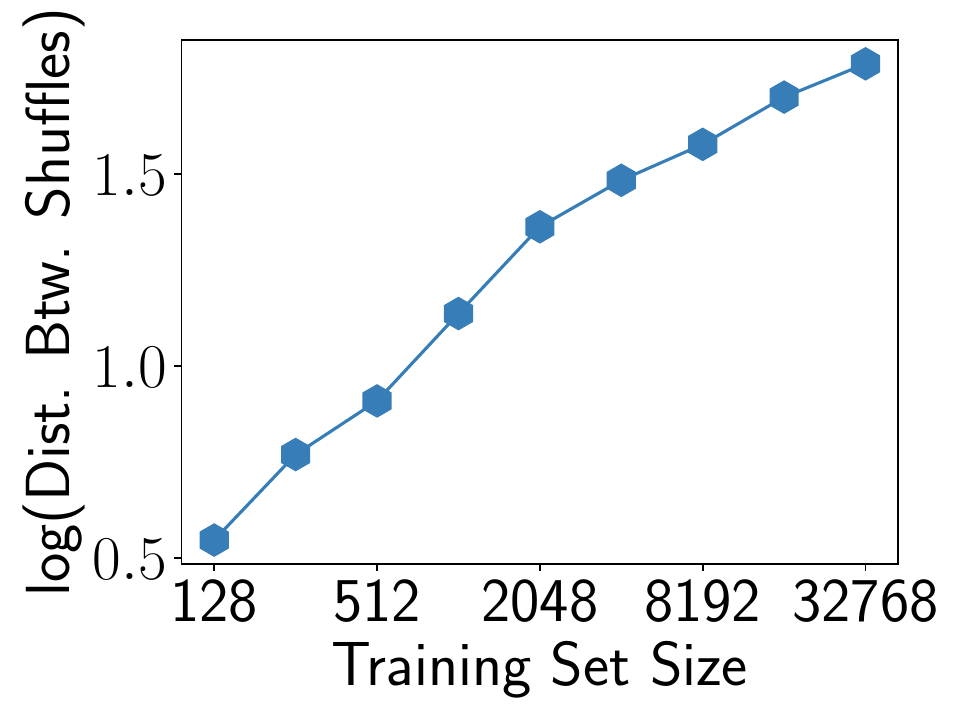} %
    \end{minipage}%
            \begin{minipage}{.35\textwidth}
        \centering
        \adjincludegraphics[width=1\textwidth,trim={0 {0} 0 0},clip,valign=t]{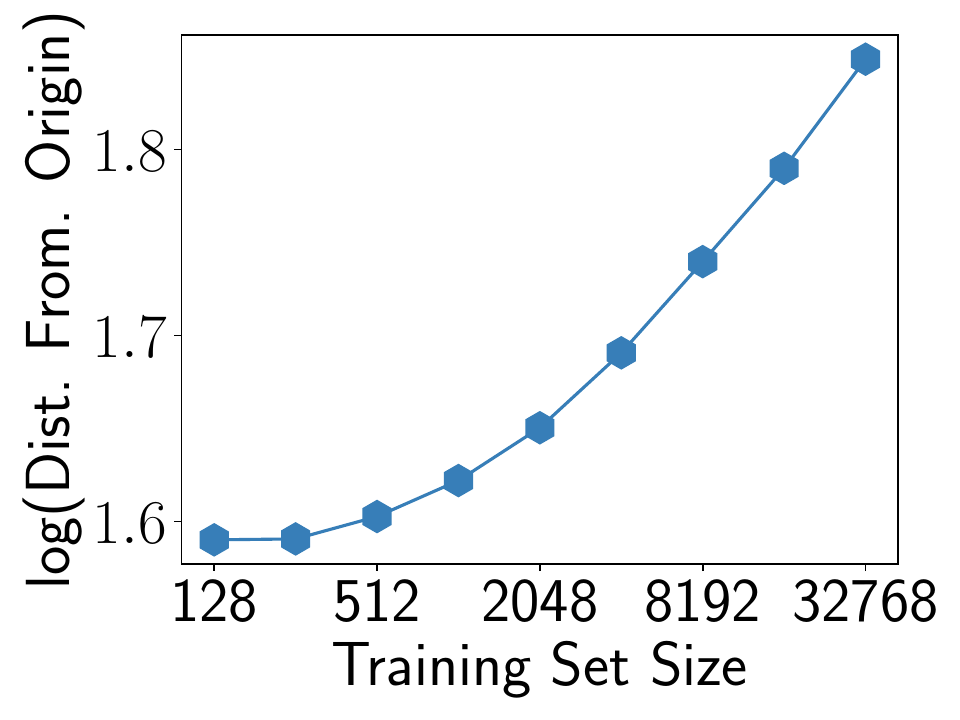} %
    \end{minipage}
       \caption{On the \textbf{left}, we plot the distance between the weights learned on the two different shuffles of the same dataset, and it grows as fast as the distance from initialization. On the \textbf{right}, we plot the distance of the weights from the origin, learned for a network of width $h=256$ and depth $d=6$;   for sufficiently large $m$, this grows as $\Omega(m^{0.42})$.  
       }      \label{fig:more-evidence-dist}
\end{figure}

\subparagraph{Even a relaxed notion of margin does not address the $m$-dependency.} 
Recall that in the main paper, we computed the generalization error bound in Equation~\ref{eq:gen-error-bound} by setting $\gamma$ to be $\gamma^\star$, the margin achieved by the network on at least $99\%$ of the data. One may hope that by choosing a larger value of $\gamma$, this bound would become smaller, and that the $m$-dependence may improve. 
 We consider this possibility by computing the  median margin of the network over the training set (instead of the $1\%$-percentile'th margin) and substituting this in the second term in the right hand side of the guarantee in Equation~\ref{eq:gen-error-bound}. By doing this, the first margin-based train error term in the right hand side of Equation~\ref{eq:gen-error-bound} would simplify to $0.5$ (as half the training data are misclassified by this large margin). Thereby we already forgo an explanation of half of the generalization behavior. At least we could hope that the second term no longer grows with $m$. Unfortunately, we observe in Figure~\ref{fig:median-margin} (left) that the bounds still grow with $m$.  This is because, as shown in Figure~\ref{fig:median-margin} (right), the median margin value does not grow as fast with $m$ as the numerators of these bounds grow.

\begin{figure}[t]
    \centering
       \begin{minipage}{.35\textwidth}
        \centering
        \adjincludegraphics[width=1\textwidth,trim={0 {0} 0 0},clip,valign=t]{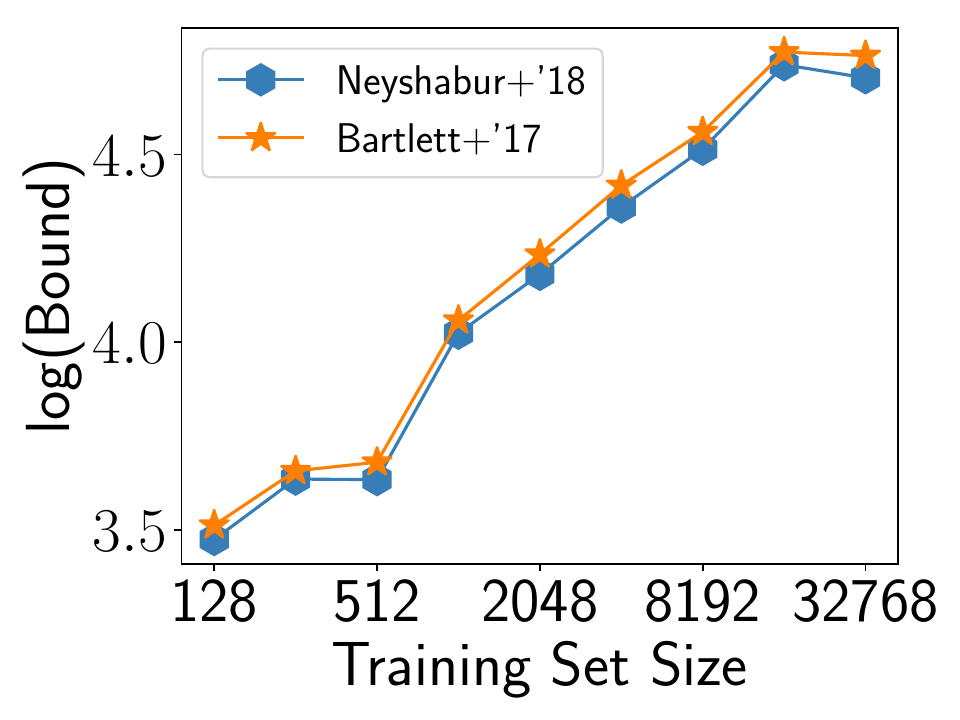} %
    \end{minipage}%
        \begin{minipage}{.35\textwidth}
        \centering
                \adjincludegraphics[width=1\textwidth,trim={0 {0} 0 0},clip,,valign=t]{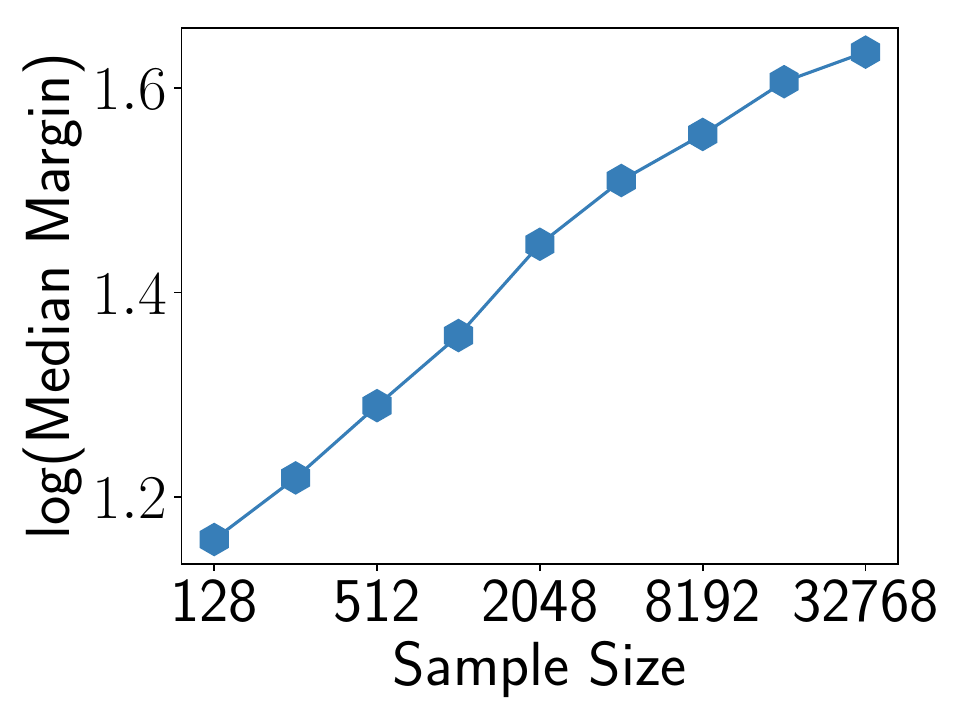} 
    \end{minipage}%
       \caption{
        In the \textbf{left} plot, we plot the bounds after setting $\gamma$ to be the median margin on the training data -- these bounds grow as $\Omega(m^{0.48})$.  In the \textbf{right} plot the median value of the margin $\Gamma(f(\vec{x}),y)$ on the training dataset and observe that it grows as $\mathcal{O}(m^{0.2})$. 
       } \label{fig:median-margin}
\end{figure}

\subparagraph{Effect of depth.}

We observed that as the network gets shallower the bounds show better dependence with $m$. As an extreme case, we consider a network with only one hidden layer, and with $h=50000$. Here we also present a third bound, namely that of \citet{neyshabur18unitwise}, besides the two bounds discussed in the main paper. Specifically, if $Z_1, Z_2$ are the random initializations of the weight matrices in the network, the generalization error bound (the last term in Equation~\ref{eq:gen-error-bound}) here is of the following form, ignoring log factors:

\[
\frac{\|W_2 \|_F (\|W_1- Z_1 \|_F + \| Z_1\|_2)}{\gamma \sqrt{m}} + \frac{\sqrt{h}}{\sqrt{m}.
}\]
The first term here is meant to be width-independent, while the second term clearly depends on the width and does decrease with $m$ at the rate of $m^{-0.5}$. Hence, in our plots in Figure~\ref{fig:one-layer}, we only focus on the first term. We see that these bounds are almost constant and decrease at a minute rate of $\Omega(m^{-0.066})$ while the test errors decrease much faster, at the rate of $\mathcal{O}(m^{-0.35})$.

\begin{figure}[!h]
    \centering
        \begin{minipage}{.35\textwidth}
        \centering
        \adjincludegraphics[width=1\textwidth,trim={0 {0} 0 0},clip,,valign=t]{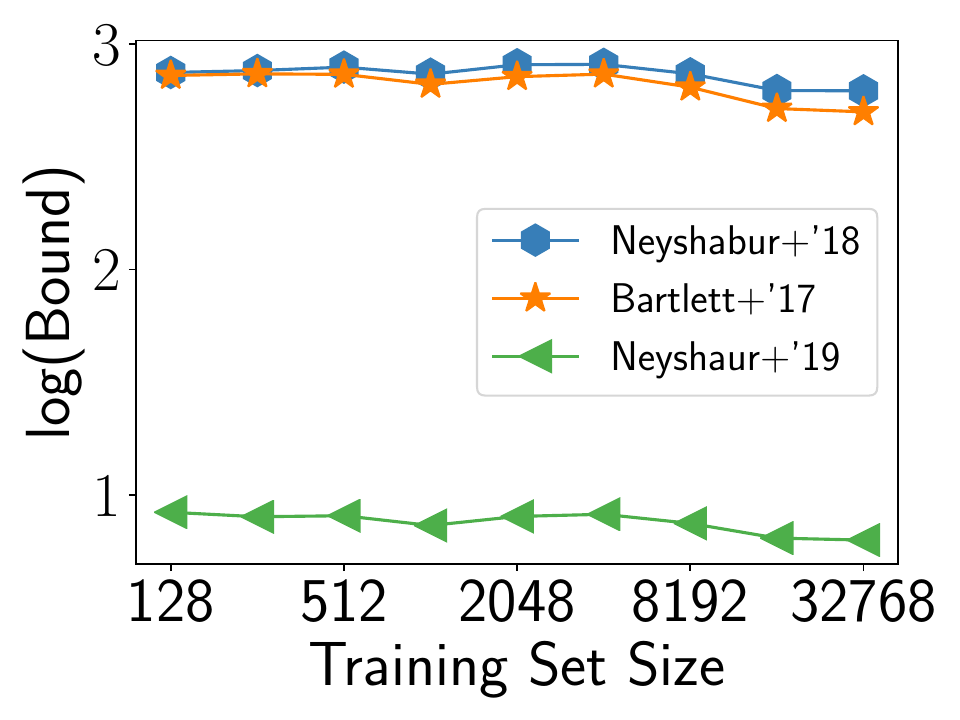} %
    \end{minipage}%
            \begin{minipage}{.35\textwidth}
        \centering
        \adjincludegraphics[width=1\textwidth,trim={0 {0} 0 0},clip,valign=t]{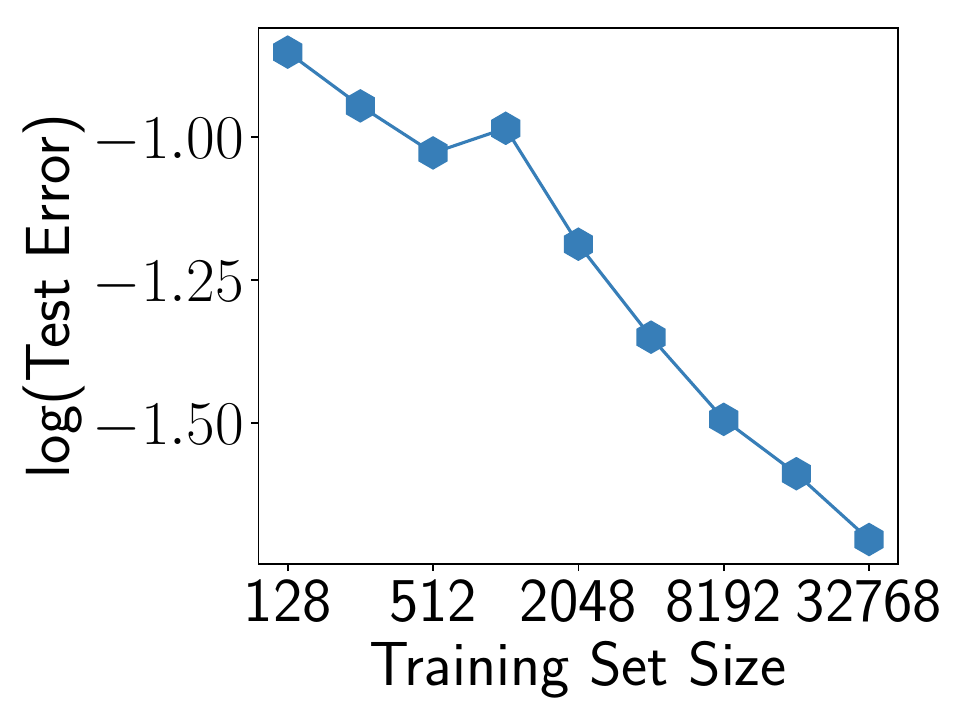} %
    \end{minipage}
       \caption{On the \textbf{left}, we plot how the bounds vary with sample size for a single hidden layer network with $50k$ hidden units. 
       We observe that these bounds are almost constant, and at best decrease at a meagre rate of $\Omega(m^{-0.066})$. On the \textbf{right}, we plot the test errors for this network and observe that it decreases with $m$ at the rate of at least $\mathcal{O}(m^{0.35})$.
       }      \label{fig:one-layer}
\end{figure}

\subparagraph{Effect of width.} In Figure~\ref{fig:width}, we demonstrate that our observation that
the bounds increase with $m$ extends to widths $h=128$ and $h=2000$ too.

\begin{figure}[h]
    \centering
        \begin{minipage}{.35\textwidth}
        \centering
        \adjincludegraphics[width=1\textwidth,trim={0 {0} 0 0},clip,,valign=t]{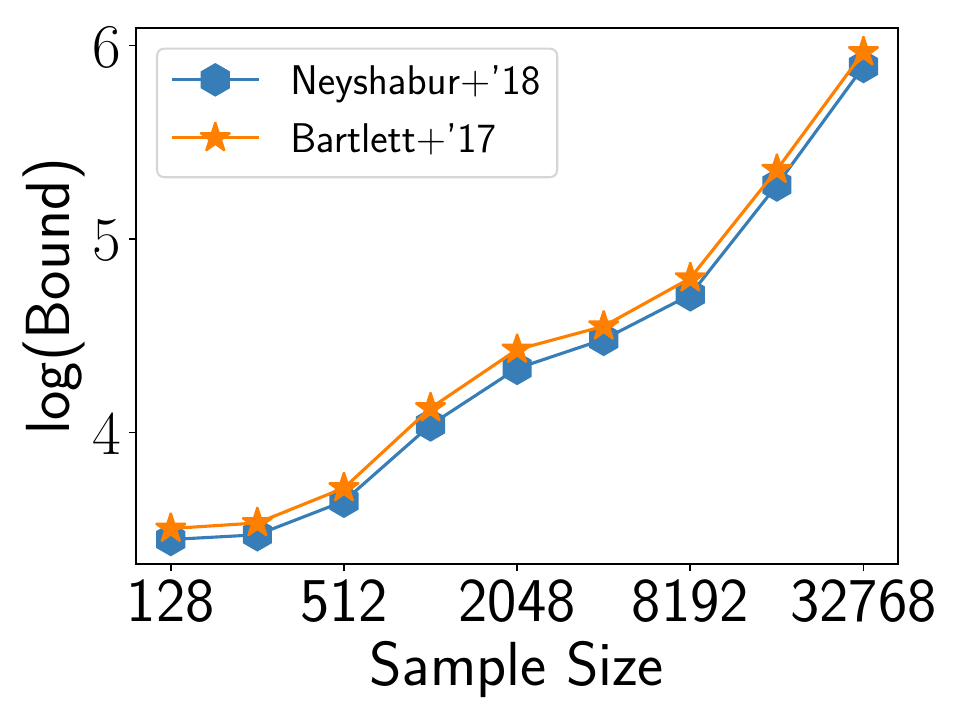} %
    \end{minipage}%
            \begin{minipage}{.35\textwidth}
        \centering
        \adjincludegraphics[width=1\textwidth,trim={0 {0} 0 0},clip,valign=t]{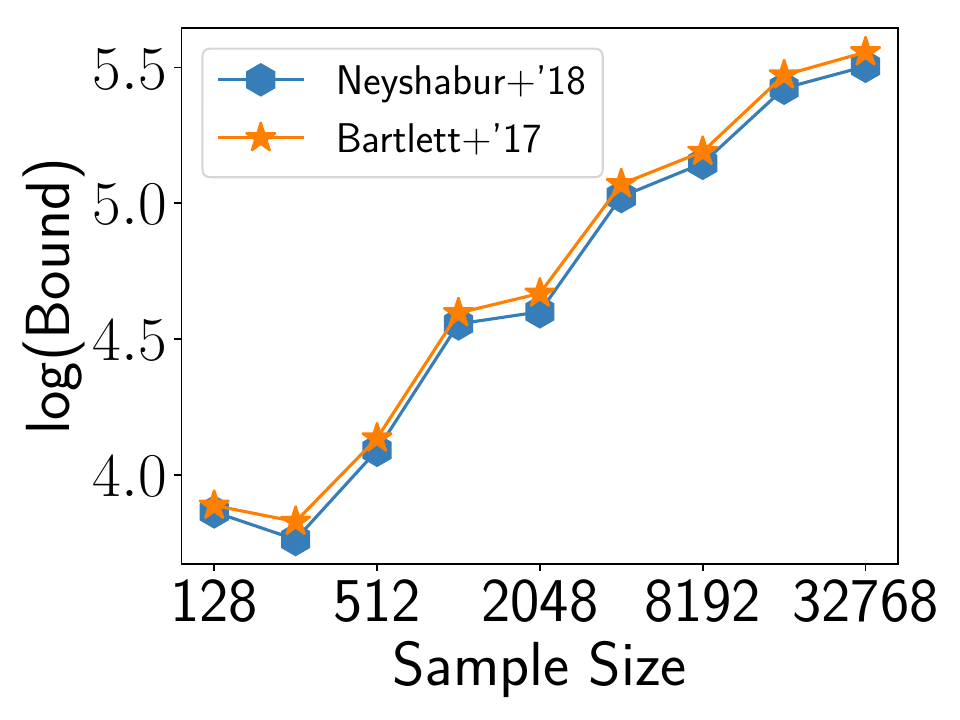} %
    \end{minipage}
       \caption{On the \textbf{left}, we plot the bounds for varying $m$ for $h=128$. All these bounds grow with $m$ as $\Omega(m^{0.94})$. On the \textbf{right}, we show a similar plot for $h=2000$ and observe that the bounds grow as $\Omega(m^{0.79})$.
       }      \label{fig:width}
\end{figure}

\subsection{Effect of batch size}

\subparagraph{Bounds vs. batch size for fixed $m$.}
In Figure~\ref{fig:bs-2}, we show how the bounds vary with the batch size for a fixed sample size of $16384$. It turns out that even though the test error decreases with decreasing batch size (for our fixed stopping criterion), all these bounds {\em increase}  (by a couple of orders of magnitude)  with decreasing batch size. Again, this is because the terms like distance from initialization {\em increase} for smaller batch sizes (perhaps because of greater levels of noise in the updates). Overall, existing bounds do not reflect the same behavior as the actual generalization error in terms of their dependence on the batch size.

\begin{figure}[!h]
    \centering
        \begin{minipage}{.35\textwidth}
        \centering
        \adjincludegraphics[width=1\textwidth,trim={0 {0} 0 0},clip,,valign=t]{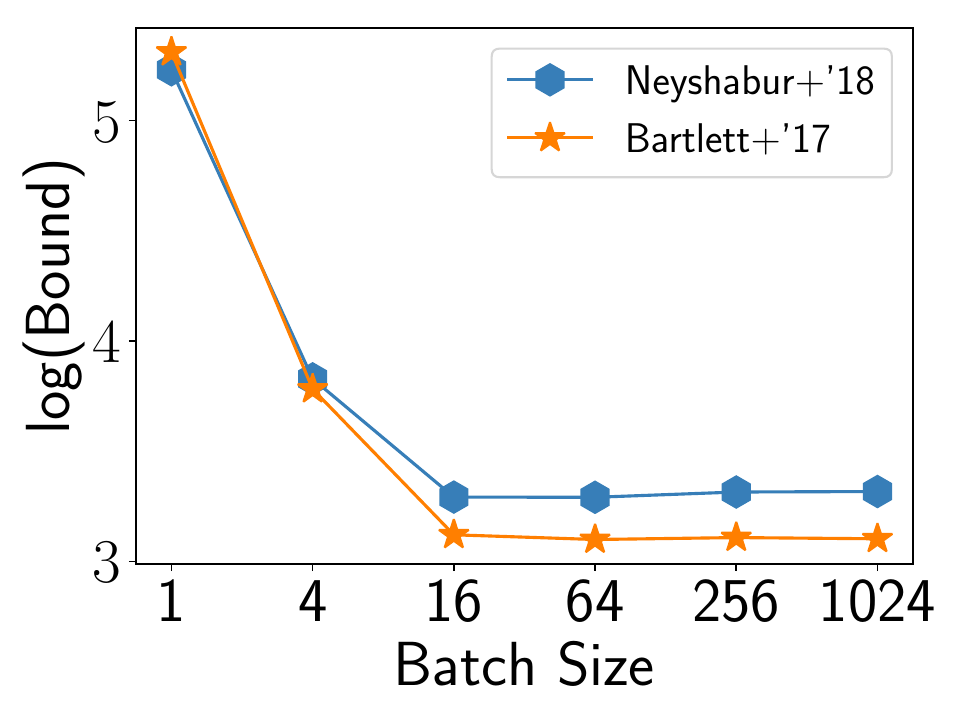} %
    \end{minipage}%
            \begin{minipage}{.35\textwidth}
        \centering
        \adjincludegraphics[width=1\textwidth,trim={0 {0} 0 0},clip,valign=t]{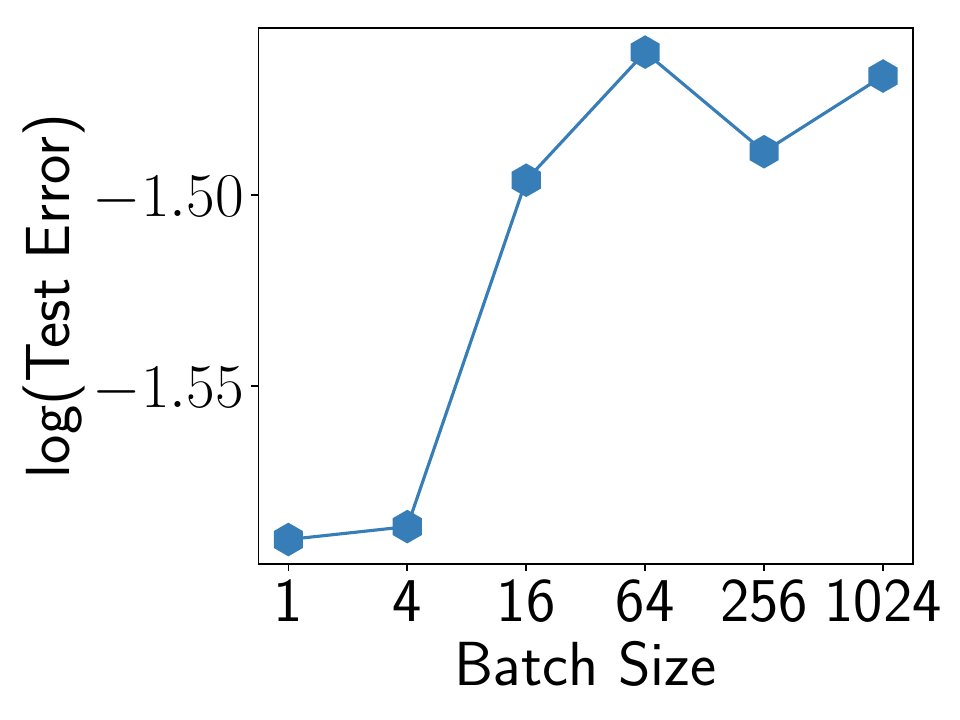} %
    \end{minipage}
       \caption{On the \textbf{left}, we plot the bounds for varying batch sizes for $m=16384$ and observe that these bounds {\em decrease} by around $2$ orders of magnitude. On the \textbf{right}, we plot the test errors for varying batch sizes and observe that test error increases with batch size albeit slightly.
       }      \label{fig:bs-2}
\end{figure}

\subparagraph{Bounds vs.  $m$ for batch size of $32$.}
In the main paper, we only dealt with a small batch size of $1$. In Figure~\ref{fig:bs}, we show bounds vs. sample size plots for a batch size of $32$. We observe that in this case, the bounds do decrease with sample size, although only at a rate of $\mathcal{O}(m^{-0.23})$ which is not as fast as the observed decrease in test error which is $\Omega(m^{-0.44})$. Our intuition as to why the bounds behave better (in terms of $m$-dependence) in the larger batch size regime is that here the amount of noise in the parameter updates is much less compared to smaller batch sizes (and as we discussed earlier, uniform convergence finds it challenging to explain away such noise).

\begin{figure}[!h]
    \centering
        \begin{minipage}{.35\textwidth}
        \centering
        \adjincludegraphics[width=1\textwidth,trim={0 {0} 0 0},clip,,valign=t]{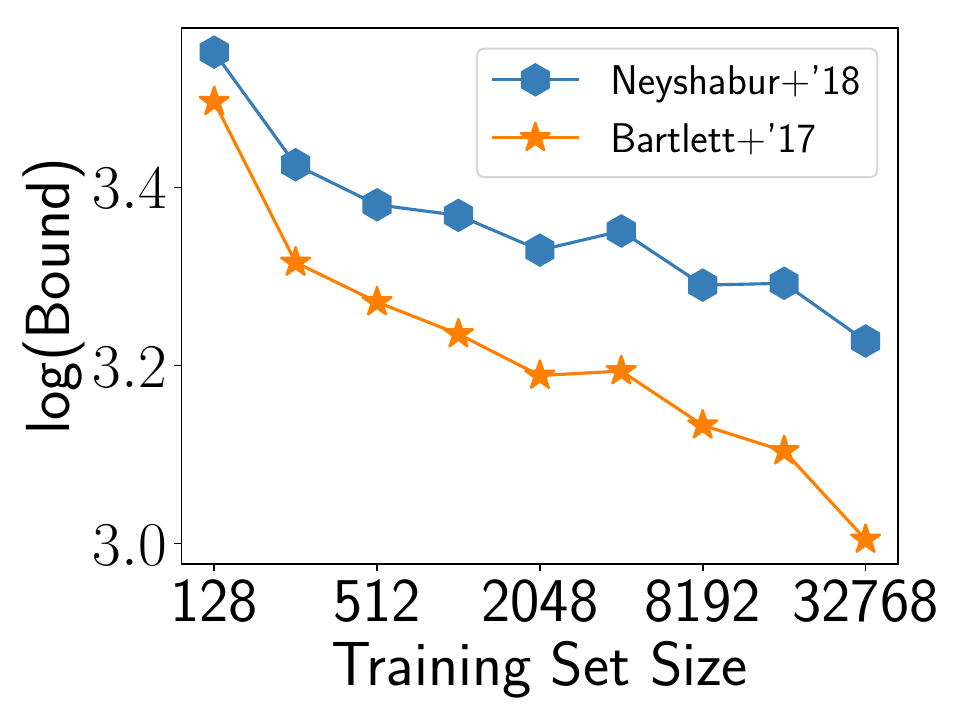} %
    \end{minipage}%
            \begin{minipage}{.35\textwidth}
        \centering
        \adjincludegraphics[width=1\textwidth,trim={0 {0} 0 0},clip,valign=t]{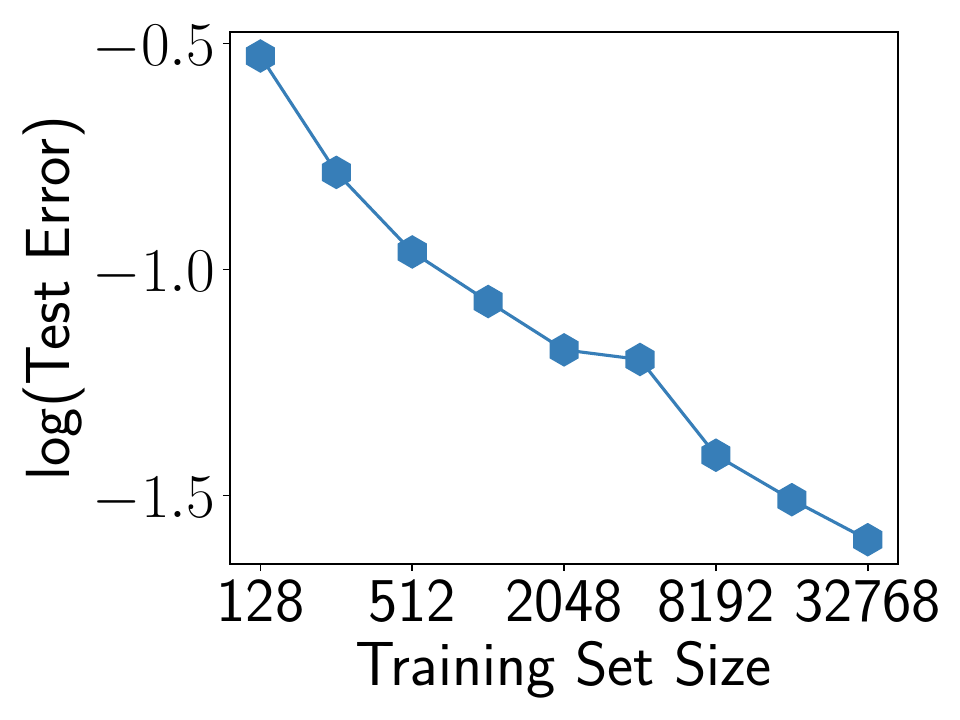} %
    \end{minipage}
       \caption{On the \textbf{left}, we plot the bounds for varying $m$ for a batch size of $32$ and observe that these bounds do decrease with $m$ as $\mathcal{O}(1/m^{0.23})$. On the \textbf{right}, we plot the test errors for various $m$ for batch size $32$ and observe that test error varies as $\Omega(1/m^{0.44})$.
       }      \label{fig:bs}
\end{figure}

\subparagraph{Squared error loss.} All the experiments presented so far deal with the cross-entropy loss, for which the optimization procedure ideally diverges to infinity; thus, one might suspect that our results are sensitive to the stopping criterion. It would therefore be useful to consider the squared error loss where the optimum on the training loss can be found in a finite distance away from the random initialization. Specifically, we consider the case where the squared error loss between the outputs of the network and the one-hot encoding of the true labels is minimized to a value of $0.05$ on average over the training data. 
 
We observe in Figure~\ref{fig:squared-error} that even for this case, the distance from initialization and the spectral norms grow with the sample size at a rate of at least $m^{0.3}$. On the other hand, the test error decreases with sample size as $1/m^{0.38}$, indicating that even for the squared error loss, these terms hurt would hurt the generalization bound with respect to its dependence on $m$.

\begin{figure}[h]
    \centering
            \begin{minipage}{.35\textwidth}
        \centering
        \adjincludegraphics[width=1\textwidth,trim={0 {0} 0 0},clip,,valign=t]{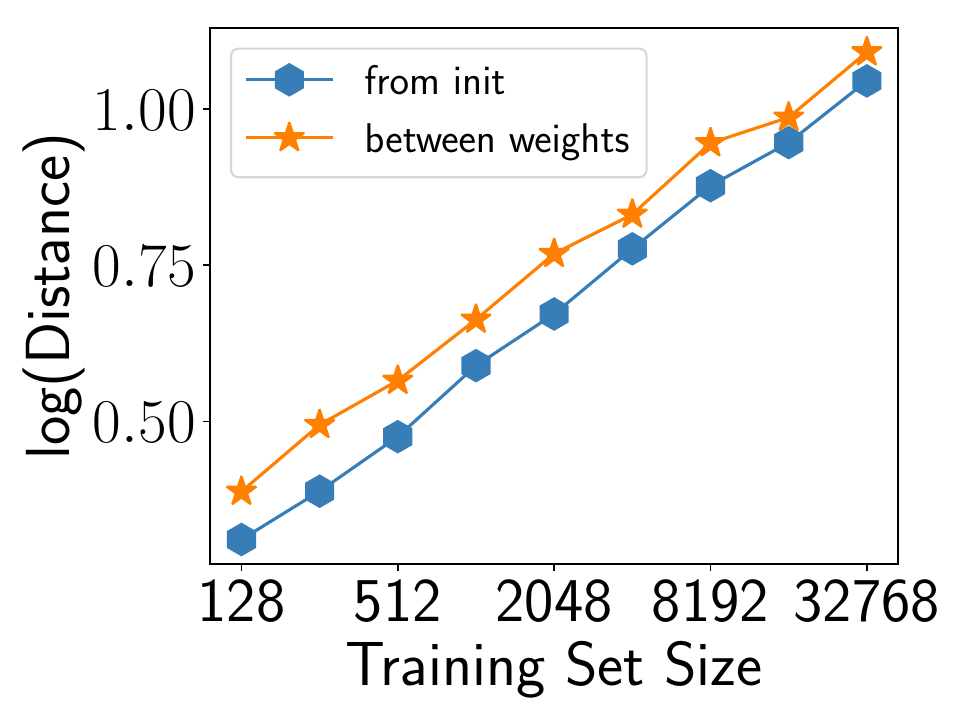} %
    \end{minipage}%
                \begin{minipage}{.35\textwidth}
        \centering
        \adjincludegraphics[width=1\textwidth,trim={0 {0} 0 0},clip,valign=t]{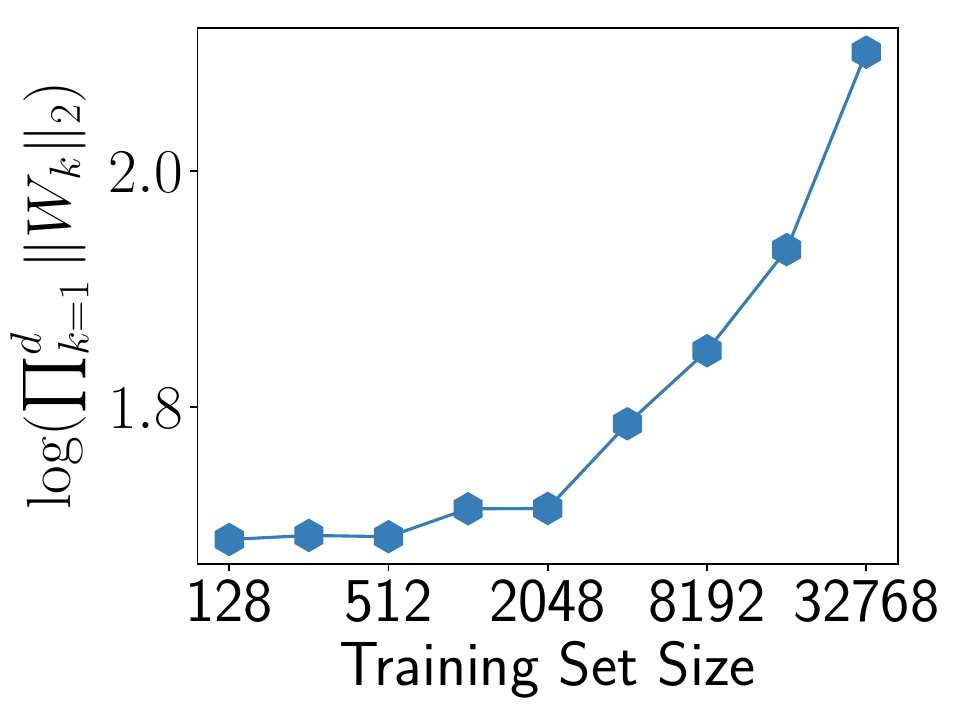} %
    \end{minipage}%
            \begin{minipage}{.35\textwidth}
        \adjincludegraphics[width=1\textwidth,trim={0 {0} 0 0},clip,,valign=t]{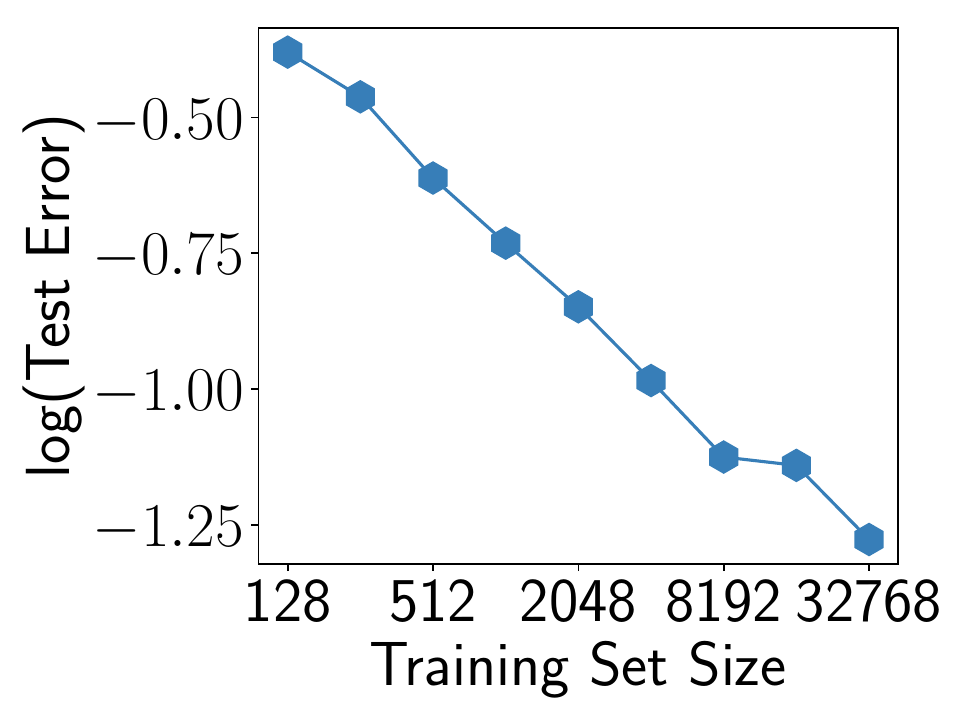} %
        \end{minipage}
       \caption{On the \textbf{left} we plot the distance from initialization and the distance between weights learned on two different random draws of the datasets, as a function of varying training set size $m$, when trained on the squared error loss. Both these quantities grow as $\Omega(m^{0.35})$. In the {\bf middle}, we show how the product of spectral norms grow as $\Omega(m^{0.315})$ for sufficiently large $m \geq 2048$. On the \textbf{right}, we observe that the test error (i.e., the averaged squared error loss on the test data) decreases with $m$ as $\mathcal{O}(m^{-0.38})$.
       }      \label{fig:squared-error}
\end{figure}

\section{Pseudo-overfitting}
\label{app:pseudo-overfit}
Recall that in the main paper, we briefly discussed a new notion that we call as pseudo-overfitting and noted that it is not the reason behind why some techniques lead to vacuous generalization bounds. We describe this in more detail here.
We emphasize this discussion because i) it brings up a fundamental and so far unknown  issue that might potentially exist in current approaches to explaining generalization and ii) rules it out before making more profound claims about uniform convergence. 

Our argument specifically applies to margin-based Rademacher complexity approaches (such as \citet{bartlett17spectral,neyshabur18unitwise}). These result in a bound like in Equation~\ref{eq:gen-error-bound} that we recall here:

\begin{align*}
 Pr_{(x,y) \sim \mathcal{D}}[\Gamma(f(\vec{x}),y) & \leq 0]  \leq \frac{1}{m}\sum_{(x,y)\in S}\mathbf{1}[\Gamma(f(\vec{x}),y) \leq \gamma] \\
& + \text{generalization error bound}. \; (\ref{eq:gen-error-bound})
\end{align*}

These methods upper bound the uniform convergence bound on the $\mathcal{L}^{(\gamma)}$ error on the network in terms of a uniform convergence bound on the margins of the network (see \cite{mohri12foundations} for more details about margin theory of Rademacher complexity). The resulting generalization error bound in Equation~\ref{eq:gen-error-bound} would take the following form, as per our notation from Definition~\ref{def:unif-alg}:
\begin{align*}
\sup_{S \in \mathcal{S}_{\delta}} \sup_{h \in \mathcal{H}_{\delta}}  \frac{1}{\gamma}  \left|   \mathbb{E}_{\mathcal{D}}[ \Gamma(h(\vec{x}),y) ] -  \frac{1}{m} \sum_{(x,y) \in S} \Gamma(h(\vec{x}),y)\right|. \numberthis\label{eq:margin-unif}
\end{align*}

This particular upper bound on the generalization gap in the $\mathcal{L}^{(\gamma)}$ loss is also an upper bound on the generalization gap on the margins. That is, with high probability $1-\delta$ over the draws of $S$, the above bound is larger than the following term that corresponds to the difference in test/train margins:

\begin{align*}
& \frac{1}{\gamma} \left(    \mathbb{E}_{(x,y)\sim \mathcal{D}}[ \Gamma(h_{S}(\vec{x}),y) ] -  \frac{1}{m} \sum_{(x,y) \in S} \Gamma(h_{S}(\vec{x}),y)\right).
\numberthis \label{eq:margin-bound}
\end{align*}

We first argue that it is possible for the generalization error of the algorithm to decrease with $m$ (as roughly $m^{-0.5}$), but for the above quantity to be independent of $m$. As a result, the margin-based bound in Equation~\ref{eq:margin-unif} (which is larger than Equation~\ref{eq:margin-bound}) will be non-decreasing in $m$, and even vacuous. Below we describe such a scenario.

Consider a network that first learns a simple hypothesis to fit the data, say, by learning a simple linear input-output mapping on linearly separable data. But subsequently, the classifier proceeds to {\em pseudo-overfit} to the samples by skewing up (down) the real-valued output of the network by some large constant $\Delta$ in a tiny neighborhood around the positive (negative) training inputs. Note that this would be possible if and only if the network is overparameterized. Now, even though the classifier's real-valued output is skewed around the training data, the decision boundary is still linear as the sign of the classifier's output has not changed on any input. Thus, the boundary is still simple and linear and the generalization error small. 

However, the training margins are at least a constant $\Delta$ larger than the test margins (which are not affected by the bumps created in tiny regions around the training data). Then, the term in Equation~\ref{eq:margin-bound} would be larger than $\Delta/\gamma$ and as a result, so would the term in Equation~\ref{eq:margin-unif}. Now in the generalization guarantee of Equation~\ref{eq:gen-error-bound}, recall that we must pick a value of $\gamma$ such that the first term is low i.e., most of the training datapoints must be classified by at least $\gamma$ margin. In this case, we can at best let $\gamma \approx \Delta$ as any larger value of $\gamma$ would make the margin-based training error non-negligible; as a result of this choice of $\gamma$, the bound in Equation~\ref{eq:margin-bound} would be an $m$-independent constant close to $1$. The same would also hold for its upper bound in Equation~\ref{eq:margin-unif}, which is the generalization bound provided by the margin-based techniques.

Clearly, this is a potential fundamental limitation in existing approaches, and if deep networks were indeed {pseudo-overfitting} this way, we would have identified the reason why at least some existing bounds are vacuous. However, (un)fortunately, we rule this out by observing that the difference in the train and test margins in Equation~\ref{eq:margin-bound} does decrease with training dataset size $m$ (see Figure~\ref{fig:margin-convergence}) as $\mathcal{O}(m^{-0.33})$. Additionally, this difference is numerically much less than $\gamma^\star=10$ (which is the least margin by which $99\%$ of the training data is classified) as long as $m$ is large, implying that Equation~\ref{eq:margin-bound} is non-vacuous.

 It is worth noting that the generalization error decreases at a faster rate of $\mathcal{O}(m^{-0.43})$ implying that the upper bound in Equation~\ref{eq:margin-bound} which decreases only as $m^{-0.33}$,
 is loose.  This already indicates a partial weakness in this specific approach to deriving generalization guarantees. Nevertheless, even this upper bound decreases at a significant rate with $m$ which the subsequent uniform convergence-based upper bound in Equation~\ref{eq:margin-unif} is unable to capture, thus hinting at more fundamental weaknesses specific to uniform convergence.
   \begin{figure}
        \centering
        \adjincludegraphics[width=0.35\textwidth,trim={0 {0} 0 0},clip,,valign=t]{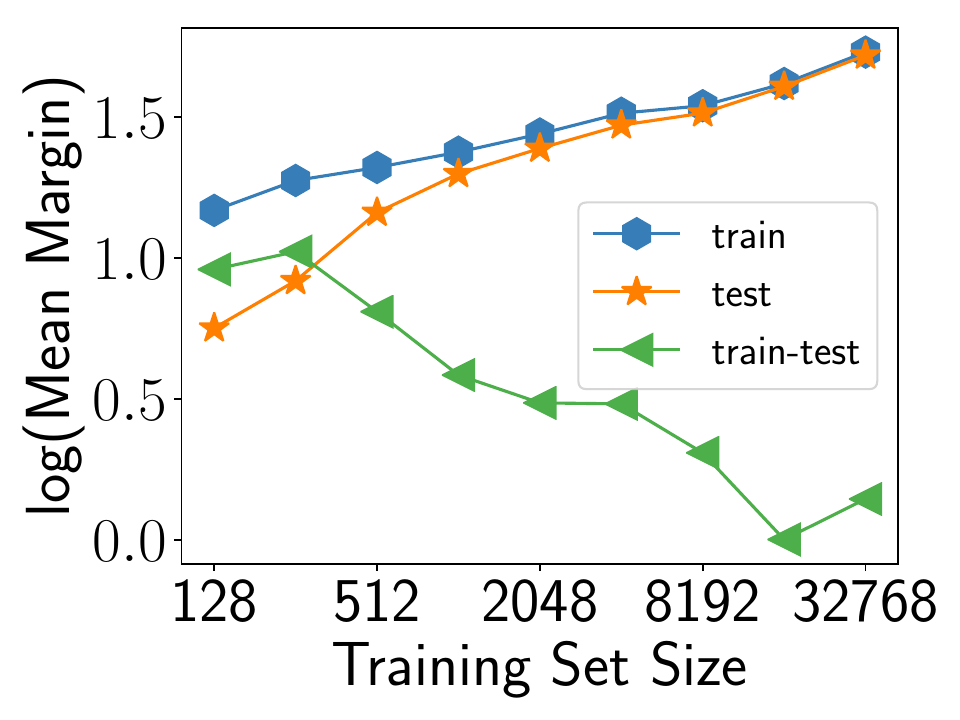} %
        \caption{We plot the average margin of the network on the train and test data, and the difference between the two, the last of which decreases with $m$ as $\mathcal{O}(1/m^{0.33})$.}
                \label{fig:margin-convergence}
\end{figure}

\subparagraph{Do our example setups suffer from pseudo-overfitting?}

Before we wrap up this section, we discuss a question brought up by an anonymous reviewer, which we believe is worth addressing. Recall that in Section~\ref{sec:linear} and Section~\ref{sec:hypersphere}, we presented a linear and hypersphere classification task where we showed that uniform convergence provably fails. In light of the above discussion, one may be tempted to ask: do these two models fail to obey uniform convergence because of pseudo-overfitting? 

The answer to this is that our proof for failure of uniform convergence in both these examples did not rely on any kind of pseudo-overfitting -- had our proof relied on it, then we would have been able to show failure of only specific kinds of uniform convergence bounds (as discussed above). More formally, pseudo-overfitting in itself does not imply the lower bounds on $\epsilon_{{\textrm{\tiny unif-alg}}}$ that we have shown in these settings.

One may still be curious to understand the level of pseudo-overfitting in these examples, to get a sense of the similarity of this scenario with that of the MNIST setup. To this end, we note that our linear setup does indeed suffer from significant pseudo-overfitting -- the classifier's output does indeed have bumps around each training point (which can be concluded from our proof). 

In the case of the hypersphere example, we present Figure~\ref{fig:margin-convergence-hypersphere}, where we plot of the average margins in this setup like in Figure~\ref{fig:margin-convergence}. Here, we observe that, the mean margins on the test data (orange line) and on training data (blue line) {do converge to each other} with more training data size $m$ i.e., {the gap in the mean test and training margins (green line) {\em does} decrease with $m$}. Thus our setup exhibits a behavior similar to deep networks on MNIST in Figure~\ref{fig:margin-convergence}. 
As noted in our earlier discussion, since the rate of decrease of the mean margin gap in MNIST is not as large as the decrease in test error itself,  there should be ``a small amount'' of psuedo-overfitting in MNIST. The same holds in this setting, although, here we observe an even milder decrease, implying a larger amount of pseudo-overfitting. Nevertheless, we emphasize that, our proof shows that uniform convergence cannot capture even this decrease with $m$.

To conclude, pseudo-overfitting is certainly a phenomenon worth exploring better; however, our examples elucidate that there is a phenomenon beyond pseudo-overfitting that is at play in deep learning.

  \begin{figure}
        \centering
        \adjincludegraphics[width=0.35\textwidth,trim={0 {0} 0 0},clip,,valign=t]{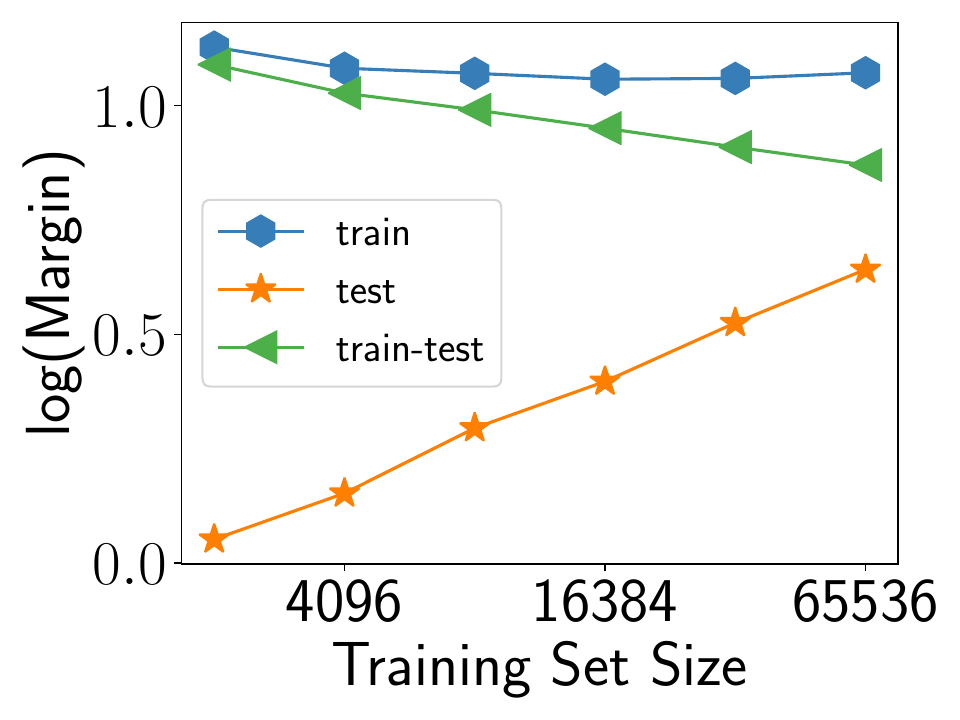} %
        \caption{In the hypersphere example of Section~\ref{sec:hypersphere}, we plot the average margin of the network on the train and test data, and the difference between the two. We observe the train and test margins do converge to each other.}
                \label{fig:margin-convergence-hypersphere}
\end{figure}

\section{Useful Lemmas}
\label{sec:lemma}

In this section, we state some standard results we will use in our proofs. We first define some constants: $c_1=1/2048$, $c_2 =\sqrt{15/16}$ and $c_3 = \sqrt{17/16}$ and $c_4 =\sqrt{2}$.

 First, we state a tail bound for sub-exponential random variables \citep{wainwright19high}.

\begin{lemma}
For a sub-exponential random variable $X$ with parameters $(\nu, b)$ and mean $\mu$, for all $t > 0$,
\[
Pr\left[|X - \mu| \geq t \right] \leq 2 \exp\left(-\frac{1}{2} \min \left(\frac{t}{b}, \frac{t^2}{\nu^2}\right) \right). \\
\]
\end{lemma}

As a corollary, we have the following bound on the sum of squared normal variables:

\begin{corollary}
\label{cor:chi}
For $z_1, z_2, \hdots, z_D \sim \mathcal{N}(0,1)$, we have that
\[
Pr\left[  \frac{1}{D}\sum_{j=1}^{D} z_j^2 \in [c_2^2, c_3^2]\right] \leq 2\exp(-c_1 D).
\]
\end{corollary}

We now state the Hoeffding bound for sub-Gaussian random variable. 

\begin{lemma}
\label{lem:hoeffding}
Let $z_1, z_2, \hdots, z_D$ be independently drawn sub-Gaussian variables with mean $0$ and sub-gaussian parameter $\sigma_i$. Then,
\[
Pr\left[ \left|\sum_{d=1}^{D} z_d \right| \geq t \right] \leq 2\exp(-t^2/2\sum_{d=1}^{D} \sigma_d^2).
\]
\end{lemma}

Again, we restate it as follows:
\begin{corollary}
\label{cor:gaussian}
For any $\vec{u} = (u_1, u_2, \hdots, u_d) \in \mathbb{R}^D$,
for $z_1, z_2, \hdots, z_D \sim \mathcal{N}(0,1)$,
\[
Pr\left[ \left|\sum_{d=1}^{D} u_d z_d \right| \geq \| \vec{u} \|_2 \cdot c_4  \sqrt{\ln \frac{2}{\delta}}  \right] \leq \delta.
\]
\end{corollary}

\section{Proof for Theorem~\ref{thm:example}}
\label{sec:proof}

In this section, we prove the failure of uniform convergence for our linear model. We first recall the setup:

\textbf{Distribution $\mathcal{D}$:} Each input $(\vec{x}_1, \vec{x}_2)$ is a $K+D$ dimensional vector where  $\vec{x}_1 \in \mathbb{R}^K$ and $\vec{x}_2 \in \mathbb{R}^D$. $\vec{u} \in \mathbb{R}^K$ determines the centers of the classes. The label $y$ is drawn uniformly from $\{ -1, +1\}$, and conditioned on $y$, we have $\vec{x}_1 = 2\cdot y \cdot \vec{u}$ while $\vec{x}_2$ is sampled independently from $\mathcal{N}(0,\frac{32}{D}I)$.

\textbf{Learning algorithm $\mathcal{A}$:} We consider a linear classifier with weights $\vec{w} = (\vec{w}_1, \vec{w}_2)$. The output is computed as $h(\vec{x}) = \vec{w}_1 \vec{x}_1 + \vec{w}_2 \vec{x}_2$. Assume the weights are initialized to origin. 
Given $S = \{(\vec{x}^{(1)},y^{(1)}), \hdots, (\vec{x}^{(m)},y^{(m)})  \}$, $\mathcal{A}$ takes a gradient step of learning rate $1$ to maximize $y \cdot h(\vec{x})$ for each $(\vec{x},y) \in S$. Regardless of the batch size, the learned weights would satisfy,  $\vec{w}_1 = 2m \vec{u}$ and $\vec{w}_2 = \sum_i y^{(i)} \vec{x}_2^{(i)}$.

Below, we state the precise theorem statement (where we've used the constants $c_1=1/32$, $c_2 =1/2$ and $c_3 = 3/2$ and $c_4 =\sqrt{2}$):

\textbf{Theorem~\ref{thm:example}}
{\em{
In the setup above, for any $\epsilon,\delta > 0$ and $\delta < 1/4$, let $D$ be sufficiently large that it satisfies
\begin{align*}
D & \geq \frac{1}{c_1} \ln \frac{6m}{\delta}, \numberthis\label{eq:d1}\\
{D} & \geq  {m} \left( \frac{4c_4 c_3}{c_2^2} \right)^2 \ln \frac{6m}{\delta},  \numberthis\label{eq:d2}\\
{D}  & \geq {m}\left(\frac{4c_4 c_3}{c_2^2}\right)^2 \cdot 2{\ln \frac{2}{\epsilon}},  \numberthis\label{eq:d3}
\end{align*}
then we have that  for all $\gamma \geq 0$, for the $\mathcal{L}^{(\gamma)}$ loss, $\epsilon_{\text{unif-alg}}(m,\delta) \geq 1- \epsilon_{\text{gen}}(m,\delta)$. 

Specifically, for $\gamma \in [0,1]$, $\epsilon_{\text{gen}}(m,\delta) \leq \epsilon$, and so $\epsilon_{\text{unif-alg}}(m,\delta) \geq 1- \epsilon$.
}}

\begin{proof}
The above follows from Lemma~\ref{lem:gen}, where we upper bound the generalization error, and from Lemma~\ref{lem:unif} where we lower bound uniform convergence.
\end{proof}

We first prove that the above algorithm generalizes well with respect to the losses corresponding to $\gamma \in [0,1]$. First for the training data, we argue that both $\vec{w}_1$ and a small part of the noise vector $\vec{w}_2$ align along the correct direction, while the remaining part of the high-dimensional noise vector are orthogonal to the input; this leads to correct classification of the training set. Then, on the test data, we argue that $\vec{w}_1$ aligns well, while $\vec{w}_2$ contributes very little to the output of the classifier because it is high-dimensional noise. As a result, for most test data, the classification is correct, and hence the test and generalization error are both small.
\\

\begin{lemma}
\label{lem:gen}
In the setup of Section~\ref{sec:setup}, when $\gamma \in [0,1]$, for $\mathcal{L}^{(\gamma)}$, $\epsilon_{\text{gen}} (m,\delta)\leq \epsilon$.
\end{lemma}

\begin{proof}
The parameters learned by our algorithm satisfies $\vec{w}_1= 2 m\cdot \vec{u}$ and $\vec{w}_2  = \sum y^{(i)} \vec{x}_2^{(i)} \sim \mathcal{N}(0, \frac{8m}{c_2^2 D})$.  

First, we have from Corollary~\ref{cor:chi} that with probability $1-\frac{\delta}{3m}$ over the draws of $\vec{x}_2^{(i)}$, as long as $\frac{\delta}{3m} \geq 2e^{-c_1D}$ (which is given to hold by Equation~\ref{eq:d1}), \begin{equation}c_2 \leq  \frac{1}{2\sqrt{2}}{c_2} \|\vec{x}_2^{(i)}\| \leq c_3.
\label{eq:ce1}
 \end{equation}

Next, for a given $\vec{x}^{(i)}$, we have from Corollary~\ref{cor:gaussian}, with probability $1-\frac{\delta}{3m}$ over the draws of $\sum_{j\neq i} y^{(j)}\vec{x}_2^{(j)}$, \begin{equation}|\vec{x}_2^{(i)} \cdot \sum_{j\neq i} y^{(j)}\vec{x}_2^{(j)}| \leq  c_4 \|\vec{x}_2^{(i)} \| \frac{2\sqrt{2} \cdot \sqrt{m}}{c_2 \sqrt{D}} \sqrt{ \ln \frac{6m}{\delta
}}.
\label{eq:ce2}\end{equation}

Then, with probability $1-\frac{2}{3}\delta$ over the draws of the training dataset we have for all $i$, 

\begin{align*}
y^{(i)} h(\vec{x}^{(i)}) &= y^{(i)} \vec{w}_1 \cdot \vec{x}^{(i)}_1 + y^{(i)} \cdot y^{(i)} \|\vec{x}_2^{(i)} \|^2 + y^{(i)} \cdot \vec{x}_2^{(i)} \cdot \sum_{j\neq i} y^{(j)}\vec{x}_2^{(j)}  \\
&=4 +  \underbrace{\|\vec{x}_2^{(i)} \|^2}_{\text{apply Equation~\ref{eq:ce1}}} +  \underbrace{y^{(i)} \vec{x}_2^{(i)} \cdot \sum_{j\neq i} y^{(j)}\vec{x}_2^{(j)}}_{\text{apply Equation~\ref{eq:ce2}}}  \\
& \geq 4 +4 \cdot 2   -  c_4 \frac{2 \sqrt{2} c_3}{c_2} \cdot \underbrace{\frac{2\sqrt{2} \cdot \sqrt{m}}{c_2 \sqrt{D}} \sqrt{\ln \frac{6m}{\delta}}}_{\text{apply Equation~\ref{eq:d2}}} \\
& \geq 4 + 8 - 2 = 10 > 1. \numberthis\label{eq:train-margin}
\end{align*}

Thus, for all $\gamma \in [0,1]$, the $\mathcal{L}^{(\gamma)}$ loss of this classifier on the training dataset $S$ is zero. 

Now, from Corollary~\ref{cor:chi}, with probability $1-\frac{\delta}{3}$ over the draws of the training data, we also have that, as long as $\frac{\delta}{3m} \geq 2 e^{-c_1 D}$ (which is given to hold by Equation~\ref{eq:d1}), \begin{equation}c_2 \sqrt{m} \leq \frac{1}{2\sqrt{2}} c_2 \|\sum y^{(i)} \vec{x}_2^{(i)}\| \leq c_3\sqrt{m}. \label{eq:ce3} \end{equation}

Next, conditioned on the draw of $S$ and the learned classifier, for any $\epsilon' > 0$, with probability $1-\epsilon'$ over the draws of a test data point, $(\vec{z},y)$, we have from Corollary~\ref{cor:gaussian} that

\begin{equation}|\vec{z}_2 \cdot \sum y^{(i)} \vec{x}_2^{(i)}| \leq c_4 \|\sum y^{(i)} \vec{x}_2^{(i)}\| \cdot \frac{2\sqrt{2}}{c_2 \sqrt{D}}  \cdot  \ln \frac{1}{\epsilon'}.
\label{eq:ce4}
\end{equation}

Using this, we have that with probability $1- 2 \exp \left(- \frac{1}{2}\left({\frac{c_2^2}{4 c_4 c_3}  \sqrt{\frac{D}{m}} }\right)^2\right)$ over the draws of a test data point, $(\vec{z},y)$,
\begin{align*}
 y h(\vec{x})  & =  y \vec{w}_1 \cdot \vec{z}_1 +\underbrace{  y \cdot \vec{z}_2 \cdot \sum_{j} y^{(j)}\vec{x}_2^{(j)} }_{\text{apply Equation~\ref{eq:ce4}} } \\
& \geq  4 -  c_4 \underbrace{\|\sum y^{(i)} \vec{x}_2^{(i)}\| }_{\text{apply Equation~\ref{eq:ce3}}}\cdot \frac{2\sqrt{2}}{c_2 \sqrt{D}} \frac{c_2^2}{4 c_4 c_3} \sqrt{\frac{D}{m}} \\
& \geq 4 - 2 \geq 2. \numberthis\label{eq:test-margin}
\end{align*}


Thus, we have that for $\gamma \in [0,1]$, the $\mathcal{L}^{(\gamma)}$ loss of the classifier on the distribution $\mathcal{D}$ is $2\exp \left(-\frac{1}{2}\left({\frac{c_2^2}{4 c_4 c_3} \sqrt{\frac{D}{m}} }\right)^2\right)$ which is at most $\epsilon$ as assumed in Equation~\ref{eq:d3}. In other words, the absolute difference between the distribution loss and the train loss is at most $\epsilon$
 and this holds for at least $1-\delta$ draws of the samples $S$. Then, by the definition of $\epsilon_{\text{gen}}$ we have the result. 

\end{proof}

We next prove our uniform convergence lower bound. The main idea is that when the noise vectors in the training samples are negated, with high probability, the classifier misclassifies the training data. We can then show that for any choice of $\mathcal{S}_{\delta}$ 
as required by the definition of $\epsilon_{\text{unif-alg}}$, we can always find an $S_{\star}$ and its noise-negated version $S_{\star}'$ both of which belong to $\mathcal{S}_{\delta}$. Furthermore, we can show that $h_{S_\star}$ has small test error but high empirical error on $S_{\star}'$, and that this leads to a nearly vacuous uniform convergence bound.

\begin{lemma}
\label{lem:unif}
In the setup of Section~\ref{sec:setup}, for any $\epsilon > 0$ and for any $\delta \leq 1/4$, and for the same lower bounds on $D$,  and for any $\gamma \geq 0$, we have that
\[
\epsilon_{\text{unif-alg}}(m,\delta) \geq 1 - \epsilon_{\text{gen}}(m,\delta)
\]
for the $\mathcal{L}^{(\gamma)}$ loss.
\end{lemma}

\begin{proof}

For any $S$, let $S'$ denote the set of noise-negated samples $S' = \{ ((\vec{x}_1, - \vec{x}_2),u) \; | \;  ((\vec{x}_1,  \vec{x}_2),y) \in S \}$. We first show with high probability $1-2\delta/3$ over the draws of $S$, that the classifier learned on $S$,  misclassifies $S'$ completely. The proof for this is nearly identical to our proof for why the training loss is zero, except for certain sign changes. For any $\vec{x}_{\text{neg}}^{(i)} = (\vec{x}_1^{(i)}, -\vec{x}_2^{(i)})$, we have

\begin{align*}
 y^{(i)} h(\vec{x}_{\text{neg}}^{(i)})& = y^{(i)} \vec{w}_1 \cdot \vec{x}^{(i)}_1 - y^{(i)} \cdot y^{(i)} \|\vec{x}_2^{(i)} \|^2 - y^{(i)} \cdot \vec{x}_2^{(i)} \cdot \sum_{j\neq i} y^{(j)}\vec{x}_2^{(j)}  \\
&=4-  \underbrace{\|\vec{x}_2^{(i)} \|^2}_{\text{apply Equation~\ref{eq:ce1}}} -  \underbrace{y^{(i)} \vec{x}_2^{(i)} \cdot \sum_{j\neq i} y^{(j)}\vec{x}_2^{(j)}}_{\text{apply Equation~\ref{eq:ce2}}}  \\
& \leq  4 - 4 \cdot 2  +  c_4 \frac{2 \sqrt{2} c_3}{c_2} \cdot \underbrace{\frac{2\sqrt{2} \cdot \sqrt{m}}{c_2 \sqrt{D}} \ln \frac{3m}{\delta}}_{\text{apply Equation~\ref{eq:d2}}} \\
& \leq 4 - 8 + 2 = -2 < 0.
\end{align*}

Since the learned hypothesis misclassifies all of $S'$, it has loss of $1$ on $S'$. 

 Now recall that, by definition, to compute $\epsilon_{\text{unif-alg}}$, one has to pick a sample set space $\mathcal{S}_{\delta}$ of mass $1-\delta$ i.e., $Pr_{S \sim \mathcal{S}^m}[S \in \mathcal{S}_{\delta}] \geq 1-\delta$. We first argue that for {\em any} choice of $\mathcal{S}_{\delta}$, there must exist a `bad' $S_\star$ such that (i) $S_\star \in \mathcal{S}_{\delta}$, (ii) $S_\star' \in \mathcal{S}_{\delta}$, (iii) $h_{S_\star}$ has test error less than $\epsilon_{\text{gen}}(m,\delta)$ and (iv) $h_{S_\star}$ completely misclassifies $S_{\star}'$. 

 We show the existence of such an $S_\star$, by arguing that over the draws of $S$, there is non-zero probability of picking an $S$ that satisfies all the above conditions. Specifically, we have by the union bound that
\begin{align*}
& Pr_{S \sim \mathcal{D}^m} \big[ S \in \mathcal{S}_{\delta}, S' \in \mathcal{S}_{\delta}, \mathcal{L}_{\mathcal{D}}(h_S) \leq \epsilon_{\text{gen}}(m,\delta), \hat{\mathcal{L}}_{S'}(h_S) =1\big]   \\
 &\geq 1
- Pr_{S \sim \mathcal{D}^m}\left[   S \notin \mathcal{S}_{\delta}\right] -  Pr_{S \sim \mathcal{D}^m}\left[   S' \notin \mathcal{S}_{\delta}\right]\\
&- Pr_{S \sim \mathcal{D}^m}\left[  \mathcal{L}_{\mathcal{D}}(h_S) > \epsilon_{\text{gen}}(m,\delta)\right] 
  - Pr_{S \sim \mathcal{D}^m}\left[   \hat{\mathcal{L}}_{S'}(h_S) \neq 1\right]. \numberthis \label{eq:probability-term}
\end{align*}

By definition of $\mathcal{S}_{\delta}$, we know $Pr_{S \sim \mathcal{D}^m}\left[   S \notin \mathcal{S}_{\delta}\right] \leq \delta$. Similarly, by definition of the generalization error, we know that $Pr_{S \sim \mathcal{D}^m}\left[  \mathcal{L}_{\mathcal{D}}(h_S) > \epsilon_{\text{gen}}(m,\delta)\right]  \leq \delta$. We have also established above that  $Pr_{S \sim \mathcal{D}^m}\left[   \hat{\mathcal{L}}_{S'}(h_S) \neq 1\right]  \leq 2\delta/3$. As for the term $Pr_{S \sim \mathcal{D}^m}\left[   S' \notin \mathcal{S}_{\delta}\right]$, observe that under the draws of $S$, the distribution of the noise-negated dataset $S'$ is identical to $\mathcal{D}^m$. This is because the isotropic Gaussian noise vectors have the same distribution under negation. Hence, again by definition of $\mathcal{S}_{\delta}$, even this probability is at most $\delta$. Thus, we have that the probability in the left hand side of Equation~\ref{eq:probability-term} is at least $1-4\delta$, which is positive as long as $\delta < 1/4$.

This implies that for any given choice of $\mathcal{S}_{\delta}$, there exists $S_{\star}$ that satisfies our requirement.  Then, from the definition of $\epsilon_{\text{unif-alg}}(m,\delta)$, we essentially have that,
\begin{align*}
\epsilon_{\text{unif-alg}}(m,\delta) &=  \sup_{S \in \mathcal{S}_{\delta}} \sup_{h \in \mathcal{H}_{\delta}} |{\mathcal{L}}_{\mathcal{D}}(h)- \hat{\mathcal{L}}_{S}(h)| \\
& \geq  |{\mathcal{L}}_{\mathcal{D}}(h_{S_\star})- \hat{\mathcal{L}}_{S_\star'}(h)|= |\epsilon-1| = 1-\epsilon.
\end{align*}

\end{proof}
 \section{Neural Network with Exponential Activations}
\label{sec:exp}
In this section, we prove the failure of uniform convergence for a neural network model with exponential activations. We first define the setup. 

\subparagraph{Distribution} Let $\vec{u}$ be an arbitrary vector in $D$ dimensional space such that $\| \vec{u}\| = \sqrt{D}/2$. 
Consider an input distribution in $2D$ dimensional space such that, conditioned on the label $y$ drawn from uniform distribution over $\{-1,+1\}$, the first $D$ dimensions $\vec{x}_1$ of a random point is given by $y \vec{u}$ and the remaining $D$ dimensions are drawn from $\mathcal{N}(0,1)$. Note that in this section, we require $D$ to be only as large as $\ln m$, and not as large as $m$.

\subparagraph{Architecture.} We consider an infinite width neural network with exponential activations, in which only the output layer weights are trainable. The hidden layer weights are frozen as initialized. Note that this is effectively a linear model with infinitely many randomized features. Indeed, recent work \citep{jacot18ntk} has shown that under some conditions on how deep networks are initialized and parameterized, they behave a linear models on randomized features.
Specifically, each hidden unit corresponds to a distinct (frozen) weight vector $\vec{w} \in \mathbb{R}^{2D}$ and an output weight $a_{\vec{w}}$ that is trainable. 
We assume that the hidden layer weights are drawn from $\mathcal{N}(0,I)$ and $a_{\vec{w}}$ initialized to zero. 
Note that the output of the network is determined as

\[
h(\vec{x}) = \mathbb{E}_{\vec{w}}[a_{\vec{w}}\exp(\vec{w}\cdot \vec{x})].
\]

\subparagraph{Algorithm} We consider an algorithm that takes a gradient descent step to maximize $y\cdot h(\vec{x})$ for each $(\vec{x},y)$ in the training dataset, with learning rate $\eta$. However, since, the function above is not a discrete sum of its hidden unit outputs, to define the gradient update on $a_{\vec{w}}$, we must think of $h$  as a functional whose input function maps every $\vec{w} \in \mathbb{R}^{2D}$ to $a_{\vec{w}} \in \mathbb{R}$. Then, by considering the functional derivative, one can conclude that the update on $a_{\vec{w}}$ can be written as

\begin{equation}
\label{eq:gradient-step}
a_{\vec{w}} \gets a_{\vec{w}} + \eta y \cdot \exp(\vec{w}\cdot \vec{x}) \cdot p(\vec{w}).
\end{equation}

where $p(\vec{w})$ equals the p.d.f of $\vec{w}$ under the distribution it is drawn from. In this case $p(\vec{w}) = \frac{1}{(2\pi)^D} \exp\left( -\frac{\| \vec{w} \|^2}{2} \right)$.

In order to simplify our calculations we will set $\eta = (4\pi)^D$, although our analysis would extend to other values of the learning rate too. Similarly, our results would only differ by constants if we consider the alternative update rule, $a_{\vec{w}} \gets a_{\vec{w}} + \eta y \cdot \exp(\vec{w}\cdot \vec{x})$.

We now state our main theorem (in terms of constants $c_1, c_2, c_3, c_4$ defined in Section~\ref{sec:lemma}).\\

\begin{theorem}
\label{thm:exp}
In the set up above, for any $\epsilon,\delta > 0$ and $\delta < 1/4$, let $D$ and $m$ be sufficiently large that it satisfies
\begin{align*}
D & \geq  \max \left(\frac{1}{c_2}, (16c_3 c_4)^2 \right) \cdot 2 \ln \frac{6m}{\epsilon} \numberthis\label{eq:d4}\\
{D} & \geq   \max \left(\frac{1}{c_2}, (16c_3 c_4)^2 \right) \cdot 2 \ln \frac{6m}{\delta}   \numberthis\label{eq:d5}\\
{D}  & \geq 6 \ln 2m \numberthis\label{eq:d6} \\
{m} & >  \max  8 \ln \frac{6}{\delta}.  \numberthis\label{eq:m}
\end{align*}
then we have that  for all $\gamma \geq 0$, for the $\mathcal{L}^{(\gamma)}$ loss, $\epsilon_{\text{unif-alg}}(m,\delta) \geq 1- \epsilon_{\text{gen}}(m,\delta)$. 

Specifically, for $\gamma \in [0,1]$, $\epsilon_{\text{gen}}(m,\delta) \leq \epsilon$, and so $\epsilon_{\text{unif-alg}}(m,\delta) \geq 1- \epsilon$.
\end{theorem}

\begin{proof}
The result follows from the following lemmas. First in Lemma~\ref{lem:exp-gradient}, we derive the closed form expression for the function computed by the learned network. In Lemma~\ref{lem:exp-gen}, we upper bound the generalization error and in Lemma~\ref{lem:exp-unif-conv}, we lower bound uniform convergence.
\end{proof}


We first derive a closed form expression for how the output of the network changes under a gradient descent step on a particular datapoint.\\

\begin{lemma}
\label{lem:exp-gradient}
Let $h^{(0)}(\cdot)$ denote the function computed by the network before updating the weights. After updating the weights on a particular input $(\vec{x},y)$ according to Equation~\ref{eq:gradient-step}, the learned network corresponds to:
\[
h(\vec{z}) = h^{(0)}(\vec{z}) +   y \exp\left( \left\|\frac{\vec{z}+\vec{x}}{2}\right\|^2\right).
\]

\end{lemma}

\begin{proof}
 From equation~\ref{eq:gradient-step}, we have that

 \begin{align*}
 &\frac{h(\vec{z}) -h^{(0)}(\vec{z})}{\eta} \\ &= \int_{\vec{w}} \left( y \cdot \exp(\vec{w}\cdot \vec{x})p(\vec{w})\right) \cdot \exp(\vec{w}\cdot \vec{z})p(\vec{w}) d\vec{w} \\
 & = y \int_{\vec{w}} \exp(\vec{w}\cdot (\vec{x}+\vec{z})) \cdot \left( {\frac{1}{2\pi}}\right)^{2D} \exp(-\|\vec{w}\|^2) d\vec{w}  \\
 & = y \left( {\frac{1}{2\pi}}\right)^{2D}  \int_{\vec{w}} \exp(\vec{w}\cdot (\vec{x}+\vec{z})-\|\vec{w}\|^2) d\vec{w}  \\
 & = y \left( {\frac{1}{2\pi}}\right)^{2D} \exp\left( \left\|\frac{\vec{z}+\vec{x}}{2}\right\|^2\right) \times \int_{\vec{w}} \exp\left(-\left\|\vec{w} - \frac{\vec{z}+\vec{x}}{2}\right\|^2\right) d\vec{w}  \\
 & = y \left( {\frac{1}{4\pi}}\right)^{D} \exp\left( \left\|\frac{\vec{z}+\vec{x}}{2}\right\|^2\right) \times   \left( \frac{1}{\sqrt{2\pi (0.5)}}\right)^{2D} \int_{\vec{w}} \exp\left(-\left\|\vec{w} - \frac{\vec{z}+\vec{x}}{2}\right\|^2\right) d\vec{w}  \\
  & = y \left( {\frac{1}{4\pi}}\right)^{D} \exp\left( \left\|\frac{\vec{z}+\vec{x}}{2}\right\|^2\right).   \\
 \end{align*}

 In the last equality above, we make use of the fact that the second term corresponds to the integral of the p.d.f of $\mathcal{N}(\frac{\vec{z}+\vec{x}}{2}, 0.5 I)$ over $\mathbb{R}^{2D}$.  Since we set $\eta = (4\pi)^D$ gives us the final answer.
\end{proof}

Next, we argue that the generalization error of the algorithm is small. From Lemma~\ref{lem:exp-gradient}, we have that the output of the network is essentially determined by a summation of contributions from every training point.
To show that the training error is zero, we argue that on any training point, the contribution from that training point dominates all other contributions, thus leading to correct classification.  On any test point, we similarly show that the contribution of training points of the same class as that test point dominates the output of the network. Note that our result requires $D$ to scale only logarithmically with training samples $m$. \\

\begin{lemma}
\label{lem:exp-gen}
In the setup of Section~\ref{sec:setup}, when $\gamma \in [0,1]$, for $\mathcal{L}^{(\gamma)}$, $\epsilon_{\text{gen}} (m,\delta)\leq \epsilon$.
\end{lemma}

\begin{proof}
We first establish a few facts that hold with high probability over the draws of the training set $S$. First, from Corollary~\ref{cor:chi} we have that, since $D \geq \frac{1}{c_2} \ln \frac{3m}{\delta}$ (from Equation~\ref{eq:d4}), with probability at least $1-\delta/3$ over the draws of $S$, for all $i$, the noisy part of each training input can be bounded as

\begin{align*}
c_2 \sqrt{D} \leq \| \vec{x}_2^{(i)}\| \leq c_3 \sqrt{D}. \numberthis\label{eq:exp-norm-bound}
\end{align*}

Next, from Corollary~\ref{cor:gaussian}, we have that with probability at least $1-\frac{\delta}{3m^2}$ over the draws of $\vec{x}_2^{(i)}$ and $\vec{x}_2^{(j)}$ for  $i \neq j$,

\begin{align*}
|\vec{x}_2^{(i)} \cdot \vec{x}_2^{(j)}| \leq \| \vec{x}_2^{(i)} \| \cdot c_4 \sqrt{2\ln \frac{6m}{\delta}}.   \numberthis\label{eq:exp-dot-prod-bound}
\end{align*} 

Then, by a union bound, the above two equations hold for all $i\neq j$ with probability at least $1-\delta/2$. 

Next, since each $y^{(i)}$ is essentially an independent sub-Gaussian with mean $0$ and sub-Gaussian parameter $\sigma = 1$, we can apply Hoeffding's bound (Lemma~\ref{lem:hoeffding}) to conclude that with probability at least $1-\delta/3$ over the draws of $S$,
\begin{align*}
\left|\sum_{j=1}^{m} y^{(j)}\right| \leq \underbrace{\sqrt{2m\ln \frac{6}{\delta} }}_{Eq~\ref{eq:m}} < \frac{m}{2}  \numberthis\label{eq:class-count}.
\end{align*}

Note that this means that there must exist at least one training data in each class.


Given these facts, we first show that the training error is zero by showing that for all $i$, $y^{(i)} h(\vec{x}^{(i)})$ is sufficiently large. On any training input $(\vec{x}^{(i)},y^{(i)})$, using Lemma~\ref{lem:exp-gradient}, we can write

\begin{align*}
{y^{(i)} h(\vec{x}^{(i)})} =  & \exp \left( \left\| \vec{x}^{(i)}  \right\|^2 \right) +\sum_{j \neq i} y^{(i)}y^{(j)} \exp \left( \left\| \frac{\vec{x}^{(i)} + \vec{x}^{(j)}}{2} \right\|^2 \right)\\
\geq& \exp \left( \left\| \vec{x}^{(i)}  \right\|^2 \right) 
- \sum_{\substack{j \neq i \\ y^{(i)} \neq y^{(j)}} }  \exp \left( \left\| \frac{\vec{x}^{(i)} + \vec{x}^{(j)}}{2} \right\|^2 \right) \\
\geq& \exp \left( \left\| \vec{x}^{(i)}  \right\|^2 \right) \times  \left(1 - \sum_{\substack{j \neq i \\ y^{(i)} \neq y^{(j)}} } \exp \left(  \frac{\|\vec{x}^{(i)} + \vec{x}^{(j)}\|^2 - 4 \| \vec{x}^{(i)}\|^2}{4} \right) \right).
\end{align*}


Now, for any  $j$ such that $y^{(j)} \neq y^{(i)}$, we have that 

\begin{align*}
\|\vec{x}^{(i)}  + \vec{x}^{(j)}\|^2 - 4\| \vec{x}^{(i)} \|^2 &=  -3\|\vec{x}^{(i)} \|^2 + \|\vec{x}_1^{(j)} \|^2+ 2 \vec{x}_1^{(i)} \cdot \vec{x}_1^{(j)} \\
& \; \; \; + \underbrace{\|\vec{x}_2^{(j)}\|^2}_{Eq~\ref{eq:exp-norm-bound}}  + \underbrace{2 \vec{x}_2^{(i)} \cdot \vec{x}_2^{(j)}}_{Eq~\ref{eq:exp-dot-prod-bound}} \\
& \leq  -3\|\vec{u}\|^2 - 3\underbrace{\| \vec{x}_2^{(i)}\|^2}_{Eq~\ref{eq:exp-norm-bound}}  + \| \vec{u}\|^2 - 2\|\vec{u}\|^2  +c_3^2 D + \underbrace{\| \vec{x}_2^{(i)}\|}_{Eq~\ref{eq:exp-norm-bound}} \cdot 2 c_4\sqrt{2 \ln \frac{6m}{\delta}} \\
& \leq  -4 \| \vec{u} \|^2 -3c_2^2D +c_3^2 D + \underbrace{\sqrt{D} \cdot c_3 c_4\sqrt{2 \ln \frac{6m}{\delta}}}_{Eq~\ref{eq:d5}} \\
& \leq  -1 - \frac{45}{16}D + \frac{17}{16} D + \frac{1}{16}D  = -\frac{43}{16}D.\\
\end{align*}

Plugging this back in the previous equation we have that 

\begin{align*}
{y^{(i)} h(\vec{x}^{(i)})} \geq  & \geq \exp \left( \underbrace{\left\| \vec{x}^{(i)}  \right\|^2}_{Eq~\ref{eq:exp-norm-bound}} \right) \left(1-m \underbrace{\exp\left(-\frac{43}{64}D\right)}_{Eq~\ref{eq:d6}}\right) \\
& \geq \underbrace{\exp \left( \frac{15}{16}D\right)}_{Eq~\ref{eq:d6}} \cdot \frac{1}{2} \geq 1.
\end{align*}

Hence, $\vec{x}^{(i)}$ is correctly classified by a margin of $1$ for every $i$.

Now consider any test data point $(\vec{z},y)$. Since $D \geq \frac{1}{c_2} \ln \frac{2}{\epsilon}$ (Equation~\ref{eq:d5}), we have that with probability at least $1-\epsilon/2$ over the draws of $\vec{z}_2$,  by Corollary~\ref{cor:chi}

\begin{align*}
c_2 \sqrt{D} \leq \|\vec{z}_2 \| \leq c_3 \sqrt{D}. \numberthis \label{eq:z-exp-norm-bound}
\end{align*}

Similarly, for each $i$, we have that with probability at least $1-\epsilon/2m$ over the draws of $\vec{z}$, the following holds good by Corollary~\ref{cor:gaussian}
\begin{align*}
|\vec{x}_2^{(i)} \cdot \vec{z}_2| \leq \| \vec{x}_2^{(i)} \| \cdot c_4 \sqrt{2\ln \frac{6m}{\epsilon}}.   \numberthis\label{eq:z-exp-dot-prod-bound}
\end{align*} 

Hence, the above holds over at least $1-\epsilon/2$ draws of $\vec{z}$, and by extension, both the above equations hold over at least $1-\epsilon$ draws of $\vec{z}$.

Now, for any $i$ such that $y^{(i)} = y$, we have that 

\begin{align*}
\| \vec{x}^{(i)} + \vec{z}\|^2 =  &\|\vec{x}^{(i)}_1\|^2 + \|\vec{z}_1 \|^2 +  2 \vec{x}^{(i)}_1 \cdot \vec{z}_1 + \underbrace{\|\vec{x}^{(i)}_2\|^2}_{Eq~\ref{eq:exp-norm-bound}} + \underbrace{\|\vec{z}^{(i)}_2\|^2}_{Eq~\ref{eq:z-exp-norm-bound}}+  2 \underbrace{\vec{x}^{(i)}_2 \cdot \vec{z}_2}_{Eq~\ref{eq:z-exp-dot-prod-bound},~\ref{eq:exp-norm-bound}} \\ 
& \geq 4\|\vec{u}\|^2 + 2c_2^2 D - \underbrace{\sqrt{D} \cdot 2 c_3 c_4 \sqrt{2 \ln \frac{6m}{\epsilon}}}_{Eq~\ref{eq:d4}} \\
& \geq D + \frac{30}{16}D - \frac{1}{16} D = \frac{45}{16} D.
\end{align*}

Similarly, for any $i$ such that $y^{(i)} \neq y$, we have that 

\begin{align*}
\| \vec{x}^{(i)} + \vec{z}\|^2 =  &\|\vec{x}^{(i)}_1\|^2 + \|\vec{z}_1 \|^2 +  2 \vec{x}^{(i)}_1 \cdot \vec{z}_1 + \underbrace{\|\vec{x}^{(i)}_2\|^2}_{Eq~\ref{eq:exp-norm-bound}} + \underbrace{\|\vec{z}^{(i)}_2\|^2}_{Eq~\ref{eq:z-exp-norm-bound}} +  2 \underbrace{\vec{x}^{(i)}_2 \cdot \vec{z}_2}_{Eq~\ref{eq:exp-dot-prod-bound},~\ref{eq:exp-norm-bound}} \\ 
& \leq 2\|\vec{u}\|^2- 2\|\vec{u}\|^2 + 2c_3^2 D + \underbrace{\sqrt{D} \cdot 2 c_3 c_4 \sqrt{2 \ln \frac{6m}{\delta}}}_{Eq~\ref{eq:d5}} \\
& \leq \frac{34}{16}D - \frac{1}{16} D = \frac{33}{16}D.
\end{align*}

Since from Equation~\ref{eq:class-count} we know  there exists at least one training sample with a given label, we have that

\begin{align*}
{y^{(i)} h(\vec{x}^{(i)})} =  &\sum_{i: y^{(i)} = y}  \exp \left( \left\| \frac{\vec{x}^{(i)} + \vec{z}}{2} \right\|^2 \right) - \sum_{i: y^{(i)} \neq y}  \exp \left( \left\| \frac{\vec{x}^{(i)} + \vec{z}}{2} \right\|^2 \right) \\
&\geq \exp \left( \frac{45}{64}D \right) - m \exp \left( \frac{33}{64}D \right) \\
& \geq \exp \left( \frac{45}{64}D \right) \cdot \left(1 - m \underbrace{\exp \left( -\frac{12}{64}D \right)}_{Eq~\ref{eq:d6}} \right) \\
& \geq \underbrace{\exp \left( \frac{45}{64}D \right)}_{Eq~\ref{eq:d6}} \cdot \frac{1}{2} \geq 1.
\end{align*}

Thus, at least $1-\epsilon$ of the test datapoints are classified correctly. 

\end{proof}

We next show that the uniform convergence bound is nearly vacuous. In order to do this, we create a set $S'$ from $S$ by negating all values but the noise vector. We then show that for every point in $S'$, the contribution from the corresponding point in $S$ dominates over the contribution from all other points. (This is because of how the non-negated noise vector in the point from $S'$ aligns adversarially with the noise vector from the corresponding point in $S$).
 As a result, the points in $S'$ are all labeled like in $S$, implying that $S'$ is completely misclassified. Then, similar to our previous arguments, we can show that uniform convergence is nearly vacuous.\\

\begin{lemma}
\label{lem:exp-unif-conv}
In the setup of Section~\ref{sec:exp}, for any $\epsilon > 0$ and for any $\delta \leq 1/4$, and for the same lower bounds on $D$ and $m$ as in Theorem~\ref{thm:exp},  and for any $\gamma \geq 0$, we have that
\[
\epsilon_{\text{unif-alg}}(m,\delta) \geq 1 - \epsilon_{\text{gen}}(m,\delta)
\]
for the $\mathcal{L}^{(\gamma)}$ loss.
\end{lemma}

\begin{proof}
Let $S'$ be a modified version of the training set where all values are negated except that of the noise vectors i.e., $S' = \{ ((-\vec{x}_1, \vec{x}_2),-y) \; | \;  ((\vec{x}_1,  \vec{x}_2),y) \in S \}$.  First we show that with probability at least $1-2\delta/3$ over the draws of $S$, $S'$ is completely misclassified. First, we have that with probability $1-2\delta/3$, Equations~\ref{eq:exp-norm-bound} and ~\ref{eq:exp-dot-prod-bound} hold good. Let $(\vec{x}_{\textrm{neg}}^{(i)},y_{\textrm{neg}}^{(i)})$ denote  the $i$th sample from $S'$.
Then,
 we have that 

\begin{align*}
{y_{\textrm{neg}}^{(i)} h(\vec{x}_{\textrm{neg}}^{(i)})} = & -\exp \left( \left\| \frac{\vec{x}^{(i)} + \vec{x}_{\textrm{neg}}^{(i)} }{2} \right\|^2 \right) + \sum_{j \neq i} y_{\textrm{neg}}^{(i)}y^{(j)} \exp \left( \left\| \frac{\vec{x}_{\textrm{neg}}^{(i)} + \vec{x}^{(j)}}{2} \right\|^2 \right)\\
\leq& -\exp \left( \left\| \vec{x}_2^{(i)}  \right\|^2 \right) 
+ \sum_{\substack{j \neq i \\ y_{\textrm{neg}}^{(i)} = y^{(j)}} }  \exp \left( \left\| \frac{\vec{x}_{\textrm{neg}}^{(i)} + \vec{x}^{(j)}}{2} \right\|^2 \right) \\
\leq& \exp \left( \left\| \vec{x}_2^{(i)}  \right\|^2 \right) \times  \left(-1 + \sum_{\substack{j \neq i \\ y_{\textrm{neg}}^{(i)} = y^{(j)}} } \exp \left(  \frac{\|\vec{x}_{\textrm{neg}}^{(i)} + \vec{x}^{(j)}\|^2 - 4 \| \vec{x}_2^{(i)}\|^2}{4} \right) \right). \numberthis{\label{eq:exp-bad-margin}}
\end{align*}

Now, consider  $j$ such that $y^{(j)} = y_{\textrm{neg}}^{(i)}$.  we have that 

\begin{align*}
\|\vec{x}_{\textrm{neg}}^{(i)}  + \vec{x}^{(j)}\|^2 - 4\| \vec{x}_2^{(i)} \|^2 &=  \|\vec{x}_1^{(i)}\|^2 + \|\vec{x}_1^{(j)} \|^2 - 2 \vec{x}_1^{(i)} \cdot \vec{x}_1^{(j)}  - \underbrace{3\|\vec{x}_2^{(i)} \|^2}_{Eq~\ref{eq:exp-norm-bound}}  + \underbrace{\|\vec{x}_2^{(j)}\|^2}_{Eq~\ref{eq:exp-norm-bound}}  - \underbrace{2 \vec{x}_2^{(i)} \cdot \vec{x}_2^{(j)}}_{Eq~\ref{eq:exp-dot-prod-bound}} \\
& \leq  4\|\vec{u}\|^2 -  3c_2^2 D  +c_3 D + \underbrace{\| \vec{x}_2^{(i)}\|}_{Eq~\ref{eq:exp-norm-bound}} \cdot 2 c_4\sqrt{2 \ln \frac{6m}{\delta}} \\
& \leq  4 \| \vec{u} \|^2 -3c_2^2D +c_3^2 D + \underbrace{\sqrt{D} \cdot c_3 c_4\sqrt{2 \ln \frac{6m}{\delta}}}_{Eq~\ref{eq:d5}} \\
& \leq  D  - \frac{45}{16}D  + \frac{17}{16} D + \frac{1}{16}D = \frac{-11}{16} D. \\
\end{align*}

Plugging the above back in Equation~\ref{eq:exp-bad-margin}, we have
\begin{align*}
\frac{y_{\textrm{neg}}^{(i)} h(\vec{x}_{\textrm{neg}}^{(i)})}{ \exp \left( \left\| \vec{x}_2^{(i)}  \right\|^2 \right) }  \leq  -1 + m \underbrace{\exp \left( -11D/64\right)}_{Eq~\ref{eq:d6}} \leq -1/2,
\end{align*}
implying that $\vec{x}_{\textrm{neg}}^{(i)}$ is misclassified. This holds simultaneously for all $i$, implying that $S'$ is misclassified with high probability $1-2\delta/3$ over the draws of $S$. Furthermore, $S'$ has the same distribution as $\mathcal{D}^m$. Then, by the same argument as that of Lemma~\ref{lem:unif}, we can prove our final claim.

\end{proof}

\section{Further Remarks.}

In this section, we make some clarifying remarks about our theoretical results.

\subsection{Nearly vacuous bounds for any $\gamma > 0$.}
\label{sec:any-gamma}
Typically, like in \citet{mohri12foundations,bartlett17spectral}, the 0-1 test error is upper bounded in terms of the $\mathcal{L}^{(\gamma)}$ test error for some optimal choice of $\gamma > 0$ (as it is easier to apply uniform convergence for $\gamma > 0$). From the result in the main paper, it is obvious that for $\gamma \leq 1$, this approach would yield vacuous bounds. We now establish that this is the case even for $\gamma > 1$.

To help state this more clearly, for the scope of this particular section,
 let $\epsilon^{(\gamma)}_{\textrm{unif-alg}}, \epsilon^{(\gamma)}_{\textrm{gen}}$ denote the uniform convergence and generalization error for $\mathcal{L}^{(\gamma)}$ loss. Then, 
 the following inequality is used to derive a bound on the 0-1 error:
\begin{align*}
\mathcal{L}^{(0)}_{\mathcal{D}}(h_S) \leq \mathcal{L}^{(\gamma)}_{\mathcal{D}}(h_S) \leq \hat{\mathcal{L}}^{(\gamma)}_{S}(h_S) + \epsilon^{(\gamma)}_{\textrm{unif-alg}}(m,\delta) \numberthis\label{eq:margin-upper-bound}
\end{align*}
where the second inequality above holds with probability at least $1-\delta$ over the draws of $S$, while the first holds for all $S$ (which follows by definition of $\mathcal{L}^{(\gamma)}$ and $\mathcal{L}^{(0)}$).

To establish that uniform convergence is nearly vacuous in any setting of $\gamma$, we must show that the right hand side of the above bound is nearly vacuous for any choice of $\gamma \geq 0$ (despite the fact that $\mathcal{L}_{\mathcal{D}}^{(0)}(S) \leq \epsilon$). In our results, we explicitly showed this to be true for only small values of $\gamma$, by arguing that the second term in the R.H.S, namely $\epsilon^{(\gamma)}_{\textrm{unif-alg}}(m,\delta)$, is nearly vacuous.

Below, we show that the above bound is indeed nearly vacuous for any value of $\gamma$, when we have that $\epsilon_{\text{unif-alg}}^{(\gamma)}(m,\delta) \geq 1- \epsilon_{\text{gen}}^{(\gamma)}(m,\delta)$. Note that we established the relation $\epsilon_{\text{unif-alg}}^{(\gamma)}(m,\delta) \geq 1- \epsilon_{\text{gen}}^{(\gamma)}(m,\delta)$ to be true in all of our setups. \\

\begin{proposition}
Given that for all $\gamma \geq 0$,
$\epsilon_{\text{unif-alg}}^{(\gamma)}(m,\delta) \geq 1- \epsilon_{\text{gen}}^{(\gamma)}(m,\delta)$
then, we then have that for all $\gamma \geq 0$,
\[
Pr_{S \sim \mathcal{D}^m}\left[ \hat{\mathcal{L}}^{(\gamma)}_{S}(h_S) + \epsilon^{(\gamma)}_{\textrm{unif-alg}}(m,\delta) \geq \frac{1}{2} \right] > \delta
\]
or in other words, 
the guarantee from the right hand side of Equation~\ref{eq:margin-upper-bound} is nearly vacuous.
\end{proposition}

\begin{proof}
Assume on the contrary that for some choice of $\gamma$, we are able to show that with probability at least $1-\delta$ over the draws of $S$, the right hand side of Equation~\ref{eq:margin-upper-bound} is less than $1/2$. This means that $\epsilon^{(\gamma)}_{\textrm{unif-alg}}(m,\delta) < 1/2$. Furthermore, this also means that with probability at least $1-\delta$ over the draws of $S$, $\hat{\mathcal{L}}^{(\gamma)}_{S}(h_S) < 1/2$ and  $\mathcal{L}^{(\gamma)}_{\mathcal{D}}(h_S) < 1/2$ (which follows from the second inequality in Equation~\ref{eq:margin-upper-bound}). 

As a result, we have that with probability at least $1-\delta$, $\mathcal{L}^{(\gamma)}_{\mathcal{D}}(h_S) - \hat{\mathcal{L}}^{(\gamma)}_{S}(h_S) < 1/2$. In other words, $\epsilon_{\textrm{gen}}^{(\gamma)}(m,\delta) < 1/2$. 
Since we are given that $\epsilon_{\text{unif-alg}}^{(\gamma)}(m,\delta) \geq 1- \epsilon_{\text{gen}}^{(\gamma)}(m,\delta)$, by our upper bound on the generalization error, we have $\epsilon_{\text{unif-alg}}^{(\gamma)}(m,\delta) \geq 1/2$, which is a contradiction to our earlier inference that $\epsilon_{\text{unif-alg}}^{(\gamma)}(m,\delta) < 1/2$. Hence, our assumption is wrong. 

\end{proof}

\subsection{Applicability of the observation in Section~\ref{sec:hypersphere} to other settings}
\label{sec:applicability}
Recall that in the main paper, we discussed a setup where two hyperspheres of radius $1$ and $1.1$ respectively are classified by a sufficiently overparameterized ReLU network. We saw that even when the number of training examples was as large as $65536$, we could project all of these examples on to the other corresponding hypersphere, to create a completely misclassified set $S'$. How well does this observation extend to other hyperparameter settings?

First, we note that in order to achieve full misclassification of $S'$, the network would have to be sufficiently overparameterized i.e., either the width or the input dimension must be larger. When the training set size $m$ is too large, one would observe that $S'$ is not as significantly misclassified as observed. (Note that on the other hand, increasing the parameter count would not hurt the generalization error. In fact it would improve it.)

Second, we note that our observation is sensitive to the choice of the difference in the radii between the hyperspheres (and potentially to other hyperparameters too). For example, when the outer sphere has radius $2$, SGD learns to classify these spheres perfectly, resulting in zero error on both test data and on $S'$.  As a result, our lower bound on $\epsilon_{\text{unif-alg}}$ would not hold in this setting. 

However, here we sketch a (very) informal argument as to why there is reason to believe that our lower bound can still hold on a weaker notion of uniform convergence, a notion that is always applied in practice (in the main paper we focus on a strong notion of uniform convergence as a negative result about it is more powerful). More concretely, in reality, uniform convergence is computed without much knowledge about the data distribution, save a few weakly informative assumptions such as those bounding its support. Such a uniform convergence bound is effectively computed uniformly in supremum over a class of distributions.

Going back to the hypersphere example, the intuition is that even when the radii of the spheres are far apart, and hence, the classification perfect, the decision boundary learned by the network could still be microscopically complex -- however these complexities are not exaggerated enough to misclassify $S'$. Now, for this given decision boundary, one would be able to construct an $S''$ which corresponds to projecting $S$ on two concentric hyperspheres that fall within these skews. Such an $S''$ would have a distribution that comes from some $\mathcal{D}'$ which, although not equal to $\mathcal{D}$, still obeys our assumptions about the underlying distribution. The uniform convergence bound which also holds for $\mathcal{D'}$ would thus have to be vacuous.

\subsection{On the dependence of $\epsilon_{\textrm{gen}}$ on $m$ in our examples.}
\label{sec:remark}

As seen in the proof of Lemma~\ref{lem:gen}, the generalization error $\epsilon$ depends on $m$ and $D$ as $\mathcal{O}(e^{-D/m})$ ignoring some constants in the exponent. Clearly, this error decreases with the parameter count $D$. 

On the other hand, one may also observe that this generalization error grows with the number of samples $m$, which might at first make this model seem inconsistent with our real world observations. However, we emphasize that this is a minor artefact of the simplifications in our setup, rather than a conceptual issue. With a small modification to our setup, we can make the generalization error decrease with $m$, mirroring our empirical observations. Specifically, in the current setup, we learn the true boundary along the first $K$ dimensions exactly. We can however modify it to a more standard learning setup where the boundary is not exactly recoverable and needs to be estimated from the examples. This would lead to an additional generalization error that scales as $\mathcal{O}(\sqrt{\frac{K}{m}})$ that is non-vacuous as long as $K \ll m$. Thus, the overall generalization error would be $\mathcal{O}(e^{-D/m} + \sqrt{\frac{K}{m}})$. 

What about the overall dependence on $m$? Now, assume we have an overparameterization level of $D \gg m \ln (m/K)$, so that $e^{-D/m} \ll \sqrt{K/m}$. Hence, in the sufficiently overparameterized regime, the generalization error $\mathcal{O}(e^{-D/m})$ that comes from the noise we have modeled, pales in comparison with the generalization error that would stem from estimating the low-complexity boundary. Overall, as a function of $m$, the resulting error would behave like $\mathcal{O}(\sqrt{\frac{K}{m}})$ and hence show a decrease with increasing $m$ (as long the increase in $m$ is within the overparameterized regime). 

\subsection{Failure of hypothesis-dependent uniform convergence bounds.}

\label{sec:weight-dependence}
Often, uniform convergence bounds are written as a bound on the generalization error of a specific hypothesis rather than the algorithm. These bounds have an explicit dependence on the weights learned. As an example, a bound may be of the form that, with high probability over draws of training set $\tilde{S}$, for any hypothesis $h$ with weights $\vec{w}$,

\[
 \mathcal{L}_{\mathcal{D}}(h) -
 \hat{\mathcal{L}}_{\tilde{S}}(h)  \leq \frac{\|\vec{w}\|_2}{\sqrt{m}}.
\]

Below we argue why even these kinds of hypothesis-dependent bounds fail in our setting. We can informally define the tightest hypothesis-dependent uniform convergence bound as follows, in a manner similar to Definition~\ref{def:unif-alg} of the tightest uniform convergence bound. 
Recall that we first pick a set of datasets $\mathcal{S}_{\delta}$ such that $Pr_{\tilde{S} \sim \mathcal{D}^m}[\tilde{S} \notin \mathcal{S}_{\delta}] \leq \delta$. Then, for all $\tilde{S} \in S_{\mathcal{\delta}}$, we denote
the upper bound on the generalization gap of $h_{\tilde{S}}$ by $\epsilon_{\text{unif-alg}}(h_{\tilde{S}},m,\delta)$, where:

\[
\epsilon_{\text{unif-alg}}(h_{\tilde{S}},m,\delta) := \sup_{\tilde{S} \in \mathcal{S}_{\delta}} |\mathcal{L}_D(h_{\tilde{S}}) - \hat{\mathcal{L}}_{S}(h_{\tilde{S}})|.
\]

In other words, the tightest upper bound here corresponds to the difference between the test and empirical error of the specific hypothesis $h_{\tilde{S}}$ but computed across nearly all datasets $S$ in $\mathcal{S}_{\delta}$.

To show failure of the above bound, recall from all our other proofs of failure of uniform convergence, we have that for at least $1-O(\delta)$ draws of the sample set ${\tilde{S}}$, four key conditions are satisfied: (i) ${\tilde{S}} \in \mathcal{S}_{\delta}$, (ii) the corresponding bad dataset  ${\tilde{S}}' \in \mathcal{S}_{\delta}$, (iii) the error on the bad set $\hat{\mathcal{L}}_{{\tilde{S}}'}(h_{\tilde{S}})=1$ and (iv) the test error $\mathcal{L}_D(h_{\tilde{S}}) \leq \epsilon_{\text{gen}}(m,\delta)$.  For all such ${\tilde{S}}$, in the definition of $\epsilon_{\text{unif-alg}}(h_{\tilde{S}},m,\delta)$, let us set $S$ to be ${\tilde{S}}'$. Then, we would get $\epsilon_{\text{unif-alg}}(h_{\tilde{S}},m,\delta) \geq |\mathcal{L}_D(h_{\tilde{S}}) - \hat{\mathcal{L}}_{{\tilde{S}}'}(h_{\tilde{S}})| \geq 1-\epsilon_{\text{gen}}(m,\delta)$. In other words, with probability at least $1-O(\delta)$ over the draw of the training set, even a hypothesis-specific generalization bound fails to explain generalization of the corresponding hypothesis.

\section{An abstract setup}
\label{sec:warm-up}

We now present an abstract setup that, although unconventional in some ways, conveys the essence behind how uniform convergence fails to explain generalization. Let the underlying distribution over the inputs be a spherical Gaussian in $\mathbb{R}^D$ where $D$ can be however small or large as the reader desires. Note that our setup would apply to many other distributions, but a Gaussian would make our discussion easier. Let the labels of the inputs be determined by some $h^{\star}: \mathbb{R}^{D} \to \{-1, +1 \}$.
Consider a scenario where the learning algorithm outputs a very slightly modified version of $h^\star$. Specifically, let $S' = \{-\vec{x} \; | \; \vec{x} \in S\}$; then, the learner outputs
\[
h_{S}(\vec{x}) = \begin{cases}
-h^\star(\vec{x}) & \text{if } \vec{x} \in S' \\
h^\star(\vec{x}) & \text{otherwise}.
\end{cases} 
\]
That is, the learner misclassifies inputs that correspond to the negations of the samples in the training data -- this would be possible if and only if the classifier is overparameterized with $\Omega(mD)$ parameters to store $S'$. We will show that uniform convergence fails to explain generalization for this learner. 

First we establish that this learner generalizes well.
Note that a given $S$ has zero probability mass under $\mathcal{D}$, and so does $S'$. Then,
the training and test error are zero -- except for pathological draws of $S$ that intersect with $S'$, which are almost surely never drawn from $\mathcal{D}^m$ -- and hence, the generalization error of $\mathcal{A}$ is zero too.

It might thus seem reasonable to expect that one could explain this generalization using implicit-regularization-based uniform convergence by showing $\epsilon_{\text{unif-alg}}(m,\delta)=0$. Surprisingly, this is not the case as $\epsilon_{\text{unif-alg}}(m,\delta)$ is in fact $1$! 

First it is easy to see why the looser bound $\epsilon_{\text{unif}}(m,\delta)$ equals 1, if we let $\mathcal{H}$ be the space of all hypotheses the algorithm could output: there must exist a non-pathological $S \in \mathcal{S}_{\delta}$, and we know that $h_{S'} \in \mathcal{H}$ misclassifies the negation of its training set, namely $S$. Then, $\sup_{h \in \mathcal{H}} |\mathcal{L}_{\mathcal{D}}(h) - \hat{\mathcal{L}}_{S}(h)| = |\mathcal{L}_{\mathcal{D}}(h_{S'}) - \hat{\mathcal{L}}_{S}(h_{S'})|  = |0-1|=1$.

One might hope that in the stronger bound of $\epsilon_{\text{unif-alg}}(m,\delta)$ since we truncate the hypothesis space, it is possible that the above adversarial situation would fall apart. However, with a more nuanced argument, we can similarly show that 
 $\epsilon_{\text{unif-alg}}(m,\delta)=1$. First, recall that any bound on $\epsilon_{\text{unif-alg}}(m,\delta)$,  would have to pick a truncated sample set space $\mathcal{S}_{\delta}$. Consider any choice of $\mathcal{S}_{\delta}$, and the corresponding set of explored hypotheses $\mathcal{H}_{\delta}$. We will show that for any choice of $\mathcal{S}_{\delta}$, there exists $S_{\star} \in \mathcal{S}_{\delta}$ such that (i) $h_{S_\star}$ has zero test error and (ii) the negated training set $S_{\star}'$ belongs to $\mathcal{S}_{\delta}$ and (iii) $h_{S_\star}$ has error $1$ on $S_\star$. Then, it follows that 
$\epsilon_{\text{unif-alg}}(m,\delta) = \sup_{S \in \mathcal{S}_{\delta}} \sup_{h \in \mathcal{H}_{\delta}} |{\mathcal{L}}_{\mathcal{D}}(h)- \hat{\mathcal{L}}_{S}(h)| \geq  |{\mathcal{L}}_{\mathcal{D}}(h_{S_\star})- \hat{\mathcal{L}}_{S_\star'}(h)|= |0-1| = 1 $.


We can prove the existence of such an $S_{\star}$ by showing that the probability of picking one such set under $\mathcal{D}^{m}$ is non-zero for $\delta < 1/2$.  Specifically, under $S \sim \mathcal{D}^m$, we have by the union bound that
\begin{align*}
& Pr \left[  \mathcal{L}_{\mathcal{D}}(h_S) = 0, \hat{\mathcal{L}}_{S'}(h_{S}) = 1, S \in \mathcal{S}_{\delta}, S' \in \mathcal{S}_{\delta}\right] \geq  \\
& 1- Pr\left[  \mathcal{L}_{\mathcal{D}}(h_S) \neq 0,  \hat{\mathcal{L}}_{S'}(h_{S}) \neq 1\right] 
- Pr\left[   S \notin \mathcal{S}_{\delta}\right] 
- Pr\left[   S' \notin \mathcal{S}_{\delta}\right] .
\end{align*}
Since the pathological draws have probability zero, the first probability term on the right hand side is zero. The second term is at most $\delta$ by definition of $\mathcal{S}_{\delta}$. Crucially, the last term too is at most $\delta$ because $S'$ (which is the negated version of $S$) obeys the same distribution as $S$ (since the isotropic Gaussian is invariant to a negation). Thus, the above probability is at least $1-2\delta > 0$, implying that there exist (many) $S_\star$, proving our main claim.

\subparagraph{Remark.} While our particular learner might seem artificial, much of this artificiality is only required to make the argument simple. The crucial trait of the learner that we require is that the misclassified region in the input space (i) covers low probability and yet (ii) is complex and highly dependent on the training set draw. Our intuition is that SGD-trained deep networks possess these traits.

\section{Learnability and Uniform Convergence}

\label{sec:learnability}

Below, we provide a detailed discussion on learnability, uniform convergence and generalization. Specifically, we argue why the fact that uniform convergence is necessary for learnability does not preclude the fact that uniform convergence maybe unable to {\em explain} generalization of a particular algorithm for a particular distribution.

We first recall the notion of learnability. First, formally, a binary classification problem consists of a hypothesis class $\mathcal{H}$ and an instance space $\mathcal{X} \times \{-1,1\}$. 
The problem is said to be {\em learnable} if there exists a learning rule $\mathcal{A}': \bigcup\limits_{m=1}^{\infty} \mathcal{Z}^m \to \mathcal{H}$ and a monotonically decreasing sequence $\epsilon_{\text{lnblty}}(m)$ such that $\epsilon_{\text{lnblty}}(m) \xrightarrow{m\to\infty} 0$ and
\begin{align*}
\forall \mathcal{D}' \;
\mathbb{E}_{S \sim \mathcal{D}'^m}\left[ 
\mathcal{L}^{(0)}_{\mathcal{D}'}(\mathcal{A}'(S)) -
\min_{h \in \mathcal{H}} \mathcal{L}^{(0)}_{\mathcal{D}'}(h)  \right] \leq \epsilon_{\text{lnblty}}(m).
\numberthis\label{eq:learnability}
\end{align*}

 \citet{vapnik71uniform} showed that finite VC dimension  of the hypothesis class is  necessary and sufficient for {learnability} in binary classification problems. As \citet{shwartz10learnability} note, since finite VC dimension is equivalent to uniform convergence, it can thus be concluded that uniform convergence is necessary and sufficient for learnability binary classification problems.

However, learnability is a strong notion that does not necessarily have to hold for a particular learning algorithm to generalize well for a particular underlying distribution. Roughly speaking, this is because learnability evaluates the algorithm under all possible distributions, including many complex distributions; while a learning algorithm may generalize well for a particular distribution under a given hypothesis class, it may fail to do so on more complex distributions under the same hypothesis class.

For more intuition, we present a more concrete but informal argument below. However, this argument is technically redundant because learnability is equivalent to uniform convergence for binary classification, and since we established the lack of necessity of uniform convergence, we effectively established the same for learnability too. However, we still provide the following informal argument as it provides a different insight into why learnability and uniform convergence are not necessary to explain generalization.

Our goal is to establish that in the set up of Section~\ref{sec:setup}, even if we considered the binary classification problem corresponding to $\mathcal{H}_{\delta}$ (the class consisting of only those hypotheses explored by the algorithm $\mathcal{A}$ under a distribution $\mathcal{D}$), the corresponding binary classification problem is not learnable i.e., Equation~\ref{eq:learnability} does not hold when we plug in $\mathcal{H}_{\delta}$ in place of $\mathcal{H}$. 

First consider distributions of the following form that is more complex than the linearly separable $\mathcal{D}$: for any dataset $S'$, let $\mathcal{D}_{S'}$ be the distribution that has half its mass on the part of the linearly separable distribution $\mathcal{D}$ excluding $S'$,
and half its mass on the distribution that is uniformly distributed over $S'$. Now let $S'$ be a random dataset drawn from $\mathcal{D}$ {\em but with all its labels flipped}; consider the corresponding complex distribution $\mathcal{D}_{S'}$.

We first show that there exists $h \in \mathcal{H}_{\delta}$ that fits this distribution well.
Now, for most draws of the ``wrongly'' labeled $S'$, we can show that the hypothesis $h$ for which $\vec{w}_1 = 2 \cdot \vec{u}$ and $\vec{w}_2 = \sum_{(x,y) \in S'} y \cdot \vec{x}_2$ fits the ``wrong'' labels of $S'$ perfectly; this is because, just as argued in Lemma~\ref{lem:unif}, $\vec{w}_2$ dominates the output on all these inputs, although $\vec{w}_1$ would be aligned incorrectly with these inputs. Furthermore, since $\vec{w}_2$ does not align with most inputs from $\mathcal{D}$, by an argument similar to Lemma~\ref{lem:gen}, we can also show that this hypothesis has at most $\epsilon $ error on $\mathcal{D}$, and that this hypothesis belongs to $\mathcal{H}_{\delta}$.  Overall this means that, w.h.p over the choice of $S'$, there exists a hypothesis $h \in \mathcal{H}_{\delta}$ for which the error on the complex distribution $\mathcal{D}_{S'}$ is at most $\epsilon/2$ i.e., 

\[
\min_{h \in \mathcal{H}} \mathbb{E}_{(x,y) \sim \mathcal{D}_{S'}}[\mathcal{L}(h(x),y)]  \leq \epsilon/2.
\]

On the other hand, let $\mathcal{A}'$ be any learning rule which outputs a hypothesis given $S \sim \mathcal{D}_{S'}$. With high probability over the draws of $S \sim \mathcal{D}_{S'}$, only at most, say $3/4$th of $S$ (i.e., $0.75 m$ examples) will be sampled from  $S'$ (and the rest from $\mathcal{D}$). Since the learning rule which has access only to $S$, has not seen at least a quarter of $S'$, with high probability over the random draws of $S'$, the learning rule will fail to classify roughly half of the unseen 
examples from $S'$ correctly (which would be about $(m/4) \cdot 1/2 = m/8$). Then, the error on $\mathcal{D}_{S'}$ will be at least $1/16$.  From the above arguments, we have that $\epsilon_{\text{learnability}}(m) \geq 1/16 - \epsilon/2$, which is a non-negligible constant that is independent of $m$.

\section{Deterministic PAC-Bayes bounds are two-sided uniform convergence bounds}
\label{sec:pac-bayes}

By definition, VC-dimension, Rademacher complexity and other covering number based bounds are known to upper bound the term $\epsilon_{\text{unif-alg}}$ and therefore our negative result immediately applies to all these bounds. However, it may not be immediately clear if bounds derived through the PAC-Bayesian approach fall under this category too. In this discussion, we show that existing deterministic PAC-Bayes based bounds are in fact two-sided in that they are lower bounded by $\epsilon_{\text{unif-alg}}$ too. 

For a given prior distribution $P$ over the  parameters, a PAC-Bayesian bound is of the following form: with high probability $1-\delta$ over the draws of the data $S$, we have that {\em for all distributions} $Q$ over the hypotheses space:
\begin{align*}
& KL\left( \left. \mathbb{E}_{\tilde{h} \sim Q}[\hat{\mathcal{L}}_{S}(\tilde{h}) ]\right \| \mathbb{E}_{\tilde{h} \sim Q}[\mathcal{L}_{\mathcal{D}}(\tilde{h}) ] \right) \leq  \underbrace{\frac{KL(Q \| P) + \ln \frac{2m}{\delta}}{m-1}}_{:= \epsilon_{\textrm{pb}}(P,Q,m,\delta)}. \numberthis \label{eq:pb}
\end{align*}

Note that here for any $a,b \in [0,1]$, $KL(a\| b) = a \ln \frac{a}{b} + (1-a) \ln \frac{1-a}{1-b}$. Since the precise form of the PAC-Bayesian bound on the right hand side is not relevant for the rest of the discussion, we will concisely refer to it as $\epsilon_{\textrm{pb}}(P,Q,m,\delta)$. What is of interest to us is the fact that the above bound holds for all $Q$ for most draws of $S$ and that the KL-divergence on the right-hand side is in itself two-sided, in some sense.

Typically, the above  bound is simplified to derive the following one-sided bound on the difference between the expected and empirical errors of a stochastic network (see \cite{mcallester03simplified} for example):
\begin{align*}
 &\mathbb{E}_{\tilde{h} \sim Q}[\mathcal{L}_{\mathcal{D}}(\tilde{h}) ] -  \mathbb{E}_{\tilde{h} \sim Q}[\hat{\mathcal{L}}_{S}(\tilde{h}) ] \leq \sqrt{2\epsilon_{\textrm{pb}}(P,Q,m,\delta)} + 2\epsilon_{\textrm{pb}}(P,Q,m,\delta).
 \numberthis \label{eq:stochastic-pac-bayes}
\end{align*}

This bound is then manipulated in different ways to obtain bounds on the deterministic network. In the rest of this discussion, we focus on the two major such derandomizing techniques and argue that both these techniques boil down to two-sided convergence. While, we do not formally establish that there may exist other techniques which ensure that the resulting deterministic bound is strictly one-sided, we suspect that no such techniques may exist. This is because the KL-divergence bound in Equation~\ref{eq:pb} is in itself two-sided in the sense that for the right hand side bound to be small, both the stochastic test and train errors must be close to each other; it is not sufficient if the stochastic test error is smaller than the stochastic train error.

\subsection{Deterministic PAC-Bayesian Bounds of Type A}
To derive a deterministic generalization bound, one approach is to add extra terms that account for the perturbation in the loss of the network \cite{neyshabur17exploring,mcallester03simplified,nagarajan2018deterministic}. That is, define:
\begin{align*}
\Delta(h,Q,\mathcal{D}) &= | \mathcal{L}_{\mathcal{D}}(h) -  \mathbb{E}_{\tilde{h} \sim Q}[\mathcal{L}_{\mathcal{D}}(\tilde{h}) ]|, \\
\Delta(h,Q,S) &= \left|  \hat{\mathcal{L}}_{S}(h) -  \mathbb{E}_{\tilde{h} \sim Q}[\hat{\mathcal{L}}_{S}(\tilde{h}) ]\right|. 
\end{align*}
Then, one can get a deterministic upper bound as:
\begin{align*}
  &\mathcal{L}_{\mathcal{D}}(h)-   \hat{\mathcal{L}}_{S}(h) \leq\sqrt{2\epsilon_{\text{pb}}(P,Q_h,m,\delta)} +  2\epsilon_{\text{pb}}(P,Q_h,m,\delta)  + { \Delta(h,Q,\mathcal{D}) + \Delta(h,Q,S)}.
\end{align*}

Note that while applying this technique, for any hypothesis $h$, one picks a posterior $Q_h$ specific to that hypothesis (typically, centered at that hypothesis). 

We formally define the deterministic bound resulting from this technique below. We consider the algorithm-dependent version and furthermore, we consider a bound that results from the best possible choice of $Q_h$ for all $h$. 
We define this deterministic bound in the format of $\epsilon_{\text{unif-alg}}$ as follows: \\

\begin{definition} 
The distribution-dependent, algorithm-dependent, deterministic PAC-Bayesian bound of (the hypothesis class $\mathcal{H}$, algorithm $\mathcal{A}$)-pair with respect to $\mathcal{L}$ is defined to be the smallest value $\epsilon_{\text{pb-det-A}}(m, \delta)$ such that the following holds:
\begin{enumerate}
\item there exists a set of $m$-sized samples $\mathcal{S}_{\delta} \subseteq (\mathcal{X} \times \{-1, +1 \})^m$ for which:
\[
Pr_{S \sim \mathcal{D}^m} [S \notin \mathcal{S}_{\delta}]  \leq \delta,
\] 
\item and if we define $\mathcal{H}_{\delta} = \bigcup_{S \in \mathcal{S}_{\delta}}  \{ h_{S} \}$ to be the space of hypotheses explored only on these samples, then there must exist a prior $P$ and for each $h \in \mathcal{H}_{\delta}$, a distribution $Q_h$, such that uniform convergence must hold as follows:
\begin{align*}
& \sup_{S \in \mathcal{S}_{\delta}}\ \sup_{h \in \mathcal{H}_{\delta}}  \sqrt{2\epsilon_{\text{pb}}(P,Q_h,m,\delta)} +  2\epsilon_{\text{pb}}(P,Q_h,m,\delta)\\
& + \Delta(h,Q_h,\mathcal{D}) + \Delta(h,Q_h,S) < \epsilon_{\text{pb-det-A}}(m,\delta), 
\numberthis \label{eq:pb-det-A}
\end{align*}
as a result of which, by Equation~\ref{eq:stochastic-pac-bayes}, the following one-sided uniform convergence also holds:
\begin{equation}
\label{eq:pb-one-side}
\sup_{S \in \mathcal{S}_{\delta}}\sup_{h \in \mathcal{H}_{\delta}}  \mathcal{L}_{\mathcal{D}}(h)-   \hat{\mathcal{L}}_{S}(h) < \epsilon_{\text{pb-det-A}}(m,\delta).
\end{equation}
\end{enumerate}

\end{definition}


Now, recall that $\epsilon_{\text{unif-alg}}(m,\delta)$ is a two-sided bound, and in fact our main proof crucially depended on this fact in order to lower bound   $\epsilon_{\text{unif-alg}}(m,\delta)$. Hence, to extend our lower bound to $\epsilon_{\text{pb-det-A}}(m,\delta) $ we need to show that it is also two-sided in that it is lower bounded by 
 $\epsilon_{\text{unif-alg}}(m,\delta)$.  The following result establishes this:

\begin{theorem}
\label{thm:pb-det-A}
Let $\mathcal{A}$ be an algorithm such that on at least $1-\delta$ draws of the training dataset $S$, the algorithm outputs a hypothesis $h_S$ that has $\hat{\epsilon}(m,\delta)$ loss on the training data $S$. Then

\begin{align*}
&e^{-3/2} \cdot \epsilon_{\text{unif-alg}}(m,3\delta) - (1-e^{-3/2})(\hat{\epsilon}(m,\delta) + \epsilon_{\text{gen}}(m,\delta)) \leq \epsilon_{\textrm{pb-det-A}} (m,\delta).
\end{align*}
 
\end{theorem}

\begin{proof}


First, by the definition of the generalization error, we know that with probability at least $1-\delta$ over the draws of $S$,
 \[{\mathcal{L}}_{D}(h_S)\leq \hat{\mathcal{L}}_{S}(h_S)  + \epsilon_{\text{gen}}(m,\delta).  \]

Furthermore since the training loss it at most $\hat{\epsilon}(m,\delta)$ on at least $1-\delta$ draws 
 we have that on at least $1-2\delta$ draws of the dataset,
 \[{\mathcal{L}}_{D}(h_S)\leq \hat{\epsilon}(m,\delta) + \epsilon_{\text{gen}}(m,\delta).  \]
 Let $\mathcal{H}_{\delta}$
 and $\mathcal{S}_{\delta}$ be the subset of hypotheses and sample sets as in the definition of $\epsilon_{\text{pb-det-A}}$. Then, from the above,
 there exist $\mathcal{H}_{3\delta} \subseteq \mathcal{H}_{\delta}$ and $\mathcal{S}_{3\delta} \subseteq \mathcal{S}_{\delta}$ such that

\[
Pr_{S \sim \mathcal{D}^m} [S \notin \mathcal{S}_{3\delta}]  \leq 3\delta
\]
and $\mathcal{H}_{3\delta} = \bigcup_{S \in \mathcal{S}_{3\delta}}  \{ h_{S} \}$, and furthermore,  \begin{equation*}\sup_{h\in\mathcal{H}_{3\delta}}{\mathcal{L}}_{\mathcal{D}}(h) \leq \hat{\epsilon}(m,\delta) + \epsilon_{\text{gen}}(m,\delta).
\end{equation*}.

Using the above, and the definition of $\Delta$, we have for all $h\in\mathcal{H}_{3\delta}$, the following upper bound on its stochastic test error:

 \begin{align*} &  \mathbb{E}_{\tilde{h} \sim Q_h}[ {\mathcal{L}}_{\mathcal{D}}(\tilde{h}) ] \leq {\mathcal{L}}_{\mathcal{D}}(h)  +\Delta(h,Q_h,\mathcal{D})  \leq \hat{\epsilon}(m,\delta) + \epsilon_{\text{gen}}(m,\delta) + \underbrace{\Delta(h,Q_h,\mathcal{D})}_{\text{applying Equation}~\ref{eq:pb-det-A}}\\
 &  \leq {\hat{\epsilon}(m,\delta) + \epsilon_{\text{gen}}(m,\delta) + \epsilon_{\text{pb-det-A}}(m,\delta)}. \numberthis \label{eq:stochastic-test-ub}
\end{align*}.

Now, for each pair of $h\in\mathcal{H}_{3\delta}$ and $S \in \mathcal{S}_{3\delta}$, we will bound its empirical error minus the expected error in terms of $\epsilon_{\text{pb-det-A}}(m,\delta)$. For convenience, let us denote by $a:=\mathbb{E}_{\tilde{h} \sim Q_h}[ \hat{\mathcal{L}}_{S}(\tilde{h}) ]$ and  $b := \mathbb{E}_{\tilde{h} \sim Q_h}[ {\mathcal{L}}_{\mathcal{D}}(\tilde{h}) ]$ (note that $a$ and $b$ are terms that depend on a hypothesis $h$ and a sample set $S$).  
 
We consider two cases. 
First, for some $h\in\mathcal{H}_{3\delta}$ and $S \in \mathcal{S}_{3\delta}$, consider the case that $e^{3/2} b > a$. Then, we have
\begin{align*}
  \hat{\mathcal{L}}_{S}(h)-  {\mathcal{L}}_{\mathcal{D}}(h)   \leq & 
a-b + \underbrace{\Delta(h,Q_h,\mathcal{D}) + \Delta(h,Q_h,S)}_{\text{applying Equation}~\ref{eq:pb-det-A}}\\
\leq & (e^{3/2} - 1)\underbrace{b}_{\text{apply Equation}~\ref{eq:stochastic-test-ub}} + \epsilon_{\text{pb-det-A}}(m,\delta)  \\
\leq & (e^{3/2} - 1) ({\hat{\epsilon}}(m,\delta) 
+ \epsilon_{\text{gen}}(m,\delta) + \epsilon_{\text{pb-det-A}}(m,\delta)) \\
&+ \epsilon_{\text{pb-det-A}}(m,\delta)  \\
\leq & (e^{3/2} - 1) ({\hat{\epsilon}}(m,\delta) + \epsilon_{\text{gen}}(m,\delta))+
 e^{3/2} \cdot \epsilon_{\text{pb-det-A}}(m,\delta). \numberthis \label{eq:case-1} \\
\end{align*}

Now consider the case where $a > e^{3/2} b$. This means that $(1-a) < (1-b)$. Then, if we consider the PAC-Bayesian bound of Equation~\ref{eq:pb},

\begin{equation}
\label{eq:pb-a-b}
 a \ln \frac{a}{b} + (1-a) \ln \frac{1-a}{1-b} \leq \epsilon_{\textrm{pb}}(P,Q_h,m,\delta),
\end{equation}
on the second term, 
we can apply the inequality
$\ln x \geq \frac{(x-1)(x+1)}{2x} = \frac{1}{2}\left( x - \frac{1}{x} \right)$ which holds for $x \in [0,1]$ to get:

\begin{align*}
 (1-a) \ln \frac{1-a}{1-b} \geq\frac{1}{2} (1-a) \left( \frac{1-a}{1-b} - \frac{1-b}{1-a}\right) &=\left(  \frac{(b-a)(2-a-b)}{2(1-b)}\right) \\
&\geq - (a-b)\left(  \frac{(2-a-b)}{2(1-b)}\right) \geq  - (a-b)\left(  \frac{(2-b)}{2(1-b)}\right)\\
& \geq  - \frac{(a-b)}{2}\left(  \frac{1}{(1-b)} + 1\right).
\end{align*}
Plugging this back in Equation~\ref{eq:pb-a-b}, we have,

\begin{align*}
 \epsilon_{\textrm{pb}}(P,Q_h,m,\delta)  & \geq a \underbrace{\ln \frac{a}{b}}_{\geq 3/2}   - \frac{(a-b)}{2}\left(  \frac{1}{(1-b)} + 1\right) \\
& \geq \frac{2a(1-b)-(a-b)}{2(1-b)} + \frac{b}{2} \\
& \geq \frac{2a(1-b)-(a-b)}{2(1-b)} \geq \frac{a-2ab+b}{2(1-b)} \\
& \geq \frac{a-2ab+ab}{2(1-b)} \geq \frac{a}{2} \geq \frac{a-b}{2}\\
 & \geq \frac{1}{2} \left(  \hat{\mathcal{L}}_{S}(h)-  {\mathcal{L}}_{\mathcal{D}}(h)   - (\Delta(h, Q_h,\mathcal{D}) + \Delta(h,Q_h, S) )\right).\\
\end{align*}

Rearranging, we get:
\begin{align*}
 \hat{\mathcal{L}}_{S}(h)-  {\mathcal{L}}_{\mathcal{D}}(h) &\leq \underbrace{2 \epsilon_{\textrm{pb}}(P,Q_h,m,\delta)  +  (\Delta(h, Q_h,\mathcal{D}) + \Delta(h,Q_h, S) )}_{\text{Applying Equation}~\ref{eq:pb-det-A} } \\
 & \leq \epsilon_{\textrm{pb-det-A}}(m,\delta).  \numberthis \label{eq:case-2}
\end{align*}

Since, for all $h \in \mathcal{H}_{3\delta}$ and $S \in \mathcal{S}_{3\delta}$, one of Equations~\ref{eq:case-1} and ~\ref{eq:case-2}
hold, we have that:

\begin{align*}
&\frac{1}{e^{3/2}}\left(\sup_{h \in \mathcal{H}_{3\delta}} \sup_{S \in \mathcal{S}_{3\delta}}   \hat{\mathcal{L}}_{S}(h) -
 {\mathcal{L}}_{\mathcal{D}}(h)\right)  - \frac{(e^{3/2} - 1)}{e^{3/2}}  (\hat{\epsilon}(m,\delta) + \epsilon_{\text{gen}}(m,\delta)) \leq \epsilon_{\textrm{pb-det-A}}(m,\delta).
\end{align*}

It follows from Equation~\ref{eq:pb-one-side} that the above bound holds good even after we take the absolute value of the first term in the left hand side. However, the absolute value is lower-bounded by $\epsilon_{\text{unif-alg}}(m,3\delta)$ (which follows from how $\epsilon_{\text{unif-alg}}(m,3\delta)$ is defined to be the smallest possible value over the choices of $\mathcal{H}_{3\delta}, \mathcal{S}_{3\delta}$).  

\end{proof}%

As a result of the above theorem, we can show   that $\epsilon_{\textrm{pb-det-A}}(m,\delta) = {\Omega}(1) - \mathcal{O}(\epsilon)$, thus establishing that, for sufficiently large  $D$, even though the generalization error would be negligibly small, the PAC-Bayes based bound would be as large as a constant. \\

\begin{corollary}
In the setup of Section~\ref{sec:setup}, for any $\epsilon,\delta > 0, \delta < 1/12$, when
 $D = \Omega\left(\max\left( m \ln \frac{3}{\delta}, m \ln\frac{1}{\epsilon}\right) \right)$, we have,
\[e^{-3/2} \cdot (1-\epsilon) - (1-e^{-3/2})(\epsilon) \leq \epsilon_{\textrm{pb-det-A}} (m,\delta).\]
\end{corollary}

\begin{proof}
The fact that $\epsilon_{\text{gen}}(m,\delta) \leq \epsilon$ follows from  Theorem~\ref{thm:example}. 
Additionally, $\hat{\epsilon}(m,\delta)=0$ follows from the proof of Theorem~\ref{thm:example}. Now, as long as $3\delta < 1/4$, and $D$ is sufficiently large (i.e., in the lower bounds on $D$ in Theorem~\ref{thm:example}, if we replace $\delta$ by $3\delta$), we have from Theorem~\ref{thm:example}
 that $\epsilon_{\text{unif-alg}}(m,3\delta) > 1-\epsilon$. Plugging these in Theorem~\ref{thm:pb-det-A}, we get the result in the above corollary.
 \end{proof}

\subsection{Deterministic PAC-Bayesian Bounds of Type B}

In this section, we consider another standard approach to making PAC-Bayesian bounds deterministic \citep{neyshabur18pacbayes,langford02pacbayes}. Here, the idea is to pick for each $h$ a distribution $Q_h$ such that for all $\vec{x}$:
\begin{equation*}
\mathcal{L}^{(0)}(h(\vec{x}),y) \leq \mathbb{E}_{\tilde{h} \sim Q_h} [\mathcal{L}'^{(\gamma/2)}(\tilde{h}(\vec{x}),y)] \leq  \mathcal{L}'^{(\gamma)}(h(\vec{x}),y),
\end{equation*}

where 
\[
\mathcal{L}'^{(\gamma)}(y,y') = \begin{cases} 0 & y \cdot y' \geq \gamma \\
1 & \text{else}.
\end{cases}
\]

Then, by applying the PAC-Bayesian bound of Equation~\ref{eq:stochastic-pac-bayes} for the loss $\mathcal{L}'_{\gamma/2}$, one can get a deterministic upper bound as follows, without having to introduce the extra $\Delta$ terms,

\begin{align*}
 {\mathcal{L}}^{(0)}_{\mathcal{D}}(h)  -  \hat{\mathcal{L}}^{(\gamma)}_{S}(h)  \leq 
&  \mathbb{E}_{\tilde{h} \sim Q_h} [\mathcal{L'}^{(\gamma/2)}(\tilde{h})] -  \mathbb{E}_{\tilde{h} \sim Q_h} [\hat{\mathcal{L}}_{S}'^{(\gamma/2)}(\tilde{h})]   \\
& \leq \sqrt{2\epsilon_{\text{pb}}(P,Q_h,m,\delta)} +  2\epsilon_{\text{pb}}(P,Q_h,m,\delta).
 \end{align*}

We first define this technique formally:

\begin{definition} 
The distribution-dependent, algorithm-dependent, deterministic PAC-Bayesian bound of (the hypothesis class $\mathcal{H}$, algorithm $\mathcal{A}$)-pair is defined to be the smallest value $\epsilon_{\text{pb-det-B}}(m, \delta)$ such that the following holds:
\begin{enumerate}
\item there exists a set of $m$-sized samples $\mathcal{S}_{\delta} \subseteq (\mathcal{X} \times \{-1, +1 \})^m$ for which:
\[
Pr_{S \sim \mathcal{D}^m} [S \notin \mathcal{S}_{\delta}]  \leq \delta.\]
\item and if we define $\mathcal{H}_{\delta} = \bigcup_{S \in \mathcal{S}_{\delta}}  \{ h_{S} \}$ to be the space of hypotheses explored only on these samples, then there must exist a prior $P$ and  for each $h$ a distribution $Q_h$, such that uniform convergence must hold as follows: for all $S \in \mathcal{S}_{\delta}$ and for all $h \in \mathcal{H}_{\delta}$,

\begin{align*}
 \sqrt{2\epsilon_{\text{pb}}(P,Q_h,m,\delta)} +  2\epsilon_{\text{pb}}(P,Q_h,m,\delta) < \epsilon_{\text{pb-det-B}}(m,\delta).
 \numberthis\label{eq:pb-det-B} 
\end{align*}

and for all $\vec{x}$: 
\begin{align*}
\mathcal{L}^{(0)}(h(\vec{x}),y) \leq \mathbb{E}_{\tilde{h} \sim Q_h} [\mathcal{L}'^{(\gamma/2)}(\tilde{h}(\vec{x}),y)] 
\leq  \mathcal{L}'^{(\gamma)}(h(\vec{x}),y)
\numberthis \label{eq:losses}
\end{align*}
as a result of which the following one-sided uniform convergence also holds:
\begin{align*}
\label{eq:pb-one-side-B}
&\sup_{S \in \mathcal{S}_{\delta}}\sup_{h \in \mathcal{H}_{\delta}}{\mathcal{L}}^{(0)}_{\mathcal{D}}(h)  -  \hat{\mathcal{L}}_{S}'^{(\gamma)}(h)  < \epsilon_{\text{pb-det-B}}(m,\delta).
\end{align*}
\end{enumerate}

\end{definition}

We can similarly show that $\epsilon_{\text{pb-det-B}}(m,\delta)$ is lower-bounded by the uniform convergence bound of $\epsilon_{\text{unif-alg}}$ too.

\begin{theorem}
\label{thm:pb-det-B}
Let $\mathcal{A}$ be an algorithm such that on at least $1-\delta$ draws of the training dataset $S$, the algorithm outputs a hypothesis $h_S$ such that the margin-based training loss can be bounded as:

\[
\hat{\mathcal{L}}_{S}'^{(\gamma)}(h_S) \leq \hat{\epsilon}(m,\delta)
\]
and with high probability $1-\delta$ over the draws of $S$, the generalization error can be bounded as:
\[
 \mathcal{L}'^{(\gamma)}_{\mathcal{D}}(h_{S})    - \mathcal{L}'^{(\gamma)}_{S}(h_{S}) \leq \epsilon_{\text{gen}}(m, \delta)  
\]

Then there exists a set of samples $\mathcal{S}_{3\delta}$ of mass at least $1-3\delta$, and a corresponding set of hypothesis $\mathcal{H}_{3\delta}$ learned on these sample sets such that:
\begin{align*}
&\left(\sup_{h \in \mathcal{H}_{3\delta}} \sup_{S \in \mathcal{S}_{3\delta}}  \mathcal{L}^{(0)}_{S}(h) -  \mathcal{L}'^{(\gamma)}_{\mathcal{D}}(h)\right)  - (e^{3/2}-1) (\hat{\epsilon}(m,\delta) + \epsilon_{\text{gen}}(m,\delta)) \leq \epsilon_{\textrm{pb-det-B}} (m,\delta).
\end{align*} 
\end{theorem}

Note that the above statement is slightly different from how Theorem~\ref{thm:pb-det-A} is stated as it is not expressed in terms of $\epsilon_{\text{unif-alg}}$. In the corollary that follows the proof of this statement, we will see how it can be reduced in terms of $\epsilon_{\text{unif-alg}}$.

\begin{proof}
Most of the proof is similar to the proof of Theorem~\ref{thm:pb-det-A}. Like in the proof of Theorem~\ref{thm:pb-det-A}, we can argue that there exists $\mathcal{S}_{3\delta}$ and $\mathcal{H}_{3\delta}$
 for which the test error can be bounded as,

 \begin{equation*}
\mathbb{E}_{\tilde{h} \sim Q_h}[ \mathcal{L}'^{(\gamma/2)}_{\mathcal{D}}(\tilde{h})]
\leq \mathcal{L}'^{(\gamma)}_{\mathcal{D}}(h)
 \leq \hat{\epsilon}(m,\delta) + \epsilon_{\text{gen}}(m,\delta),
\end{equation*}
where we have used $\epsilon_{\text{gen}}(m,\delta)$ to denote the generalization error of $\mathcal{L}'^{(\gamma)}$ and not the 0-1 error (we note that this is ambiguous notation, but we keep it this way for simplicity).

 For convenience, let us denote by $a:=\mathbb{E}_{\tilde{h} \sim Q_h}[ \hat{\mathcal{L}}'^{(\gamma/2)}_{S}(\tilde{h})]$ and  $b := \mathbb{E}_{\tilde{h} \sim Q_h}[ \mathcal{L}'^{(\gamma/2)}_{\mathcal{D}}(\tilde{h})]$. Again, let us consider, for some $h\in\mathcal{H}_{3\delta}$ and $S \in \mathcal{S}_{3\delta}$, the case that $e^{3/2} b \geq a$. Then, we have, using the above equation,

\begin{align*}
\hat{\mathcal{L}}^{(0)}_{S}({h}) -  \hat{\mathcal{L}}^{(\gamma)}_{\mathcal{D}}({h})
\leq &  a-b \\
\leq & (e^{3/2} - 1) b \\
\leq & (e^{3/2} - 1) (\hat{\epsilon}(m,\delta) + \epsilon_{\text{gen}}(m,\delta))   \\
\leq & (e^{3/2} - 1) (\hat{\epsilon}(m,\delta) + \epsilon_{\text{gen}}(m,\delta) + \epsilon_{\text{pb-det-B}}(m,\delta)).  \numberthis \label{eq:case-B-1} \\
\end{align*}

Now consider the case where $a  > e^{3/2} b$. Again, by similar arithmetic manipulation in the PAC-Bayesian bound of Equation~\ref{eq:stochastic-pac-bayes} applied on $\mathcal{L}'^{(\gamma/2)}$, we get,

\begin{align*}
 \epsilon_{\textrm{pb}}(P,Q_h,m,\delta) & \geq 
a \underbrace{\ln \frac{a}{b}}_{\geq 3/2}   - \frac{(a-b)}{2}\left(  \frac{1}{(1-b)} + 1\right)\\
&  \geq \frac{a-b}{2}\\
 & \geq \frac{1}{2} \left( \mathcal{L}^{(0)}_{S}(h) -  \mathcal{L}'^{(\gamma)}_{\mathcal{D}}(h)  \right). \\
\end{align*}

Rearranging, we get:
\begin{align*}
 & \mathcal{L}^{(0)}_{S}(h) -  \mathcal{L}'^{(\gamma)}_{\mathcal{D}}(h) \leq \underbrace{2 \epsilon_{\textrm{pb}}(P,Q_h,m,\delta)}_{\text{Applying Equation}~\ref{eq:pb-det-B}}\\
 & \leq  \epsilon_{\textrm{pb-det-B}}(m,\delta).  \numberthis \label{eq:case-B-2}
\end{align*}

Since, for all $h \in \mathcal{H}_{3\delta}$ and $S \in \mathcal{S}_{3\delta}$, one of Equations~\ref{eq:case-B-1} and ~\ref{eq:case-B-2}
hold, we have the claimed result.

\end{proof}%
Similarly,  as a result of the above theorem, we can show   that $\epsilon_{\textrm{pb-det-B}}(m,\delta) = {\Omega}(1) - \mathcal{O}(\epsilon)$, thus establishing that, for sufficiently large  $D$, even though the generalization error would be negligibly small, the PAC-Bayes based bound would be as large as a constant and hence cannot explain generalization. \\

\begin{corollary}
In the setup of Section~\ref{sec:setup}, for any $\epsilon,\delta > 0, \delta < 1/12$, when
 $D = \Omega\left(\max\left( m \ln \frac{3}{\delta}, m \ln\frac{1}{\epsilon}\right) \right)$, we have,
\[1-(e^{3/2}-1)\epsilon \leq \epsilon_{\textrm{pb-det-B}} (m,\delta).\]
\end{corollary}

\begin{proof}
It follows from the proof of Theorem~\ref{thm:example} that $\hat{\epsilon}(m,\delta)=0$, since all training points are classified by a margin of $\gamma$ (see Equation~\ref{eq:train-margin}). Similarly, from Equation~\ref{eq:test-margin} in that proof,  since most test points are classified by a margin of $\gamma$,  $\epsilon_{\text{gen}}(m,\delta) \leq \epsilon$. Now, as long as $3\delta < 1/4$, and $D$ is sufficiently large (i.e., in the lower bounds on $D$ in Theorem~\ref{thm:example}, if we replace $\delta$ by $3\delta$), we will get that there exists $S \in \mathcal{S}_{3\delta}$ and $h \in \mathcal{H}_{3\delta}$ for which the empirical loss  $\mathcal{L}^{(0)}$ loss is $1$. Then, by Theorem~\ref{thm:pb-det-B}, we get the result in the above corollary.
 \end{proof}

\end{document}